\def\year{2021}\relax
\documentclass[letterpaper]{article} 
\usepackage{aaai21}  
\usepackage{times}  
\usepackage{helvet} 
\usepackage{courier}  
\usepackage[hyphens]{url}  
\usepackage{graphicx} 
\usepackage{natbib}  
\usepackage{caption} 
\usepackage{comment}
\usepackage{lineno}

\usepackage{xr}
\usepackage{textcomp}

\usepackage{wrapfig}
\usepackage{xr-hyper}
\usepackage[utf8]{inputenc} 
\usepackage[T1]{fontenc}    
\usepackage{url}            
\usepackage{booktabs}       
\usepackage{amsfonts}       
\usepackage{nicefrac}       
\usepackage{microtype}      

\usepackage{comment}

\usepackage{amsmath,amsfonts,amssymb,mathtools,url,enumerate,verbatim}
\usepackage{amsthm}
\usepackage[german,french,british]{babel}
\usepackage{ucs}
\usepackage{times}
\usepackage{pgf,tikz}
\usepackage{float}
\usepackage{mathrsfs}
\usepackage{mathtools}
\usepackage{multirow}
\usepackage{yfonts}
\usepackage{caption}
\usepackage{subcaption}
\usetikzlibrary{arrows}
\usepackage{cleveref}

    \newcommand\newdot{{\kern.8pt\cdot\kern.8pt}}
\newcommand\nbull{{\kern.8pt\raise1.5pt\hbox{\small\bf .}\kern.8pt}}
\newcommand\1{\hbox{\kern.375em\vrule height1.57ex depth-.1ex
		width.05em\kern-.375em \rm 1}}

\newcommand\ncal{\mathcal{N}}

\newcommand\zcal{\mathcal{Z}}
\newcommand\rbb{\mathbb{R}}

\newcommand\acal{\mathcal{A}}

\newcommand\xcal{\mathcal{X}}

\DeclareMathOperator*{\argmax}{arg\,max}

\DeclareMathOperator{\iter}{iter}

\DeclareMathOperator{\rad}{\mathfrak{R}}

\DeclareMathOperator{\Fr}{Fr}

\newtheorem{theorem}{Theorem}
\newtheorem{lemma}[theorem]{Lemma}
\newtheorem{proposition}[theorem]{Proposition}
\newtheorem{corollary}[theorem]{Corollary}
\theoremstyle{definition}
\newtheorem{definition}{Definition}

\theoremstyle{definition}

\newtheorem{remark}{Remark}

\nocopyright
\title{Norm-based generalisation bounds for multi-class convolutional neural networks\footnote{
		 AL and MK acknowledge support by the German Research Foundation (DFG) award KL 2698/2-1 and by the Federal Ministry of Science and Education (BMBF) awards 01IS18051A and 031B0770E. YL acknowledges support by the National Natural Science Foundation of China (Grant No 61806091) and the Alexander von Humboldt Foundation.}
	\author{
		Antoine Ledent \textsuperscript{\rm 1}, Waleed Mustafa \textsuperscript{\rm 1}, Yunwen Lei \textsuperscript{\rm 1,2} and Marius Kloft \textsuperscript{\rm 1}}
}

\affiliations{
	
	\textsuperscript{\rm 1}Department of Computer Science, TU Kaiserslautern, 67653 Kaiserslautern, Germany\\
	\textsuperscript{\rm 2}School of Computer Science, University of Birmingham, Birmingham B15 2TT, United Kingdom
}

\begin{document}

\maketitle

\begin{abstract}
	We show generalisation error bounds for deep learning with two main improvements over the state of the art. (1) Our bounds have no explicit dependence on the number of classes except for logarithmic factors. This holds even when formulating the bounds in terms of the $L^2$-norm of the weight matrices, where previous bounds exhibit at least a square-root dependence on the number of classes.
(2) We adapt the classic Rademacher analysis of DNNs to incorporate weight sharing---a task of fundamental theoretical importance which was previously attempted only under very restrictive assumptions. In our results, each convolutional filter contributes only once to the bound, regardless of how many times it is applied. Further improvements exploiting pooling and sparse connections are provided.
The presented bounds scale as the norms of the parameter matrices, rather than the number of parameters. In particular, contrary to bounds based on parameter counting, they are asymptotically tight (up to log factors) when the weights approach initialisation, making them suitable as a basic ingredient in bounds sensitive to the optimisation procedure. We also show how to adapt the recent technique of loss function augmentation to our situation to replace spectral norms by empirical analogues whilst maintaining the advantages of our approach.
\end{abstract}
\section{Introduction}
Deep learning has enjoyed an enormous amount of success in a variety of engineering applications in the last decade \cite{bullshit1,bullshit2,bullshit3,GO}. However, providing a satisfying explanation to its sometimes surprising generalisation capabilities remains an elusive goal \cite{requires,gradientover,mutualinformation,adversarial}. The statistical learning theory of deep learning approaches this question by providing a theoretical analysis of the generalisation performance of deep neural networks (DNNs) through better understanding of the complexity of the function class corresponding to a given architecture or training procedure.

This field of research has enjoyed a revival since 2017 with the advent of learning guarantees for DNNs expressed in terms of various norms of the weight matrices and classification margins~\cite{NeyshabourPacs,Spectre,LSTMBytheway,convattempt,zhu}.
Many improvements have surfaced to make bounds non-vacuous at realistic scales, including better depth dependence, bounds that apply to ResNets \cite{resnet}, and PAC-Bayesian bounds using network compression~\cite{nonvacuous}, data-dependent Bayesian priors~\cite{Gintare}, fast rates~\cite{Taji}, and reduced dependence on the product of spectral norms via data-dependent localisation~\cite{weima,inversepreac}.
A particularly interesting new branch of research combines norm-based generalisation bounds with the study of how the optimisation procedure (stochastic gradient descent) implicitly restricts the function class~\cite{GradientDescent,gradientover,Aroratwolayers,Repeat,NTK,Lottery}. One idea at the core of many of these works is that the weights stay relatively close to initialisation throughout training, reinforcing lucky guesses from the initialised network rather than constructing a solution from scratch. Thus, in this branch of research, it is critical that the bound is negligible when the network approaches initialisation, i.e., \textit{the number of weights involved is not as important as their size}. This observation was first made as early as in~\cite{SupportsNormbased}.

Despite progress in so many new directions, we note that some basic questions of fundamental theoretical importance have remain unsolved.
(1) How can we remove or decrease the dependence of bounds on the number of classes?
(2) How can we account for weight sharing in convolutional neural networks (CNNs)?
In the present paper, we contribute to an understanding of both questions.

Question (1) is of central importance in extreme classification \cite{extremevarma}, where we deal with an extremely high number of classes (e.g. millions). \cite{Spectre} showed a bound with no explicit class dependence (except for log terms). However, this bound is formulated in terms of the $L^{2,1}$ norms of the network's weight matrices. If we convert the occurring $L^{2,1}$ norms into the more commonly used $L^2$ norms, we obtain a square-root dependence on the number of classes.

Regarding (2),
\cite{convattempt} showed a bound that accounts for weight sharing. However, this bound is valid only under the assumption of orthonormality of the weight matrices. The assumption of unit norm weights---which is violated by typical convolutional architectures (GoogLeNet, VGG, Inception, etc.)---makes it difficult to leverage the generalisation gains from small weights, and it is a fortiori not easy to see how the bounds could be expressed in terms of distance to initialisation.

In this paper, we provide, up to only logarithmic terms, a complete solution to both of the above questions. First, our bound relies only the $L^2$ norm at the last layer, yet it has no explicit (non-logarithmic) dependence on the number of classes.\footnote{As explained below, this corresponds to an implicit dependence of the order $\sqrt{C}$ if the classifying vectors have comparable norms. Our result is in line with the state of the art in shallow learning.} In deep learning, no generalization bound other than ours has ever achieved a lack of non-logarithmic class dependency with $L^2$ norms.
Second, our bound accounts for weight sharing in the following way. The Frobenius norm of the weight matrix of each convolutional filter contributes only \textit{once} to the bound, regardless of how many times it is applied.
Furthermore, our results have several more properties of interest: (i) We exploit the $L^\infty$-continuity of nonlinearities such as pooling and ReLu to further significantly reduce the explicit width dependence in the above bounds. (ii) We show how to adapt the recent technique of loss function augmentation to our setting to replace the dependence on the spectral norms by an \emph{empirical} Lipschitz constant with respect to well chosen norms.
(iii) Our bounds also have very little explicit dependence on architectural choices and rely instead on norms of the weight matrices expressed as distance to initialisation, affording a high degree of architecture robustness compared to parameter-space bounds. In particular, our bounds are negligible as the weights approach initialisation.

In parallel to our efforts, \cite{GRRR} recently made progress on question (2), providing a remedy to the weight-sharing problem. Their work, which is scheduled to appear in the proceedings of ICLR 2020, is independent of ours.  This can be observed from the fact that their work and ours were first preprinted on arXiv on the very same day. Their approach is completely different from ours, and both approaches have their merits and disadvantages. We provide an extensive discussion and comparison in the sections below and in Appendix~\ref{comparisondeep}.

\section{Related Work}
\label{relatedwork}
In this section, we discuss related work on the statistical learning theory (SLT) of DNNs. The SLT of neural networks can be dated back to 1970s, based on the concepts of VC dimension, fat-shattering dimension~\cite{book}, and Rademacher complexities~\cite{bartlett2002rademacher}. Here, we focus on recent work in the era of deep learning.

Let $(x_1,y_1),\ldots,(x_n,y_n)$ be training examples independently drawn from a probability measure defined on the sample space $\zcal=\xcal\times\{1,\ldots,K\}$, where $\xcal\subset\rbb^d$, $d$ is the input dimension, and $K$ is the number of classes. We consider DNNs parameterized by weight matrices $\acal=\{A^1,\ldots,A^L\}$, so that the prediction function can be written $F_{\acal}(x)=A^L\sigma_{L-1}\big(A^{L-1}\sigma_{L-2}\big(\cdots A^1 x\big)\big)$,
where $L$ is the depth of the DNN, $A^l\in\rbb^{W_l\times W_{l-1}},W_0=d,W_L=K$, and $\sigma_i:\rbb^{W_i}\mapsto \rbb^{W_i}$ is the non linearity (including any pooling and activation functions), which we assume to be 1-Lipschitz. 

When providing PAC guarantees for DNNs, a critical quantity is the Rademacher complexity of the network obtained after appending any loss function. The first work in this area \cite{early} therefore focused on bounding the Rademacher complexity of networks satisfying certain norm conditions, where the last layer is one-dimensional.
They apply the concentration lemma and a peeling technique to get a bound on the Rademacher complexity of the order \small
$O\big(\frac{2^L}{\sqrt{n}}\prod_{i=1}^{L}\|A^i\|_{\Fr}\big)$\normalsize,
where \small $\|A\|_{\Fr}$ \normalsize denotes the Frobenius norm of a matrix $A$.
\cite{BottomUp} showed that this exponential dependency on the depth can be avoided by an elegant use of the contraction lemma to obtain bounds of the order $
O\big((\sqrt{L}/\sqrt{n})\prod_{i=1}^{L}\|A^i\|_{\Fr}\big).
$\footnote{Note that both of these works require the output node to be one dimensional and thus are not multiclass}
The most related work to ours is the spectrally-normalized margin bound by \cite{Spectre} for multi-class classification. Writing $\|A\|_{\sigma}$ for the spectral norm is, and $M^i$ for initialised weights, the result is of order \small $\tilde{O}(M/\gamma)$ \normalsize  with 
\begin{align}\label{source}
M=\frac{1}{\sqrt{n}}\prod_{i=1}^{L}\|A^i\|_\sigma\left(\sum_{i=1}^{L}\frac{\|A^{i\top}-M^{i^\top}\|_{2,1}^{\frac{2}{3}}}{\|A^i\|_\sigma^{\frac{2}{3}}}\right)^{\frac{3}{2}},
\end{align}
where \small$\|A\|_{p,q}=\Big(\sum_j\big(\sum_i|A_{ij}|^p\big)^{\frac{q}{p}}\big)^{\frac{1}{q}}$ \normalsize is the $(p,q)$-norm, and $\gamma$ denotes the classification margin.

At the same time as the above result appeared, the authors in~\cite{NeyshabourPacs} used a PAC Bayesian approach to prove an analogous result \footnote{Note that the result using formula~\eqref{SourceNey} can also be derived by expressing~\eqref{source} in terms of $L^2$ norms and using Jensen's inequality}, where \small $W=\max\{W_0,W_1,\ldots,W_L\}$ \normalsize is the width:
\begin{align}
\label{SourceNey}
\tilde{O}\left(\frac{L\sqrt{W}}{\gamma\sqrt{n}}\left(\prod_{i=1}^L \|A^i\|_{\sigma} \right)   \left( \sum_{i=1}^L\frac{\|A^i+M^i\|_{\Fr}^{2}}{\|A^i\|_{\sigma}^{2}}  \right)^{\frac{1}{2}}\right).
\end{align}

These results provide solid theoretical guarantees for DNNs. However, they take very little architectural information into account. In particular, if the above bounds are applied to a CNN, when calculating the squared Frobenius norms $\|A^i\|_{\Fr}^2$, the matrix $A^i$ is the matrix representing the linear operation performed by the convolution, which implies that the weights of each filter will be summed as many times as it is applied. This effectively adds a dependence on the square root of the size of the corresponding activation map at each term of the sum. A notable exception would be the bound in Theorem 2 of~\cite{BottomUp}, which applies to DNN's and  scales like \small $\widetilde{O}\left(\sqrt{d}(\prod_{l=1}^L M(l))/\sqrt{n}\right)$ \normalsize where $M(l)$ is an upper bound for the $l^1$ norm of the rows of the matrix $\tilde{A}^l$. In this case, there is also a lack of explicit dependence on the number of times each filter is applied. However, the implicit dependence on other architectural parameters such as the size of the patches and the depth is stronger. Also, the activations are applied element-wise, which rules out pooling and multi-class losses.

Note also that the $L^2$ version~\eqref{SourceNey} of the above bound~\eqref{source} includes a dependence on the square root of the number of classes through the maximum width $W$ of the network. This square-root dependence is not favorable when the number of classes is very large. Although many efforts have been performed to improve the class-size dependency in the shallow learning literature~\cite{LauererMulti1,Shallow1,Shallow2,Koltch,Germ1,Germ2,Mohri,Multilei}, extensions of those results to deep learning are missing so far. 

In late 2017 and 2018, there was a spur of research effort on the question of fine-tuning the analyses that provided the above bounds, with improved dependence on depth~\cite{BottomUp}, and some bounds for recurrent neural networks \cite{LSTMrefused,LSTMBytheway}). Notably, in~\cite{convattempt}, the authors provided an analogue of~\eqref{source} for convolutional networks, but only under some very specific assumptions, including orthonormal filters.

Independently of our work, \cite[to appear at ICLR 2020]{GRRR} address the weight-sharing problem using a parameter-space approach. Their bounds scale roughly as the square root of the number of parameters in the model. In contrast to ours, their employed proof technique is more similar to~\cite{convattempt}: it focuses on computing the Lipschitz constant of the functions with respect to the parameters.
The result by \cite{GRRR} and ours, which we contrast in detail below, both have their merits.
In nutshell, the bound by \cite{GRRR} remarkably comes along without dependence on the product of spectral norms (up to log terms),
thus effectively removing the exponential dependence on depth.
Our result on the other hand comes along without an explicit dependence on the number of parameters, which can be very large in deep learning. As already noted in \cite{SupportsNormbased}, this property is crucial when the weights are small or close to the initialisation.

Lastly, we would like to point out that, over the course of the past year, several techniques have been introduced to replace the dependence on the product of spectral norms by an empirical version of it, at the cost of either assuming smoothness of the activation functions \cite{weima} or a factor of the inverse minimum preactivation \cite{inversepreac}.
Slightly earlier, a similar bound to that in~\cite{GRRR} (with explicit dependence on the number of parameters) had already been proved for an unsupervised data compression task (which does not apply to our supervised setting) in~\cite{latedec}.
Recently, another paper addressing the weight sharing problem appeared on arXiv \cite{newwork}. In this paper, which was preprinted several months after \cite{GRRR} and ours, the authors provided another solution to the weight sharing problem, which incorporates elements from both our approach and that of~\cite{GRRR}: they bound the $L^2$-covering numbers at each layer independently, but use parameter counting at each layer, yielding \textit{both} an unwanted dependence on the number of parameters in each layer (from the parameter counting) \textit{and} a dependence on the spectral norms from the chaining of the layers. 

Further related work includes the following. \cite{Du} showed size-free bounds for CNNs in terms of the number of parameters for two-layer networks. In~\cite{Otherlong}, the authors provided an ingenious way of computing the spectral norms of convolutional layers, and showed that regularising the network to make them approach $1$ for each layer is both feasible and beneficial to accuracy. 
Other than the above mentioned work, several researchers have provided interesting insights into DNNs from different perspectives, including through model compression~\cite{NeyshabourPacs}, capacity control by VC dimensions~\cite{BartVC}, and the implicit restriction on the function class imposed by
the optimisation procedure~\cite{compression,nonvacuous,Overparametrised,Taji,gradientover,NTK,Aroratwolayers}.


\section{Contributions in a Nutshell}
\label{informal}
\label{notation}
In this section, we state the simpler versions of our main results for specific examples of neural networks. The general results are described in in more technical detail in Section~\ref{precise}.  
\subsection{Fully Connected Neural Networks}
In the fully connected case, the bound is particularly simple: 
\begin{theorem}[Multi-class, fully connected]
	\label{fully}
	Assume that we are given some fixed reference matrices $M^1,M^2,\ldots,M^L$ representing the initialised values of the weights of the network. Set \small $ \widehat{R}_\gamma(F_{\mathcal{A}})=(1/n)(\#(i:F(x_i)_{y_i}< \gamma + \max_{j\neq y_i}F(x_i)_{j}))$ \normalsize
	With probability at least $1-\delta$, every network $F_{\mathcal{A}}$ with weight matrices $\mathcal{A}=(A^1,A^2,\ldots,A^L)$ and every margin $\gamma>0$ satisfy:
	\begin{align}
	\label{ourR}
	&\mathbb{P}(\argmax_j(F_{\mathcal{A}}(x)_j)\neq y)\leq \widehat{R}_\gamma(F_{\mathcal{A}})+\\&\widetilde{\mathcal{O}}\left(   \frac{\max_{i=1}^n\|x_i\|_{\Fr}R_{\mathcal{A}}}{\gamma \sqrt{n}}\log(\bar{W})+\sqrt{\frac{ \log(1/\delta) }{ n }  }\right),
	\end{align}
	where \footnotesize $W=\bar{W}=\max_{i=1}^LW_i$\normalsize \normalsize  is the maximum width of the network, and \begin{align}
	&R_{\mathcal{A}}:=L\max_{i}\|A^L_{i,\nbull}\|_{\Fr}\left(\prod_{i=1}^{L-1} \|A^i\|_{\sigma} \right) \\  &\left( \sum_{i=1}^{L-1}\frac{(\|A^i-M^i\|_{2,1}^{2/3}}{\|A^i\|_{\sigma}^{2/3}}  +\frac{\|A^L\|_{\Fr}^{2/3}}{\max_{i}\|A^L_{i,\nbull}\|^{2/3}_{\Fr}}\right)^{\frac{3}{2}}.
	\end{align}
\end{theorem}

Note that the last term of the sum does not explicitly contain architectural information, and assuming bounded $L^2$ norms of the weights, the bound only implicitly depends on $W_i$ for $i\leq L-1$ (through $\|A^i-M^i\|_{2,1}\leq \sqrt{W_{i-1}}\|A^i-M^i\|_{\Fr}$), but not on $W_L$ (the number of classes). This means the above is a class-size free generalisation bound (up to a logarithmic factor) with $L^2$ norms of the last layer weight matrix. This improves on the earlier $L^{2,1}$ norm result in~\cite{Spectre}.
To see this, let us consider a standard situation where the rows of the matrix $A^L$ have approximately the same $L^2$ norm, i.e., $\|A^L_{i,\nbull}\|_2\asymp a$. (In Section~\ref{fillup} in the Appendix, we show that this condition holds except on a subset of weight space of asymptotically vanishing lebesgue measure and further discuss possible behaviour of the norms.)  In this case, our bound involves $\|A^L\|_{\Fr}\asymp\sqrt{W_L}a$, which incurs a square-root dependency on the number of classes. As a comparison, the bound in \cite{Spectre} involves $\|(A^L)^\top\|_{2,1}\asymp W_La$, which incurs a linear dependency on the number of classes. If we further impose an $L_2$-constraint on the last layer as $\|A^L\|_{\Fr}\leq a$ as in the SVM case for a constant $a$~\cite{Multilei}, then our bound would enjoy a logarithmic dependency while the bound in \cite{Spectre} enjoys a square-root dependency.
This cannot be improved without also changing the dependence on $n$. Indeed, if it could, we would be able to get good guarantees for classifiers working on fewer examples than classes.
Furthermore, in the above bound, the dependence on the spectral norm of $A^L$ in the other terms of the sum is reduced to a dependence on $\max_{i}\|A^L_{i,\nbull}\|_{2}$.
Both improvements are based on using the $L^\infty$-continuity of margin-based losses.

\subsection{Convolutional Neural Networks}

Our main contribution relates to CNNs. For the convenience of the reader, we first present a simple versions of our results.
\subsubsection{Two-layers}

The topic of the present paper is often notationally cumbersome, which imposes an undue burden on the reviewers and readers. Therefore, we first present a particular case of our bound for a two-layer network composed of a convolutional layer and a fully connected layer with a single input channel, \textit{with explicit pre chosen norm constraints}\footnote{It is common practice to leave the post hoc step to the reader in this way. Cf.,e.g.,~\cite{GRRR})}. Note that the restrictions are purely based on notational and reader convenience: more general results are presented later and in the supplementary material.

\textbf{2-layer Notation:}
Consider a two-layer network with a convolutional layer and a fully connected layer. Write $d,C$ for the dimensions of the input space and the number of classes respectively.  We write $w$ for the \textit{spacial} dimension of the hidden layer \textit{after pooling}\footnote{This is less than the number of convolutional patches in the input and is not influenced by the number of filters applied.} Write $A^1,A^2$ for the weight matrices of the first and second layer, with the weights appearing only once in the convolutional case (thus, the matrix $\tilde{A}^1$representing the convolution operation presents the weights of the matrix $A_1$ repeated as many times as the filters are applied). For any input $x\in \mathbb{R}^d$, we write $|x|_0$ for the maximum $L^2$ norm of \textit{a single convolutional patch} of $x$. The network is represented by the function $$F(x)=A^2 \sigma(\tilde{A}^1x),$$ where $\sigma$ denotes the non linearities (including both pooling and activation functions).  As above, $M^1,M^2$ are the initialised weights.
\begin{theorem}
	\label{Saviour}
	
	Let $a_1,a_2,a_*,b_0,b_1>0$. Suppose that the distribution over inputs is such that $|x|_{0}\leq b_0$ a.s. With probability $>1-\delta$ over the draw of the training set, for every network $\mathcal{A}=(A^1,A^2)$ with weights satisfying $\|(A^1-M^1)^\top\|_{2,1}\leq a_1$, $\|A^2-M^2\|_{\Fr}\leq a_2$ and  $\sup_{c\leq C}\| A^{2}_{c,\nbull}\|_2\leq a_{*}$, if $\sup_{i\leq n} \|\tilde{A}^1x_n\|_{\Fr}\leq b_1$ , then
		\begin{align}
	\label{Saviour?}
	&\mathbb{P}\left(\argmax_j(F_{\mathcal{A}}(x)_j)\neq y\right)\\&\nonumber\leq \widehat{R}_{\gamma}(F_{\mathcal{A}})+3\sqrt{\frac{\log(\frac{2}{\delta })}{2n}}+\frac{\mathcal{C}}{\sqrt{n}}\mathcal{R}\left[\log_2(n^2\mathcal{D})\right]^{\frac{1}{2}}\log(n),
	\end{align}
	where   $\mathcal{C}$ is an absolute constant,
	\begin{align}
\mathcal{R}^{2/3}=\left[b_0a_1\max\left(\frac{1}{b_1},\frac{\sqrt{w}a_*}{\gamma}\right)\right]^{2/3}+\left[\frac{b_1a_2}{\gamma}\right]^{2/3},
	\end{align}
	and the quantity in the log term is $\mathcal{D}=\max(b_0a_1\bar{W}a_*/b_1,b_1a_2C/\gamma)$ where $\bar{W}$ is the number of hidden neurons before pooling. 
	
	\end{theorem}
\textbf{Remarks:} 
\begin{enumerate}
	\item Just as in the fully connected case, the implicit dependence on the number of classes is only through an $L^2$ norm of the full last layer matrix. $b_1$ is a an upper bound on the $L^2$ norms of hidden activations.
	\item  $a_1$ is the norm of the filter matrix $A^1$, which counts each filter only once regardless of how many times it is applied. This means our bound enjoys only logarithmic dependence on input size for a given stride. 
	\item As explained in more detail at the end of  Appendix~\ref{comparisondeep}, there is also no explicit dependence on the size of the filters and the bound is stable through up-resolution. In fact, there is no explicit non logarithmic dependence on architectural parameters, and the bounds converges to 0 as $a_1,a_2$ tend to zero (in contrast to parameter space bounds such as~\cite{GRRR}).
	\item $a_*$ replaces the spectral norm of $A^2$, and is only equal to the maximum $L^2$ norm of the second layer weight vectors corresponding to each class. This improvement,comes from better exploiting the continuity of margin based losses with respect to the $L^\infty$ norm. 
	\item The spectral norm of the first layer matrix $\tilde{A}_1$ is not neccessary and is absorbed into an empirical estimate of the hidden layer norms. The first term in the $\max$ relates to the estimation of the risk of a test point presenting with a hidden layer norm higher than (a multiple of) $b_1$.
	\item $b_0$ refers to the maximum $L^2$ norm of a single convolutional patch over all inputs and patches. 
	\end{enumerate}

\textbf{A result for the multi-layer case}
We assume we are given training and testing points  $(x,y),(x_1,y_1),(x_2,y_2),\ldots,(x_n,y_n)$  drawn iid from any probability distribution over $\mathbb{R}^d\times\{1,2,\ldots,C\}$. We suppose we have a convolutional architecture so that for each filter matrix $A^{l}\in \mathbb{R}^{m_l\times d_l}$ from layer $l-1$ to layer $l$, we can construct a larger matrix $\tilde{A}^l$ representing the corresponding (linear) convolutional operation. The $0^{th}$ layer is the input, whist the $L^{th}$ layer is the output/loss function. We write $w_l$ for the \textit{spacial} width at layer $l$, $W_l$ for the total width at layer $l$ (including channels), and $W$ for $\max_l W_l$. For simplicity of presentation, we assume that the activation functions are composed only of ReLu and max pooling.
\begin{theorem}
	\label{posthocasymptsimple}
	With probability $\geq 1-\delta$, every network $F_{\mathcal{A}}$ with fliter matrices $\mathcal{A}=\{A^1,A^2,\ldots,A^L\}$ and every margin $\gamma>0$ satisfy:
	\begin{align}
	\label{OurR}
	&\mathbb{P}\left(\argmax_j(F_{\mathcal{A}}(x)_j)\neq y\right)\nonumber \\&\leq \widehat{R}_{\gamma}(F_{\mathcal{A}})+\widetilde{\mathcal{O}}\left(   \frac{R_{\mathcal{A}}}{\sqrt{n}}\log(\bar{W})+\sqrt{\frac{ \log(1/\delta) }{ n }  }\right),
	\end{align}
	where  $\bar{W}$ is the maximum number of neurons in a single layer (before pooling) and
	$$R_{\mathcal{A}}^{2/3}=\sum_{l=1}^L (T_l)^{2/3}$$ for where $T_l=$
	\small
	\begin{align*}
	B_{l-1}(X)\|(A^l-M^l)^\top\|_{2,1}\sqrt{w_l}\max_{U\leq L}\frac{\prod_{u=l+1}^{U} \|\tilde{A}^u\|_{\sigma'}}{B_U(X)}
	\end{align*}
	\normalsize
	if $l\leq L-1$ and 	for $l=L$, $T_l=$
	\begin{align*}
	\frac{B_{L-1}(X)}{\gamma}\|A^L-M^L\|_{\Fr}.
	\end{align*}
	Here, $w_l$ is the spacial width at layer $l$ after pooling. By convention, $b_L=\gamma$, and for any layer $l_1$, 	\small $B_{l_1}(X):=\max_{i}\left|F^{0\rightarrow l_l}(x_i)\right|_{l_1}$ \normalsize  denotes the maximum $l^2$ norm of any convolutional patch of the layer $l_1$ activations, over all inputs. For \small$l\leq L-1$\normalsize, \small$\|\tilde{A}_l\|_{\sigma'}\leq \|\tilde{A}_l\|$\normalsize \, denotes the maximum spectral norm of any matrix obtained by deleting, for each pooling window, all but one of the corresponding rows of $\tilde{A}$.
	In particular, for $l=L$, \small $\|\tilde{A}^L\|_{\sigma'}=\rho_L\max_{i}\|A^L_{i,\nbull}\|_{2}$.\normalsize Here  $A^L_{i,\nbull}$ denotes the $i$'th row of $A^L$, and  $\|\nbull\|_{2}$ denotes the Frobenius norm\footnote{NB: A simplified version of the above Theorem can be obtained where $T_l=\prod_{i\neq l}\|\tilde{A}^i\|_\sigma \|(A^l-M^l)^\top\|_{2,1}\sqrt{w_l}/\gamma$ for $l\leq L-1$ and $T_L=\prod_{i=1}^{L-1}\|\tilde{A}^i\|_\sigma\|A^L-M^L\|_{\Fr}$. See Appendix E and in particular equation~\eqref{fuyan}.
	}.
\end{theorem}

Similarly to the two-layer case above, a notable property of the above bounds is that the norm involved is that of the matrix $A^l$ (the filter) instead of $\tilde{A}^l$ (the matrix representing the full convolutional operation), which means we are only adding the  norms of each filter once, regardless of how many patches it is applied to. As a comparison, although the genrealization bound in \cite{Spectre} also applies to CNNs, the resulting bound would involve the whole matrix $\widetilde{A}$ ignoring the structure of CNNs, yielding an extra factor of $O_{l-1}$ instead of $\sqrt{O_l}$, where $O_l$ denotes the number of convolutional patches in layer $l$: Through exploiting weight sharing, we remove a factor of $\sqrt{O_{l-1}}$ in the $l^{th}$ term of the sum compared to a standard the result in~\cite{Spectre}, and we remove another factor of $\sqrt{O_{l-1}/w_l}$ through exploitation of the $L^\infty$ continuity of max pooling and our use of $L^\infty$ covering numbers. 

A further significant improvement is in replacing the factor $\|X\|_{2,2}\prod_{i=1}^{l-1}\|\tilde{A}_i\|_{\sigma}
$ from the classic bound by $B_{l-1}(X)$, which is the maximum $L^2$ norm of a single convolutional patch. This implicitly removes another factor of $\sqrt{O_{l-1}}$, this time from the local connection structure of convolutions.

We note that it is possible to obtain more simple bounds without a maximum in the definition of $T_l$ by using the spectral norms to estimate the norms at the intermediary layers. 

\subsection{Empirical spectral norms; Lipschitz augmentation}

\label{augment}
A commonly mentioned weakness of norm-based bounds is the dependence on the product of spectral norms from above. In the case of fully connected networks, there has been a lot of progress last year on how to tackle this problem. In particular, it was shown in~\cite{inversepreac} and in~\cite{weima} that the products of spectral norms can be replaced by empirical equivalents, at the cost of either a factor of the minimum preactivation in the Relu case~\cite{inversepreac}, or Lipschitz constant of the \textit{derivative} of the activation functions if one makes stronger assumptions~\cite{weima}. In the appendix, we adapt some of those techniques to our convolutional, ReLu situation and find that the quantity $\rho^{\mathcal{A}}_l$ can be replaced in our case by:
$\rho^{\mathcal{A}}_l=\max\left(\max_{i}\max_{\tilde{l}\geq l}\frac{\rho^{\mathcal{A},x_i}_{l_1\rightarrow l_2}}{B_{l_2}(X)}, \max_{i}\max_{\tilde{l}\geq l}\frac{\theta^{\mathcal{A},x_i}_{l_1\rightarrow l_2}}{E_{l_2}(X)}\right)$
where $E_l(X)$ denotes the minimum preactivation (or distance to the max/second max in max pooling) at layer $l$ for over every input, $\rho^{\mathcal{A},x_i}_{l_1\rightarrow l_2}$ (resp. $\theta^{\mathcal{A},x_i}_{l_1\rightarrow l_2}$) is the Lipschitz constant of gradient of $F^{l_1\rightarrow l_2}$ with respect to the norms $|\nbull|_{\infty,l_1}$ and $|\nbull|_{l_2}$(resp. $|\nbull|_{\infty,l_1}$ and $|\nbull|_{\infty}$). These quantities can be computed explicitly: if $ M=\nabla_{F^{0\rightarrow l_1}(x_i)} F^{l_1\rightarrow l_2}$ so that locally around $F^{0\rightarrow l_1}(x_i)$, $F^{l_1\rightarrow l_2}(x)=Mx$, then $\theta^{\mathcal{A},x_i}_{l_1\rightarrow l_2}=\|M^{\top}\|_{1,\infty}$ and $\rho^{\mathcal{A},x_i}_{l_1\rightarrow l_2}=\max_{M}\|M'\|_{1,2}$ where $M'$ runs over all sub matrices of $M$ obtained by keeping only the rows corresponding to a single patch of layer $l_2$.

Note that an alternative approach is to obtain tighter bounds on the worst-case Lipschitz constant. Theorem~\ref{posthocasymptlip}  in the Appendix is a variation of Theorem~\ref{posthocasymptsimple} involving the explicit worst case Lipschitz constants across layers instead of spectral norms. These quantities can then be bounded, or made small via regularisation using recent techniques (cf, e.g. \citet{lip2,lip3}).

\section{General proof strategy}

Some key aspects of our proofs and general results rely on using the correct norms in activation spaces. On each activation space, we use the norm $|\nbull|_\infty$ to refer to the maximum absolute value of each neuron in the layer, the norm $|\nbull|_{l}$ to refer to the the maximum $l^2$ norm of a single convolutional patch (at layer $l$) and $|\nbull|_{\infty,l}$  for the maximum $l^2$ norm of a single pixel viewed as a vector over channels. Using these norms, we can for each pair of layers $l_1,l_2$ define the quantity $\rho_{l_1\rightarrow  l_2}$ as the Lipschitz constant of the subnetwork \small $F^{l_1\rightarrow  l_2}$ \normalsize with respect to the norms $|\nbull|_{\infty,l_1} $ and $|\nbull|_{l_2}$. Using those norms we can formulate a cleaner extention of Theorem~\ref{posthocasymptsimple}  where  the quantity $R_\mathcal{A}$ can be replaced by 
\begin{align*}
& \bigg[\sum_{l=1}^{L-1}\left(B_{l-1}(X)\|A^l-M^l\|_{2,1}\max_{\tilde{l}>l}\frac{\rho_{l\rightarrow \tilde{l}}}{B_{\tilde{l}}(X)}\right)^{2/3}\\&\quad \quad \quad \quad \quad  +\left(\frac{B_{L-1}(X)}{\gamma}\|A^L-M^L\|_{\Fr}\right)^{2/3}\bigg]^{3/2},
\end{align*}
\normalsize
where for any layer $l_1$, 	\small $B_{l_1}(X):=\max_{i}\left|F^{0\rightarrow l_l}(x_i)\right|_{l_1}$\normalsize  denotes the maximum $l^2$ norm of any conv. patch of the layer $l_1$, over all inputs. $B_L(X)=\gamma$. Our proofs derive this result, and the previous Theorems follow. See Section~\ref{precise}, Theorem~\ref{posthocasymptlip}\footnote{Our boundedness assumptions on \textit{worst-case} Lipschitz constants remove some of the interactions between layers, yielding simpler results than~\cite{weima,inversepreac}}.

 In the rest of this Section, we sketch the general strategy of the proof, focusing on the (crucial) one-layer step. At this point, we need to introduce notation w.r.t. the convolutional channels: we will collect the data matrix of the previous layer in the form of a tensor $X\in \mathbb{R}^{n\times U\times d}$ consisting of all the convolutional patch stacked together: if we fix the first index (sample i.d.) and the second index (patch i.d.), we obtain a convolutional patch of the corresponding sample. For a set of weights $A\in \mathbb{R}^{d\times m}
$, the result of the convolutional operation is written $XA$ where is defined by $(XA)_{u,i,j}=\sum_{o=1}^dX_{u,i,o}A_{o,j}$ for all $u,i,j$. 

A first step in bounding the capacity of NN's is to provide a bound on the covering numbers of individual layers.
	\begin{definition}[Covering number]\label{def:covering-number}
	Let $V\subset\rbb^n$ and $\|\cdot\|$ be a norm in $\rbb^n$. The covering number w.r.t. $\|\cdot\|$, denoted by $\ncal(V,\epsilon,\|\cdot\|)$, is the minimum cardinality $m$ of a collection of vectors $\mathbf{v}^1,\ldots,\mathbf{v}^m\in\rbb^n$ such that
	$
	\sup_{\mathbf{v}\in V}\min_{j=1,\ldots,m}\|\mathbf{v}-\mathbf{v}^j\|\leq\epsilon.
	$
	In particular, if $\mathcal{F}\subset \rbb^{\mathcal{X}}$ is a function class and $X=(x_1,x_2,\ldots,x_n)\in \mathcal{X}^n$ are data points, $\ncal(\mathcal{F}(X),\epsilon,(1/\sqrt{n})\|\cdot\|_{2})$ is the minimum cardinality $m$ of a collection of functions $\mathcal{F}\ni f^1,\ldots,f^m:\mathcal{X}\rightarrow \rbb$ such that for any $f\in \mathcal{F}$, there exists $j\leq m$ such that $\sum_{i=1}^n (1/n)\left|f^j(x_i)-f(x_i)\right|^2\leq \epsilon^2$ . Similarly,  $\ncal(\mathcal{F}(X),\epsilon,\|\cdot\|_{\infty})$ is the minimum cardinality $m$ of a collection of functions $\mathcal{F}\ni f^1,\ldots,f^m:\mathcal{X}\rightarrow \rbb$ such that for any $f\in \mathcal{F}$, there exists $j\leq m$ such that  $i\leq n, \quad\left|f^j(x_i)-f(x_i)\right|\leq \epsilon$.
\end{definition}
If we apply classical results on linear classifiers as is done in~\cite{Spectre} (where results on $L^2$ covering numbers are used) by viewing a convolutional layer as a linear map directly, we cannot take advantage of weight sharing. In this work, we circumvent this difficulty by applying results on the $L^\infty$ covering numbers of classes of linear classifiers to a different problem where each "(convolutional patch, sample, output channel)" combination is mapped into a higher dimensional space to be viewed as a single data point. A further reduction in dependence on architectural parameters is achieved by leveraging the $L^\infty$-continuity of margin-based loss functions and pooling.
We will need the following result from~\cite{PointedoutbyYunwen} (Theorem 4, page 537).
\begin{proposition}
	\label{MaureySup}
	Let $n,d\in\mathbb{N}$, $a,b>0$. Suppose we are given $n$ data points collected as the rows of a matrix $X\in \mathbb{R}^{n\times d}$, with $\|X_{i,\nbull}\|_2\leq b,\forall i=1,\ldots,n$. For $U_{a,b}(X)=\big\{X\alpha:\|\alpha\|_2\leq a,\alpha\in\rbb^d\big\}$, we have
	\small 
	\begin{align*}
	\log\mathcal{N}\left(U_{a,b}(X),\epsilon,\|\nbull\|_{\infty}  \right)\leq \frac{36a^2b^2}{\epsilon^2}\log_2\left(\frac{8abn}{\epsilon}+6n+1\right).
	\end{align*}
\end{proposition}
Note this proposition is stronger than Lemma 3.2 in~\cite{Spectre}. In the latter, the cover can be chosen independently of the data set, and the metric used in the covering is an \small $L^2$ \normalsize average over inputs. In Proposition~\ref{MaureySup}, the covering metric is a maximum over all inputs, and the data set must be chosen in advance, though the size of the cover only depends (logarithmically) on the sample size\footnote{We note that the proof is also much more obscure, although it is far more approachable to prove an analogous result with a squared log term instead, by going via the shattering dimension.}.

Using the above on the auxiliary problem based on (input, convolutional patch, ouput channel) triplets, we can prove the following covering number bounds for the one-layer case:
\begin{proposition}
	\label{suplinn}
	Let positive reals $(a,b,\epsilon)$ and positive integer $m$ be given. Let the tensor $X\in \mathbb{R}^{n\times U\times d}$ be given with $\forall i\in \{1,2,\ldots,n\},\forall u \in \{1,2,\ldots,U\},  \quad  \|X_{i,u,\nbull}\|_{2}\leq b $. For any choice of reference matrix $M$, we have
	\begin{align*}
	&\log \mathcal{N}  \left(     \{XA : A\in \mathbb{R}^{d\times m}, \|A-M\|_{\Fr}\leq a\},\epsilon,  \|\nbull\|_{\infty}   \right)\\& \leq  \frac{36a^2b^2}{\epsilon^2}\log_2\left[\left(\frac{8ab}{\epsilon}+7\right)mnU\right],
	\end{align*}
	where the norm $\|\nbull\|_{\infty}$ is over the space $ \mathbb{R}^{n\times U\times m}$. 

\end{proposition}

\textbf{Sketch of proof:}
By translation invariance, it is clear that we can suppose $M=0$. We consider the problem of bounding the $L^\infty$ covering number of $\{(v_i^\top X^j)_{i\leq I,j\leq J}: \sum_{i\leq I}\|v_i\|^2_2\leq a^2\}$ (where $X^j\in \mathbb{R}^{d\times n}$ for all $j$) with only logarithmic dependence on $n,I,J$. Here, $I$ plays the role of the number of output channels, while $J$ plays the role of the number of convolutional patches.
We now apply the above Proposition~\ref{MaureySup} on the $nIJ\times dI $ matrix constructed as follows:

\small
\begin{align*}
\left(\begin{array}{ccccccccccccccccc}
X^1&  0& \ldots & 0	 \\
0   &  X^1& \ldots & 0	\\
\ldots &  \ldots & \ldots &\ldots\\
0 &  0& \ldots & X^1\\
X^2 &  0 & \ldots& 0 \\
0 & X^2 &  \ldots & \ldots\\
\ldots & \ldots &\ldots  & \ldots\\
0  & 0	&\ldots & X^2 \\
X^3 &  0 & \ldots& 0 \\
\ldots &  \ldots & \ldots &\ldots\\
X^J&  0& \ldots & 0	 \\
0   &  X^J& \ldots & 0	\\
\ldots &  \ldots & \ldots &\ldots\\
0 &  0& \ldots & X^J\\
\end{array}\right)^\top,
\end{align*}
\normalsize 
with the corresponding vectors being constructed as $(v_1,v_2,\ldots,v_I)\in \mathbb{R}^{dI}$.

If we compose the linear map  on $\mathbb{R}^{n\times d}$ represented by $(v_1,v_2,\ldots,v_I)^\top$ with $k$ real-valued functions with $L^\infty$ Lipschitz constant $1$, the above argument yields comparable bounds on the $\|\nbull\|_{2}$ covering number of the composition, losing a factor of $\sqrt{k}$ only (for the last layer, $k=1$, and for convolutional layers, $k$ is the number of neurons in the layer left after pooling).

The proposition above is only enough to deal with a purely $l^2$ version of our bounds. To prove Theorem~\ref{posthocasymptsimple}, which involves $\|\nbull\|_{2,1}$ norms, we must show the following extension:
\begin{proposition}
	\label{suplin}
	Let positive reals $(a,b,\epsilon)$ and positive integer $m$ be given. Let the tensor $X\in \mathbb{R}^{n\times U\times d}$ be given with $\forall i\in \{1,2,\ldots,n\},\forall u \in \{1,2,\ldots,U\},  \>  \|X_{i,u,\nbull}\|_{2}\leq b $. For any fixed $M$:
	\begin{align*}
	&\log \mathcal{N}  \left(     \{XA : A\in \mathbb{R}^{d\times m}, \|A-M\|_{2,1}\leq a\},\epsilon,  \|\nbull\|_{*}   \right) \\
	&\leq  \frac{64a^2b^2}{\epsilon^2} \log_2\left[\left(\frac{8ab}{\epsilon}+7\right)mnU\right],
	\end{align*}
	where 
	 the norm $\|\nbull\|_{*}$  over the space $ \mathbb{R}^{n\times U\times m}$  is defined by \small $\| Y  \|_{*}= \max_{i\leq n}\max_{j\leq U} \left[\sum_{k=1}^m Y_{i,j,k}^2\right]^{\frac{1}{2}}. $
\end{proposition}
\textbf{Sketch of proof:}
We first assume fixed bounds on the $L^2$ norms $\|A^{i,\nbull}\|_{2}=a_i$ of each filter, and use Proposition~\ref{suplinn} with $m=1$ for each output channel with a different granularity $\epsilon_i$. We then optimize over the choice of $\epsilon_i$, and make the result apply to the case where only $a=\sum_{i} a_i\geq \|A\|_{2,1}$ is fixed in advance by $l^1$ covering the set of possible choices for $(a_1,a_2,\ldots,a_m)$ for each $a$, picking a cover for each such choice and taking the union. We accumulate a factor of $2$ because of this approach, but to our knowledge, it is not possible to rescale the inputs by factors of $\sqrt{a_i}$ as was done in~\cite{Spectre}, as the input samples in an $L^\infty$ covering number bound must be chosen in advance.

We can now \textbf{sketch the proof of the Theorem~\ref{Saviour} :} 
  we use the loss function \begin{align*}
l(x_i,y_i)&=\max\big[\lambda_{b_1}(\|\sigma(\tilde{A}^1x_i)\|_2-b_1),\nonumber\\&\lambda_{\gamma}\big(\max_{j\neq y}(A^2\sigma(\tilde{A}^1x_i))_{j}-(A^2\sigma(\tilde{A}^1x_i))_{y_i}\big)   \big],
\end{align*}
where for any $\theta>0$ the \textit{ramp loss} $\lambda_{\theta}$ is defined by $\lambda_\theta=1+\min(\max(x,-\theta),0)/\theta$.
This loss incorporates the following two failure scenarios: (1) the $L^2$ norm of the hidden activations exceed a multiple of $b_1$ (2) The activations behave normally but the network still outputs a wrong prediction.
Since pooling is continuous w.r.t. the infty norm, the above results for the one layer case applied to a layer yields an $\epsilon$ cover of hidden layer w.r.t to the $L^\infty$ norm. The contributions to the error source (1) therefore follows directly from the first layer case. The contribution of the 1st layer cover error to (2) must be multiplied $1/\gamma$ and the Lipschitz constant of $A^2$ with respect to the $L^\infty$ norms, which we estimate by $\sqrt{w}a_*$ since the Euclidean norm of the deviation from the cover at the hidden layer is bounded by $\sqrt{w}$ times the deviation in $||_{\infty,1}$ norm \footnote{This norm is a supremum over the spacial locations of the $L^2$ norms over the channel directions.}. 

\section{Remarks and comparison to concurrent work}
\label{comparisonsimple}

We have addressed the main problems of weight sharing and dependence on the number of classes. As mentioned earlier, \cite{GRRR} have recently studied the former problem independently of us. It is interesting to provide a comparison of their and our main results, which we do briefly here and in more detail in the Appendix.

The bound in \cite{GRRR} scales like \small  $\mathcal{C}\sqrt{\frac{\mathcal{W}(\sum_{l=1}^Ls_l-\log(\gamma))+\log(1/\delta)}{n}}, $ \normalsize where $s_l$ is an upper bound on the spectral norm of the matrix corresponding to the $l^{th}$ layer, $\gamma$ is the margin, and $\mathcal{W}$ is the number of parameters, taking weight sharing into account by counting each parameter of convolutional filters only once. The idea of the proof is to bound the Lipschitz constant of the map from the set of weights to the set of functions represented by the network, and use dimension-dependent results on covering numbers of finite dimensional function classes. Remarkably, this doesn't require chaining the layers, which results in a lack of a non logarithmic dependence on the product of spectral norms. Note that the term \small $\sum_{l=1}^Ls_l$ \normalsize comes from a log term via the inequality \small $\prod(1+s_i)\leq \exp(\sum s_i)$\normalsize.

On the other hand, the bound scales at least as the square root of the number of parameters, even if the weights are arbitrarily close to initialisation. 
In contrast, our bound~\eqref{posthocasymptsimple} scales like \small $O(\sqrt{1/n})$ \normalsize up to log terms when the weights approach initialisation. Furthermore, if we fix an explicit upper bound on the relevant norms (cf.Theorem~\ref{firstmilestone}) \footnote{The bounds in~\cite{GRRR} and other works deal only with this case, leaving the post hoc case to the reader}, \textbf{the bound then converges to zero} as the bounds on the norms go to zero. In a refined treatment via the NTK literature (cf.~\cite{Aroratwolayers}), explicit bounds would be provided for those quantities via other tools.

Finally, note that the main advantages and disadvantages of our bounds compared to~\cite{GRRR} are connected through a tradeoff in the proof where one can decide which quantities go inside or outside the log. In particular, it is not possible to combine the advantages of both. We refer the reader to Appendix~\ref{comparisondeep} for a more detailed explanation.

\section{Conclusion}
\label{conclusion}
We have proved norm-based generalisation bounds for deep neural networks with significantly reduced dependence on certain parameters and architectural choices. On the issue of class dependency, we have completely bridged the gap between the states of the art in shallow methods and in deep learning. Furthermore, we have, simultaneously with~\cite{GRRR}, provided the first satisfactory answer to the weight sharing problem in the Rademacher analysis of neural networks. Contrary to independent work, our bounds are norm-based and are negligible at initialisation. 

\bibliography{bibliography}

\appendix


	\onecolumn

	\setcounter{secnumdepth}{1}

	\section{Notation and general results}
	\label{precise}

	\subsection{Notation}
	
	We use the following notation to represent linear layers with weight sharing such as convolution. Let $x\in \mathbb{R}^{U\times w}$, $A\in \mathbb{R}^{m\times d}$ and $S^1,S^2,\ldots,S^O$ be $O$ ordered subsets of $(\{1,2,\ldots,w\}\times \{1,2,\ldots,U\})$ each of cardinality $d$\footnote{We suppose for notational simplicity that all convolutional filters at a given layer are of the same size. It is clear that the proof applies to the general case as well.}, where we will denote by $S^o_i$ the $i^{th}$ element of $S^o$. We will denote by $\Lambda_A(x)$ the element of $\mathbb{R}^{m\times O}$ such that $\Lambda_A(x)_{j,o}=\sum_{i=1}^d X_{S^o_i}A_{j,i}$.   In a typical example the sets $S^1,S^2,\ldots,S^O$ represent the image patches where the convolutional filters are applied, and $\Lambda$ would be represented via the "tf.nn.conv2d" function in Tensorflow.  We will also write $\tilde{A}^l$ for the matrix in $\mathbb{R}^{(U_{l-1}w_{l-1})\times(O_{l-1}m_l)}$ that represents the convolution operation $\Lambda_{A^l}$.

	To represent a full network, we suppose that we are given a number $L\in\mathbb{N}$ of layers, $7L+2$ numbers \small $m_1,m_2,\ldots,m_L,d_1,d_2,\ldots,d_L,\rho_1,\rho_2,\ldots,\rho_L,$ \\$w_0,w_1,\ldots,w_L,U_0,U_1,\ldots,U_L,O_1,O_2,\ldots,O_L,\\ \text{and}\quad k_1,k_2,\ldots,k_L$\normalsize, as well as $\sum_{l=0}^LO_l$ ordered sets $S^{l,o}\subset \{1,2,\ldots,U_l\}\times \{1,2,\ldots,w_l\}$ (for $l\leq L$, $o\leq O_l$), and $L-1$ functions $G_l:\mathbb{R}^{m_{l}\times O_{l-1}}\rightarrow \mathbb{R}^{U_l\times w_l}$ (for $l=1,2,\ldots,L$)  which are $\rho_l$-Lipschitz with respect to the $l^\infty$ norm.
	
	The architecture above can help us represent a feedforward neural network involving possible (intra-layer) weight sharing as  \begin{align}
	&F_{A^1,A^2,\ldots,A^L}:\mathbb{R}^{U_0\times w_0}\rightarrow \mathbb{R}^{U_L\times w_l}:x \mapsto \nonumber\\&(G_{L}\circ\Lambda_{A^L}\circ G_{L-1}\circ \Lambda_{A^{L-1}}\circ \ldots G_1\circ \Lambda_{A^1})(x)\nonumber,
	\end{align}\normalsize
	where for each $l\leq L$, the weight $A^l$ is a matrix in $\mathbb{R}^{ m_l\times d_l}$. We will also write $F^{l_1\rightarrow l_2}$ for the subnetwork that computes the function from the $l_1^{th}$ layer activations to the $l_2^{th}$ layer activations.
	As usual, offset terms can be accounted for by adding a dummy dimension of constants at each layer (this dimension must belong to $S^{l,o}$ for each $o$). We provide a quick table of notations in Figure~\ref{thiss}.

	\begin{figure}
		\includegraphics[width=\linewidth]{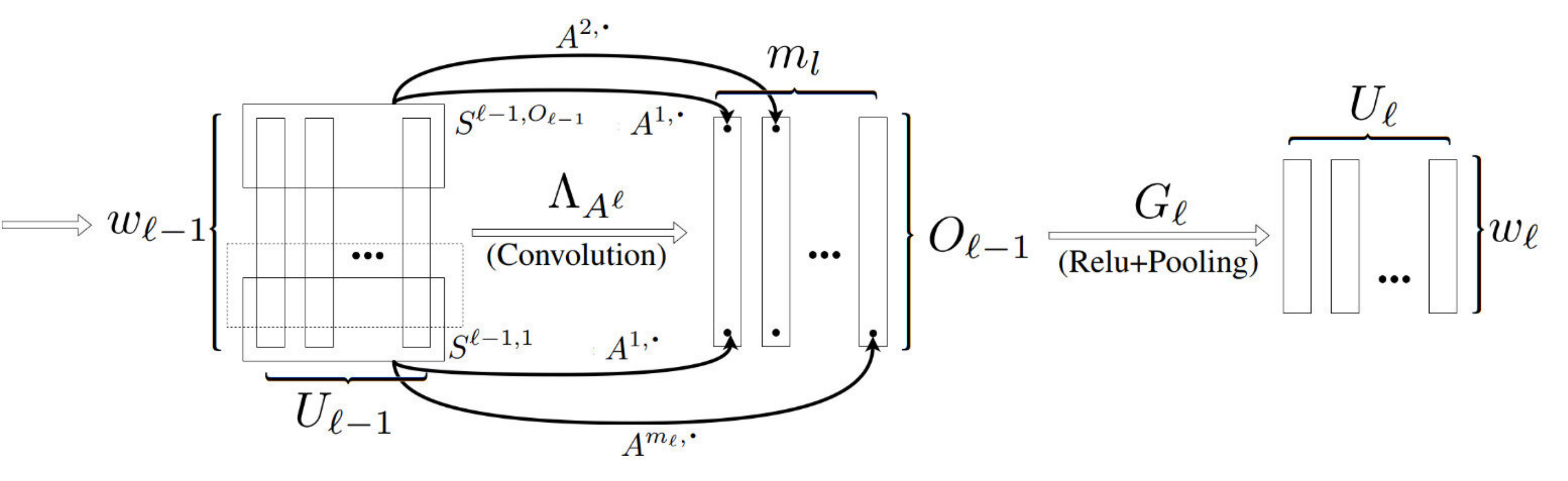}
		\caption{\small Illustration of architecture for one layer}
		\label{archit}
	\end{figure}

	\begin{table*}[h!] \small
		\centering
		\begin{tabular}{c|c}
			
			Notation & Meaning  \\
			\midrule
			$G_l$    & Activation + pooling at layer $l$\\
			$A^l$ & Filter matrix at layer $l$ \\
			$\Lambda_{A^l}$ & Convolution operation relative to filter matrix $A^l$\\
			$\tilde{A}^l$ & Matrix representing $\Lambda_{A^l}$
			(Has repeated weights in conv. net)\\
			$O_l$ & Number of convolutional patches at layer $l$\\
			$m_l$ & \# of channels at layer $l$ before nonlinearity\\ &(=\# of output channels at layer $l-1$) \\
			$\gamma$ & Classification margin\\
			$S^{l,o}$ & $o^{th}$ convolutional patch at layer $l$\\
			$w_{l}$ & Number of spatial dimensions at layer $l$ (No pooling $\implies$ $w_l=O_{l-1}$)\\
			$U_l$ & Number of channels after nonlinearity\\
			$\rho_l$ & Lipschitz constant of $G_l$\\
			$W_l=U_lw_l$ & Width (after pooling) at layer $l$ \\
			$W=\max_l W_l$ & Maximum network width (after any pooling)\\
			$\bar{W}=\max_l O_{l-1}m_l$ & Maximum network width (before any pooling)\\
			$\mathcal{W}$ & Total number of parameters \\
			$d_l$ & Size of convolutional patches corresponding to the operation $\Lambda_{A^l}$\\
			$|\nbull|_{\infty,l}$ & Max $l^2$ norm of a pixel across channels\\
			$|\nbull|_{l}$ & Max $l^2$ norm of a convolutional patch (across channels)
		\end{tabular}
		\caption{\small Table of notations for quick reference}
		\label{thiss}
	\end{table*}
	
	Some key aspects of our proofs and general results rely on using the correct norms in activation spaces. On each activation space $\mathbb{R}^{w_l\times U_l}$, we will make use of the following three norms:
	\begin{enumerate}
		\item The $l^\infty$ norm: $|x|_{\infty}=\sup_{u\in \{1,\ldots, w_l\}\times \{1,\ldots,U_l\}} |x_u|$
		\item The $|\nbull|_{l}$ norm: $|x|_{l}=\sup_{o\leq O_l}\sqrt{\sum_{i\leq  O_l}x_i^2}$, i.e. the maximum $l^2$ norm of a single convolutional patch.
		\item The $|\nbull|_{\infty,l}$ norm: $|x|_{\infty,l}=\sup_{o\leq w_l}\sqrt{\sum_{k\leq U_l}x_{o,k}^2}$, the maximum $l^2$ norm of a single pixel viewed as a vector over channels.
	\end{enumerate}
	
	\begin{remark}
		In covering number arguments, we will use the same notation to refer to the norms on (activation,input) space induced by the above norms after taking a supremum over inputs.
	\end{remark}
	
	\subsection{Main result with Global Lipschitz bounds}
	Both of the results in the Section "contributions in a Nutshell" follow from the following result.
	
	\begin{theorem}
		\label{posthocasymptlip}
		Assume that pooling does not occur over different channels, with probability $\geq 1-\delta$, every network $F_{\mathcal{A}}$ with weight matrices $\mathcal{A}=\{A^1,A^2,\ldots,A^L\}$ and every margin $\gamma>0$ satisfy:
		\begin{align}
		&\mathbb{P}\left(\argmax_j(F_{\mathcal{A}}(x)_j)\neq y\right)\leq \widehat{R}_\gamma(F_{\mathcal{A}})\nonumber \\&\quad \quad +\widetilde{\mathcal{O}}\left(   \frac{(R_{\mathcal{A}}+L)}{ \sqrt{n}}\log(\bar{W})+\sqrt{\frac{ \log(1/\delta) }{ n }  }\right),
		\end{align}
		where  $W$ is the maximum number of neurons in a single layer (after pooling) and  $R_{\mathcal{A}}:=$ \begin{align*}
		& \bigg[\sum_{l=1}^{L-1}\left(B_{l-1}(X)\|A^l-M^l\|_{2,1}\max_{\tilde{l}>l}\frac{\rho_{l\rightarrow \tilde{l}}}{B_{\tilde{l}}(X)}\right)^{2/3} +\left(\frac{B_{L-1}(X)}{\gamma}\|A^L-M^L\|_{\Fr}\right)^{2/3}\bigg]^{3/2},
		\end{align*}
		\normalsize
		where for any layer $l_1$, 	\small $B_{l_1}(X):=\max_{i}\left|F^{0\rightarrow l_l}(x_i)\right|_{l_1}$\normalsize  denotes the maximum $l^2$ norm of any convolutional patch of the layer $l_1$ activations, over all inputs. $B_L(X)=\gamma$.
		
		Here $\rho_{l_1\rightarrow  l_2}$ is the Lipschitz constant of the map \small $F^{l_1\rightarrow  l_2}$ \normalsize with respect to the norms $|\nbull|_{\infty,l_1} $ and $|\nbull|_{l_2}$
	\end{theorem}

	\section{The one layer case}
	
	A key aspect of the proof is that we can use proposition~\ref{MaureySup} to obtain an $L^\infty$-covering of the map represented by a convolutional layer. Indeed, by viewing each (sample, convolutional patch, output channel) trio as an individual data point, we can, for each $\epsilon$, find $\mathcal{N}_\epsilon$ filters $f_1,\ldots,f_{\mathcal{N}_\epsilon}$ with $\|f_i\|_{Fr}\leq a \quad \forall i$ such for any convolutional map represented by the filter $f$ (with $\|f\|_{Fr}\leq a$), there exists a $u_{f}\in \{1,2,\ldots, \mathcal{N}_\epsilon\}$ such that for any input $x_i$, any convolutional patch $S$, and any output channel $j$, the outputs of $f$ and $f_{u_f}$ corresponding to this (input, patch, channel) combination differ by less than $\epsilon$. More precisely, we can now prove Proposition~\ref{suplinn}
	
	\begin{proof}[Proof of proposition~\ref{suplinn}]
		
		This follows immediately from Lemma~\ref{MaureySup} applied to the $nmU$ data points in $\mathbb{R}^{d\times m}$ (considered as a simple vector space with the Hadamard product used as the scalar product) defined by, for all $\delta\in \{1,2,\ldots,m\}\times \{1,2,\ldots,d\}$,
		$(x_{u,i,j})_\delta=(X^u_j)_{\delta_2}$  for $\delta_1=i$ and $ (x_{u,i,j})_\delta=0$ otherwise,  and the function class $$\{F_{A}:\mathbb{R}^{d\times m}\rightarrow\mathbb{R}:x\mapsto x \odot A; \|A\|_2\leq a\},$$
		where $\odot$ denotes the Hadamard product.

	\end{proof}
	
	Before we proceed, we will need the following Proposition:
	
	\begin{proposition}[cf.~\cite{book,bartlet98,pisier}]
		\label{posthocify}
		Let $B_\beta$ denote the ball of radius $\beta$ in $\mathbb{R}^d$ with respect to the $L^1$ norm.  We have
		\begin{equation}
		\log\left(\mathcal{N}(B_{\beta},\epsilon,\|\nbull\|_{2})\right)\leq \left \lceil\frac{\beta^2}{\epsilon^2}\right \rceil\log(2d)
		\end{equation}

	\end{proposition}

	\begin{proof}
		Wlog, $\beta=1$.	Let $e_1,\ldots,e_d$ denote the standard basis in $\mathbb{R}^d$, we will show that for any integer $k\in \mathbb{N}$ and any $a=(a_1,a_2,\ldots,a_d)\in\mathbb{R}_{+}^d$ with $\sum_i a_i\leq 1$, there exists $(k_1,k_2,\ldots,k_d)$ with $k_i\in \mathbb{N}$ such that $\kappa:=\sum_{i=1}^d e_i\frac{k_i}{k}$ satisfies $$\|\kappa-a\|_{2}^2\leq \frac{1}{k}.$$
		
		Let $(W_1,\ldots,W_k)$ be $k$ iid random variables with $\mathbb{P}(W_1=e_i)=a_i$. Define $W=\frac{1}{k}\sum_{i=1}^k W_i$. We have $\mathbb{E}(W)=\mathbb{E}(W_1)=\sum_{i=1}^da_ie_i=a$. Thus we have
		
		\begin{align}
		\mathbb{E}(\|W-a\|_2^2)&=\frac{1}{k^2}\mathbb{E}\left(  \sum_{i=1}^k \|a-W_i\|^2+\sum_{i\neq j}\langle a-W_i,a-W_j\rangle\right)\nonumber\\
		&=\frac{1}{k^2}\mathbb{E}\left(  \sum_{i=1}^k \|a-W_i\|^2\right)=\frac{1}{k}\mathbb{E}\left(   \|a-W_1\|^2\right)=\frac{1}{k}\left(\mathbb{E}(\|W_1\|^2)-\|a\|^2\right)\nonumber\\
		&\leq \frac{1}{k}\mathbb{E}(\|W_1\|^2)=\frac{1}{k}.
		\end{align}
		
		By the probabilistic method, it follows that there is a choice $(k_1,k_2,\ldots,k_d)$ such that $\|\kappa-a\|_{2}^2\leq \frac{1}{k},$
		as expected.
		
		It follows that one can find a cover of the ball $B_{\beta}$ with size $\mathcal{N}$, where $\mathcal{N}$ is the number of choices of $(k_1,k_2,\ldots,k_d)$ with $k_i\in \mathbb{Z}$ and $\sum_{i=1}^d|k_i|=k$ and $k=\lceil \frac{\beta^2}{\epsilon^2}\rceil$. There are $
		2^d\left( \begin{array}{ccc} d+k-1\\d-1\end{array}\right)\leq (2d)^k
		$ such choices. The result follows.
		
	\end{proof}

	With this in our toolkit, we can prove the extention of the one layer case to the $L^{2,1}$ norm with an extra covering umber argument:

	\begin{proof}[Proof of Proposition~\ref{suplin}]

		First, note that the case $m=1$ follows from Proposition~\ref{suplinn} (also with $m=1$).
		
		Thus any set of $m$ positive real numbers $a=(a_1,a_2,\ldots,a_m)$ and any $(\epsilon_1,\epsilon_2,\ldots,\epsilon_m)$ with $\epsilon_i\leq \epsilon$ for all $i$, we can find covers $\mathcal{C}_i\subset \{A\in \mathbb{R}^d: \|A\|_{2}\leq a_i   \}$ such that for all $A_1,A_2,\ldots,A_m\in \mathbb{R}^{d}$ such that $\|A_i\|_{2}\leq a_i \quad \forall i\leq m$, there exist $\bar{A}_1,\bar{A}_2,\ldots,\bar{A}_m\in \mathbb{R}^{d}$ such that $\|\bar{A}_i\|_{2}\leq a_i \quad \forall i\leq m$ and $\|XA_i-X\bar{A}_i\|_{\infty}\leq   \epsilon_1$, and for all $i$, $\log(\#\mathcal{C}_i)\leq \frac{36a_i^2b^2}{\epsilon_i^2}\log_2\left[\left(\frac{8a_ib}{\epsilon}+6\right)nU+1\right]$ (since $\epsilon_i\leq \epsilon$).
		
		Writing $a=\sum_{i=1}^ma_i$, the cardinality of $\mathcal{C}_{(a_1,a_2,\ldots,a_m)}=\mathcal{C}_{1}\times \mathcal{C}_2,\ldots\times \mathcal{C}_m$ is bounded above by \begin{align}
		\label{coarsesum}
		36b^2\log_2\left[\left(\frac{8ab}{\epsilon}+6\right)nU+1\right]\sum_{i=1}^m \frac{a_i^2}{\epsilon_i^2}.
		\end{align}
		
		and the product cover is an $\epsilon$ cover of $\{XA : A\in \mathbb{R}^{d\times m}, \|A_{\nbull,i}\|_{2}\leq a_i\forall i\}$ with respect to the $\|\nbull\|_{*}$ norm, where $\epsilon^2=\sum_{i=1}^m\epsilon_i^2$.

		Applying the above to $3\epsilon/4$ and calculating the Lagrangian to optimize over the $\epsilon_i$'s, we obtain the condition $\left(\frac{-2a_i^2}{\epsilon^3}\right)\propto 2 \epsilon_i$, which yields $\epsilon_i=(3/4)\epsilon\frac{\sqrt{a_i}}{\sum a_i}=\epsilon \sqrt{a_i}/a$, which pluggind back into formula~\eqref{coarsesum}, yields,
		\begin{align}
		\label{one}
		\log(\#(\mathcal{C}_{1}\times \mathcal{C}_2,\ldots,\mathcal{C}_m))&\leq
		36\frac{a^2b^2}{(3/4)^2\epsilon^2}\log_2\left[\left(\frac{8ab}{\epsilon}+7\right)nU\right]\\
		&\leq 64\frac{a^2b^2}{\epsilon^2}\log_2\left[\left(\frac{8ab}{\epsilon}+7\right)nU\right]
		\end{align}
		Of course, we do not know in advance the choice of $(a_1,\ldots,a_m)$ such that $\sum_{i=1}^ma_i=a$, so we must take extra steps to make the cover post hoc with respect to this choice. To do this we can choose the set $\mathcal{D}$ to be an $\epsilon/4b$ cover of $(a_1,\ldots,a_m): \sum_{i=1}^m a_i=a$ with respect to the $L^1$ norm. By Proposition~\ref{posthocify}, we can ensure $\#(\mathcal{D})\leq \left\lceil 16a^2b^2/\epsilon^2\right \rceil \log(2m)\leq 64a^2b^2/\epsilon^2\log(m)$ (since we can assume $m\geq 2$ and wlog $16a^2b^2/\epsilon^2\geq 1$).
		
		Clearly, the cardinality of the union $\cup_{a\in \mathcal{D}}\mathcal{C}_1\times \ldots\times \mathcal{C}_{m}$ is bounded by  $$\frac{64a^2b^2}{\epsilon^2}\log_2\left[\left(\frac{8ab}{\epsilon}+7\right)mnU\right],$$ so we only need to show that it constitutes an $\epsilon $ cover of  $\{XA : A\in \mathbb{R}^{d\times m}, \|A\|_{2,1}\leq a\}$. To see this, let $A\in \mathbb{R}^{d\times m}$ be given with $\|A\|_{2,1}\leq a$. Pick $(a_1,a_2,\ldots,a_m)$ to be the closest element of $\mathcal{D}$ to $(\|A_{\nbull,1}\|_{2},\|A_{\nbull,2}\|_{2},\ldots,\|A_{\nbull,m}\|_{2})$ in terms of the $L^\infty$ norm. Then pick the element $\tilde {A}$ of $ \mathcal{C}_{(a_1,a_2,\ldots,a_m)}$ closest to $\bar{A}\in \mathbb{R}^{d\times m}   $ defined by, $\forall i$,  $\bar{A}_{\nbull,i}=A_{\nbull,i}$ if $\|A_{\nbull,i}\|_{2}\leq a_i$ and $\bar{A}_{\nbull,i}=\frac{a_i}{\|A_{\nbull,i}\|_2}A_{\nbull,i}$ otherwise.
		Clearly we have $\|XA-X\tilde{A}\|_{*}\leq \|XA-X\bar{A}\|_{*}+\|X\bar{A}-X\tilde{A}\|_{*}\leq \epsilon/5+4\epsilon/5=\epsilon$, which completes the proof.
		
	\end{proof}
	
	As a corollary of the above, we are now in a position to prove the one layer case.
	
	\begin{corollary}
		\label{onestep}
		Let $n,O,m,w,U$ be natural numbers,  let $G: \mathbb{R}^{ O}\times \mathbb{R}^{m}\rightarrow\mathbb{R}^{ w}\times \mathbb{R}^{U}$ be $\rho$-Lipschitz with respect to the $|\nbull|_{*}$ norms defined by $|x|_{*}=\sup_{i\in \{1,2\ldots,O\}}\sqrt{\sum_{j=1}^m x_{ij}^2}$ and $|x|_{*}=\sup_{i\in \{1,2\ldots,w\}}\sqrt{\sum_{j=1}^U x_{ij}^2}$ for $x\in  \mathbb{R}^{ O}\times \mathbb{R}^{m}$ and $x\in  \mathbb{R}^{ w}\times \mathbb{R}^{U}$ respectively (for instance, a combination of purely spacial pooling and elementwise relu satisfies this condition with $\rho=1$). We also write $|\nbull|_{*}$ for the norm $\sup_{i\leq n} |X_{i,\nbull}|_{*}$ for $X\in \mathbb{R}^n \times  \mathbb{R}^{ w}\times \mathbb{R}^{U}$ or  $X\in \mathbb{R}^n \times \mathbb{R}^{ O}\times \mathbb{R}^{m}$.
		For any $X\in \mathbb{R}^{n\times O\times d}$ such that $\|X^{i,o,\nbull}\|_{2}^2\leq b^2$ ($\forall i,o$), we have for any fixed reference matrix $M$:
		\begin{align}
		\log \mathcal{N}  \left(     \left\{ G( XA) : A\in \mathbb{R}^{d\times m}, \|A-M\|_{2,1}\leq a\right\},\epsilon,  |\nbull|_{*}   \right) \leq \frac{ 64 a^2b^2}{\epsilon^2\rho^2}  \log_2\left(\frac{8abnmO}{\epsilon\rho}+7Omn\right)
		\end{align}

	\end{corollary}

	\section{Generalisation bound for fixed norm constraints}
	
	Once the one layer case is taken care of, we will now need to chain the covering number bounds of each layer, taking care to control the excess $|\nbull|$ norms at each intermediary layer. To this effect, we have the following Proposition.
	
	\begin{proposition}
		\label{Chainingprop}
		Let $L$ be a natural number and $a_1,\ldots,a_L>0$ be real numbers. Let $\mathcal{V}_0,\mathcal{V}_1,\ldots,\mathcal{V}_{L}$ be $L+1$ vector spaces each endowed with two norms $|\nbull|_{\infty,l}$  and  $|\nbull|_l$ for $0\leq l\leq L$. 
		Let $B_1,B_2,\ldots,B_L$ be $L$ vector spaces with norms $\|\nbull\|_{1},\|\nbull\|_{2},\ldots,\|\nbull\|_{L}$ and $\mathcal{B}_1,\mathcal{B}_2,\ldots,\mathcal{B}_L$ be the balls of radii $a_1,a_2,\ldots,a_L$ in the spaces $B_1,B_2,\ldots,B_L$ with the norms $\|\nbull\|_{1},\|\nbull\|_{2},\ldots,\|\nbull\|_{L}$ respectively\footnote{The proof works with $\mathcal{B}_1,\mathcal{B}_2,\ldots,\mathcal{B}_L$ being arbitrary sets, but we formulate the problem as above to aid the intuitive comparison with the areas of application of the Proposition.}.

		Suppose also that for each $l\in \{1,2,\ldots,L\}$ we are given an operator $F^l: \mathcal{V}_{l-1}\times B_{l}\rightarrow \mathcal{V}_{l}: (x,A)\rightarrow F^l_{A}(x)$, continuous with respect to the norms $\|\nbull\|_{l}$ and $|\nbull|_{l}$.
		For each $l_1,l_2$ with $l_2>l_1$ and each  $\mathcal{A}^{l_1,l_2}=(A^{l_1+1},\ldots,A^{l_2})\in \mathcal{B}^{l_1,l_2}:=\mathcal{B}_{l_1+1}\times  \mathcal{B}_{l_1+2} \times \ldots  \mathcal{B}_{l_2}$, let us define $$F^{l_1\rightarrow l_2}_{\mathcal{A}^{l_1,l_2}}:\mathcal{V}_{l_1}\rightarrow \mathcal{V}_{l_2}: x\rightarrow F^{l_1\rightarrow l_2}_{\mathcal{A}^{l_1,l_2}}(x)= F^{l_2}_{A^{l_2}}\circ \ldots\circ  F^{l_1}_{A^{l_1}}(x),$$
		for all $l$, $\mathcal{F}^l_{\mathcal{A}^L}=F^{0\rightarrow l}_{\mathcal{A}^L}$ and $\mathcal{F}_{\mathcal{A}}=\mathcal{F}^L_{\mathcal{A}}.$

		For each  $\mathcal{A}^L=(A^1,A^2,\ldots,A^L)\in \mathcal{B}^L:=\mathcal{B}_1\times  \mathcal{B}_2 \times \ldots,  \mathcal{B}_L$, and for each $l\in \{1,2,\ldots,L\}$, the (worst case) Lipschitz constant of $F^{l_1\rightarrow l_2}$ with respect to the norms $\|\nbull\|_{\infty,l_1}$ and $|\nbull|_{l_2}$ is  denoted by $\rho^{\mathcal{A}}_{l_1\rightarrow l_2}$.

		We suppose the following conditions are satisfied:
		
		For all $l\in \{1,2,\ldots,L\}$, all $b>0$,  all $Z1,Z_2,\ldots,Z_n\in\mathcal{V}_{l-1}$ such that $|Z_i|_{l-1}\leq b \quad \forall i$ and all $\epsilon>0$, there exists a subset $\mathcal{C}_l(b,Z,\epsilon)\subset \mathcal{B}_l$ such that
		\begin{align}
		\log(\#\left(\mathcal{C}_l(b,Z,\epsilon)\right))\leq \frac{C_{l,\epsilon,n}a_l^2b^2}{\epsilon^2},
		\end{align}
		where $C_{l,\epsilon,n}$ is some function of $l,\epsilon,n,$
		and, for all $A\in \mathcal{B}_l$, there exists an $\bar{A}\in \mathcal{C}_l(b,\epsilon)$ such that
		\begin{align}
		\left|F^l_{A}(Z)-F^l_{\bar{A}}(Z)\right|_{\infty,l}\leq  \epsilon \quad \quad \forall i.
		\end{align}

		For any $0<\epsilon<1$, any $b=(b_0,b_{1},b_2,\ldots,b_L)$ with $b_l\geq 1 \quad \forall l$ and $b_L=1$, any set of positive number $\rho_{l+}$ (for $l\leq L$)  and for any $x_1,\ldots,x_n\in \mathcal{V}_0$ such that $|x_i|_{0}\leq b_0 \quad \forall i$, there exists a subset $\mathcal{C}_{\epsilon,b,n}$ of $\mathcal{B}^L$ such that for all $\mathcal{A}=(A^1,A^2,\ldots,A^L)\in \mathcal{B}:=\mathcal{B}^L$ such that $\rho^{\mathcal{A}}_{l_1\rightarrow l_2}\leq \rho_{l_1}b_{l_2}\quad \forall l_2\geq l_1$\footnote{Note that $\rho_{l\rightarrow l}$ is not necessarily $1$, as the norms $\|\nbull\|_{\infty,l}$ and $\|\nbull\|_{l}$ are different.}, there exists a $\bar{\mathcal{A}}\in \mathcal{C}_{\epsilon,b,X}$ such that, for any $i$ such that 	$\left| \mathcal{F}^l_{\mathcal{A}}(x_i)\right|_{l}\leq b_l \quad \forall l$, we have
		
		\begin{align}
		\label{toolll}
		\left|  F^{0\rightarrow l}_{\mathcal{A}^l}(x_i)-   F^{0\rightarrow l}_{\bar{\mathcal{A}}^l}(x_i) \right|_{l}&\leq \epsilon b_l\quad \quad (\forall l<L, \forall i) \quad \text{and}\nonumber \\
		\left| \mathcal{F}^l_{\bar{\mathcal{A}}}(X)\right|_{l}&\leq 2b_l.
		\end{align}

		Furthermore, we have
		\begin{align}
		\log \#(\mathcal{C}_{\epsilon,b,X})&\leq 4 \left[  \sum_{l=1}^L \left(\frac{C_{l,\epsilon}^{\frac{1}{2}}a_lb_{l-1}\rho_l}{\epsilon}\right)^{\frac{2}{3}} \right]^3
		\leq 4\frac{L^2}{\epsilon^2}  \sum_{l=1}^L \left(\frac{C_{l,\epsilon}^{\frac{1}{2}}a_lb_{l-1}\rho_l}{\epsilon}\right)^{2}.
		\end{align}
		
	\end{proposition}
	
	\begin{proof}

		For $l=1,\ldots,L$, let $\epsilon_l=\frac{\epsilon \alpha_l}{ \rho_{l}}$, where the $\alpha_l>0$ will be determined later satisfying $\sum_{l=1}^L\alpha_l=1$.

		For any $X=(x_1,\ldots,x_n)$, we define the covers $\mathcal{D}_l\subset \mathcal{B}_l$ for $l\leq L$  by induction by
		$\mathcal{D}_l=\cup_{\mathcal{A}\in \mathcal{D}_1\times \ldots\times \mathcal{D}_{l-1}}\mathcal{C}(2b_l,\{F^{0\rightarrow l-1}_{\mathcal{A}}(x_i):i\leq n\},\epsilon_l)$. Let us write also $\mathcal{D}:=\mathcal{D}_{1}\times \mathcal{D}{2}\ldots \mathcal{D}_L$, and let us write $d_l$ for the cardinality of $\mathcal{D}_l$.
		We prove the equations~\eqref{toolll} by induction. The case $l=1$ follows directly from the definition of $\mathcal{D}_1$ and the assumption on the $x_i$'s. For the induction case, let us assume the inequalities hold for each $u\leq l-1$. We have for any $j\leq n$,

		\begin{align}
		\left|  F^{0\rightarrow l}_{\mathcal{A}}(x_j)-   F^{0\rightarrow l}_{\bar{\mathcal{A}}}(x_j) \right|_{l}&\leq \sum_{i=1}^L	\left|  F^{0\rightarrow l}_{\left(\bar{A}_1,\bar{A}_2,\ldots,\bar{A}_{i-1},A_i,\ldots,A_l\right)}(x_j)-   F^{0\rightarrow l}_{\left(\bar{A}_1,\bar{A}_2,\ldots,\bar{A}_{i},A_{i+1},\ldots,A_i\right)}(x_j) \right|_{l}\nonumber \\
		&\leq  \sum_{i=1}^l\rho^{\mathcal{A}}_{i\rightarrow l}	\left|  F^{0\rightarrow i }_{(\bar{A}_1,\bar{A}_2,\ldots,\bar{A}_{i-1},A_{i})}(x_j)-   F^{0\rightarrow i}_{\left(\bar{A}_1,\bar{A}_2,\ldots,\bar{A}_{i}\right)}(x_j) \right|_{i}\nonumber \\
		&\leq \sum_{i=1}^l \rho^{\mathcal{A}}_{i\rightarrow l}\frac{\epsilon \alpha_i}{\rho_{i}}\leq\sum_{i=1}^l \rho^{\mathcal{A}}_{i\rightarrow l}\frac{\epsilon \alpha_i}{\rho^{\mathcal{A}}_{i\rightarrow l}/b_l} \leq \epsilon b_l,
		\end{align}
		where at the first line, we have used the triangle inequality, at the second line, the definition of $\rho_{i\rightarrow l}$, and at the last line, the fact that $\rho^{\mathcal{A}}_{l_1\rightarrow l_2}\leq \rho_{l_1}\quad \forall l_2\geq l_1$. Now, by the triangle inequality, we obtain:
		
		\begin{align}
		\left	|F^{0\rightarrow l}_{\bar{\mathcal{A}}}(x_j)\right|_{l}\leq \left	|F^{0\rightarrow l}_{\mathcal{A}}(x_j)\right|_{l}+\epsilon b_l\leq 2b_l,
		\end{align}
		which concludes the proof by induction.

		To finish the proof of the proposition, we just need to calculate the bound on the caridinality of $\mathcal{D}$:
		
		\begin{align}
		\label{countingchain}
		\log(\#(\mathcal{D}))&=\sum_{l=1}^L\log(\#(\mathcal{D}_l))=\sum_{l=1}^L\log(d_l)\nonumber\\&\leq \sum_{l=1}^L 4\frac{C_{l,\epsilon}a_l^2b_{l-1}^2}{\epsilon_l^2}=\sum_{l=1}^L 4\frac{C_{l,\epsilon}a_l^2b_{l-1}^2}{(\frac{\alpha_l\epsilon}{\rho_{l}})^2}\nonumber\\
		&=\frac{4}{\epsilon^2}\sum_{l=1}^L\frac{C_{l,\epsilon}a_l^2b_{l-1}^2\rho_{l}^2}{\alpha_l^2}.
		\end{align}

		Optimizing over the $\alpha_l$'s subject to $\sum_{l=1}^L\alpha_l=1$, we find the Lagrangian condition $$\left(-\frac{2C_{l,\epsilon} a_l^2\rho_{l}^2}{\alpha_l^3}\right)_{l=1}^L\propto (1)_{l=1}^L,$$
		yielding $$\alpha_l= \frac{(\sqrt{C_{l,\epsilon}}a_l\rho_{l})^{\frac{2}{3}}}{\sum_{i=1}^L(\sqrt{C_i}a_i\rho_{i})^{\frac{2}{3}}}. $$
		Substituting back into equation~\eqref{countingchain}, we obtain

		\begin{align*}
		\log \# \left( \{F_{\mathcal{A}}(X):\mathcal{A}\in \mathcal{D} \}   \right) & \leq4\left[\sum_{i=1}^L\left(\frac{\sqrt{C_i}a_ib_{i-1}\rho_{i}}{\epsilon}\right)^{\frac{2}{3}}\right]^{2}\sum_{l=1}^L\left(\frac{\sqrt{C_{l,\epsilon}} a_lb_{l-1}\rho_{l}}{\epsilon}\right)^{2-4/3} \\
		&=\frac{4}{\epsilon^2}\left[\sum_{i=1}^L\left(\sqrt{C_i}a_ib_{i-1}\rho_{i}\right)^{\frac{2}{3}}\right]^{3},
		\end{align*}
		as expected. The last inequality follows from Jensen's inequality.

	\end{proof}

	The next step is to use the above, together with the classic Rademacher theorem~\ref{rademachh} and Dudley's Entropy integral, to obtain a result about large margin multi-class classifiers.

	\begin{theorem}
		\label{firstmilestone}
		Suppose we have a $K$ class classification problem and are given $n$ i.i.d. observations $(x_1,y_1),(x_2,y_2),\ldots,(x_n,y_n)\in \mathbb{R}^{U_0\times w_0}\otimes \{1,2,\ldots,K\}$ drawn from our ground truth distribution $(X,Y)$, as well as a fixed architecture as described in Section~\ref{precise}, where we assume the last layer is fully connected and has width $K$ and corresponds to scores for each class. Suppose also that with probability one $|x|_{0}\leq b_0$. Suppose we are given numbers $a_1,a_2,\ldots,a_L$, $\beta=(b_0,b_1,\ldots,b_L=\gamma)$ and $\rho_{l_1\rightarrow l_2}$ (for $l_1,l_2\leq L$). For any $\delta>0$,  with probability $>1-\delta$ over the draw of the training set, for any network $\mathcal{A}=(A^1,A^2,\ldots,A^L)$ satisfying $\forall l: \|(A^l)^{\top}\|_{2,1}\leq a_l\land \rho^{\mathcal{A}}_{l_1\rightarrow l_2}\leq \rho_{l_1\rightarrow l_2}$, we have
		\begin{align}
		&\mathbb{P}\left( \argmax_{j\in\{1,2,\ldots,K\}} (F_L(x))_{j}\neq y  \right)\nonumber \\&\leq \frac{n-\#(I_{\beta,\gamma})}{n}+ \frac{8}{n}+\frac{1536}{\sqrt{n}}R\left[\log_2(32\Gamma n^2+7\bar{W}n)\right]^{\frac{1}{2}}\log(n)+3\sqrt{\frac{\log(\frac{2}{\delta })}{2n}},
		\end{align}
		where $$I=\left\{ i\leq n: (f(x_i) )_{y_i}-\max_{j\neq y_i}( (f(x_i) )_{j})>\gamma \land \forall l\leq L : \left|F^{0\rightarrow l} (x_i)\right|_{l}\leq b_l              \right\},$$

		\begin{align}
		R^{2/3}&:= \sum_{l=1}^L \left(  a_lb_{l-1} \rho_{l+}         \right)^{2/3}, \quad \text{and} \nonumber\\
		\Gamma&:=\max_{l=1}^L\left(b_{l-1}a_lO_{l-1}m_l\rho_{l+}\right),
		\end{align}
		$\rho_{l+}=\max_{i=l}^L\frac{\rho_{l\rightarrow i}}{b_i}.$
	\end{theorem}
	
	\begin{proof}

		As explained in the sketch in the main text, we apply the Rademacher theorem to the loss function:

		\begin{align}
		l(x_i,y_i)=\max\left[\sup_{l=1}^{L-1} \lambda_{B_l}\left(\left|F^{0\rightarrow l}(x_i)\right|_{l}-2B_{l}\right),\lambda_{\gamma}\left(\max_{j\neq y}(F_{\mathcal{A}}(x))_{j}-F_{\mathcal{A}}(x)_{y}\right)     \right]
		\end{align}

		Writing $\mathcal{H}$ for the function class defined by $l(x,F^{\mathcal{A}}(x))$ for $F^{\mathcal{A}}$ satisfying the conditions of the Theorem, since $Y\mapsto\max_{j\neq y}(Y)_{j}-Y_{y}$ is $2$ Lipschitz with respect to the $l^\infty$ norm,  and $l(x,y)=1$ for any $x$ such that there exists $l$ such that $|F^{0\rightarrow l}(x)|_{l}\geq 2b_l$,   Propositions~\ref{Chainingprop} and~\ref{onestep} guanrantee that the covering number of $\mathcal{H}$ satisfies
		\begin{align}
		\label{cov}
		\log(\mathcal{N}(\mathcal{H},\epsilon))&\leq 4\times 64\times 2^2 \left[ \sum_{l=1}^L \left(   \frac{a_l(2b_{l-1})}{\epsilon } \rho_{l+}   \log_2\left(8\frac{a_l(2b_{l-1})nm_lO_{l-1}}{(\epsilon/\rho_{l+})}+7O_{l-1}m_ln\right)      \right)^{2/3}\right]^3\nonumber \\
		&\leq 2^{12} R^2 \log_2(32\Gamma n/\epsilon+7\bar{W}n)
		\end{align}
		
		Applying the Rademacher Theorem~\ref{rademachh}, we now obtain

		\begin{align}
		\label{initial}
		&\mathbb{P}\left( \argmax_{j\in\{1,2,\ldots,K\}} (F_L(x))_{j}\neq y  \right)\leq \mathbb{E}\left( l(x,y) \right)\nonumber\\
		&\leq \frac{\sum_{i=1}^nl(x_i,y_i)}{n} +3\sqrt{\frac{\log(\frac{2}{\delta})}{2n}}+2\hat{\rad}_n(l(x,y) )\nonumber \\
		&\leq \frac{n-\#(I)}{n}+3\sqrt{\frac{\log(\frac{2}{\delta})}{2n}}+2\hat{\rad}_n(l(x,y) ).
		\end{align}

		Now, by Dudley's Entropy integral~\ref{Dudley1} with $\alpha=\frac{1}{n}$, we have
		\begin{align}
		\hat{\rad}_n(l(x,y))&\leq 4\alpha+\frac{12}{\sqrt{n}}  \int_{\alpha}^1      \sqrt{\log \mathcal{N}(\mathcal{F}|S,\epsilon,\|\nbull\|_{p}) }     \nonumber\\
		&\leq \frac{4}{n}+\sqrt{R}\sqrt{2^{12}}\frac{12}{\sqrt{n}}\int_{\frac{1}{n}}^{1}\frac{\sqrt{\log_2(32\Gamma n/\epsilon+7\bar{W}n)}}{\epsilon}\nonumber\\
		&\leq \frac{4}{n}+\sqrt{R}\frac{768}{\sqrt{n}}\sqrt{\log_2(32\Gamma n^2+7\bar{W}n)}\log(n)\nonumber\\
		\end{align}
		
		Plugging this back into equation~\ref{initial}, we obtain the desired result.

	\end{proof}

	\section{Proofs of post hoc result and asymptotic results}
	
	The next step from Theorem~\ref{firstmilestone} to Theorem~\ref{posthocasymptlip} is now mostly a question of applying classical techniques, namely,  splitting the space of possible choices of parameters $(a_1,\ldots,a_L,b_0,\ldots,b_L,\gamma,\rho_{0\rightarrow 1},\ldots,)$ into different regions and using a union bound. The following Lemma summarizes the techniques in question:

	\begin{lemma}
		\label{unionbound}
		Let $R_{\mathcal{N}}$ denote a random variable indexed some finite dimensional vector $\mathcal{N}$\footnote{We assume that the map from $\mathcal{N}$ to $R_{\mathcal{N}}$ is sufficiently well behaved for the random variables  $R_{\mathcal{N}}$ to be jointly defined on the same probability space for all values of $\mathcal{N}$, as is the case where $\mathcal{N}$ represents the parameters of a neural network and $R_{\mathcal{N}}$ is the misclassification probability on a test point}. Let $\gamma_1(\mathcal{N}),\ldots,\gamma_l(\mathcal{N}),N_1(\mathcal{N}),\ldots,N_L(\mathcal{N})$ be some (positive) statistics of $\mathcal{N}$. Suppose there exists a function $F:\mathbb{R}_{+}^{l\times L}\rightarrow \mathbb{R}$ which is monotonically decreasing in $\gamma_i$ for all $i\leq l$ and monotonically increasing in $N_i$ for $i \leq L$, such that the following statement holds:
		
		For any $\gamma_1,\ldots,\gamma_l,N_1,\ldots,N_L$, and for any $\delta>0$,  with probability $>1-\delta$, we have that for any $\mathcal{N}$ such that $\gamma_i(\mathcal{N})\leq \gamma_i$ for all $i$ and $N_i(\mathcal{N})\leq N_i$ for all $i$,
		\begin{align}
		\label{root}
		R_{\mathcal{N}}&\leq f(\gamma_1,\ldots,\gamma_l,N_1,\ldots,N_L)+C_1\sqrt{\frac{\log(1/\delta)}{C_2}},
		\end{align}
		for some constants $C_1,C_2$.
		
		For any fixed choice of $\beta_1,\ldots,\beta_L$ and $\kappa_1,\ldots,\kappa_l\in \mathbb{N}$, we have for any $\delta>0$ that with probability greater than $1-\delta$, for any $\mathcal{N}$,
		\begin{align}
		R_{\mathcal{N}}&\leq f(\min(\gamma_1(\mathcal{N})/2,\frac{1}{\kappa_1}),\ldots,\min(\gamma_l(\mathcal{N})/2,\frac{1}{\kappa_l}),N_1(\mathcal{N})+\beta_1,\ldots,N_L(\mathcal{N})+\beta_L)\nonumber\\
		&+\frac{C_1}{\sqrt{C_2}}\sqrt{\log(1/\delta)+\sum_{i=1}^l\log(2\kappa_i(\mathcal{N})/\gamma_i(\mathcal{N}))+2\sum_{i=1}^L\log\left(2+\frac{N_i(\mathcal{N})}{\beta_i}\right)}.
		\end{align}
		
	\end{lemma}
	
	\begin{proof}
		For any $j_1,\ldots,j_l, n_1,\ldots,n_L\in\mathbb{N}$, define $$\delta_{j_1,\ldots,j_l,n_1,\ldots,n_L}=\frac{\delta}{\prod_{i=1}^l2^{j_i}\prod_{i=1}^L n_i(n_i+1)}.$$
		
		Define $\gamma_i^{j_i}=\frac{\kappa_{i}}{2^{j_i}}$ and $N^{n_i}_{i}=n_i\beta_i$. By applying equation~\eqref{root}, we obtain that with probability $\geq 1-\delta_{j_1,\ldots,j_l,n_1,\ldots,n_L}$,
		
		\begin{align}
		R_{\mathcal{N}}&\leq f(\gamma^{j_1}_1,\ldots,\gamma^{j_l}_l,N^{n_1}_1,\ldots,N^{n_L}_L)+C_1\sqrt{\frac{\log(1/\delta_{j_1,\ldots,j_l,n_1,\ldots,n_L})}{C_2}},
		\end{align}
		
		Note that $\sum_{j_1,\ldots,j_l,n_1,\ldots,n_L}\delta_{j_1,\ldots,j_l,n_1,\ldots,n_L}=\delta$.

		Thus, by a union bound, we obtain that with probability $\geq 1-\delta$, for any choice of $j_1,\ldots,j_l, n_1,\ldots,n_L\in\mathbb{N}$, we have
		\begin{align}
		\label{postroot}
		R_{\mathcal{N}}&\leq f(\gamma^{j_1}_1,\ldots,\gamma^{j_l}_l,N^{n_1}_1,\ldots,N^{n_L}_L)+C_1\sqrt{\frac{\log(1/\delta_{j_1,\ldots,j_l,n_1,\ldots,n_L})}{C_2}}.
		\end{align}
		
		For any network $\mathcal{N}$, we can apply this for the choice of $j_1,\ldots,j_l$ which are smallest whist still guaranteeing $$\frac{1}{\gamma_i(\mathcal{N})}\leq \frac{1}{\gamma_{i}^{j_i}}=\frac{2^{j_i}}{\kappa_{i}},$$
		and $$N_{i}(\mathcal{N})\leq n_i\beta_i.$$
		
		For this choice, we have $\gamma_{i}^{j_i}\geq \min(\gamma(\mathcal{N})/2,\frac{1}{\kappa_i}))$ for all $i\leq l$ and $N_{i}^{n_i}\leq N_{i}(\mathcal{N})+\beta_i$, thus by the properties of $f$,
		\begin{align}
		\label{firstconcl}
		&f(\gamma^{j_1}_1,\ldots,\gamma^{j_l}_l,N^{n_1}_1,\ldots,N^{n_L}_L)\nonumber \\
		&\leq f(\min(\gamma_1(\mathcal{N})/2,\frac{1}{\kappa_1}),\ldots,\min(\gamma_l(\mathcal{N})/2,\frac{1}{\kappa_l}),N_1(\mathcal{N})+\beta_1,\ldots,N_L(\mathcal{N})+\beta_L).
		\end{align}
		
		Furthermore, we also have $2^{j_i}\leq 2\kappa_i/\gamma_i(\mathcal{N})$ and $n_i(n_i+1)\leq \left(1+\frac{N_i(\mathcal{N})}{\beta_i}\right)\left(2+\frac{N_i(\mathcal{N})}{\beta_i}\right)\leq \left(2+\frac{N_i(\mathcal{N})}{\beta_i}\right)$, and thus
		
		\begin{align}
		\label{secnconcl}
		C_1\sqrt{\frac{\log(1/\delta_{j_1,\ldots,j_l,n_1,\ldots,n_L})}{C_2}}\leq \frac{C_1}{\sqrt{C_2}}\sqrt{\log(1/\delta)+\sum_{i=1}^l\log(2\kappa_i(\mathcal{N})/\gamma_i(\mathcal{N}))+2\sum_{i=1}^L\log\left(2+\frac{N_i(\mathcal{N})}{\beta_i}\right)}.
		\end{align}
		Plugging equations~\eqref{firstconcl} and~\eqref{secnconcl} back into equation~\eqref{postroot} yields the desired result.
		
	\end{proof}

	Using this, we obtain the following:

	\begin{theorem}
		\label{posthocresult}
		Suppose we have a $K$ class classification problem and are given $n$ i.i.d. observations $(x_1,y_1),(x_2,y_2),\ldots,(x_n,y_n)\in \mathbb{R}^{U_0\times w_0}\otimes \{1,2,\ldots,K\}$ drawn from our ground truth distribution $(X,Y)$, as well as a fixed architecture as described in Section~\ref{precise}, where we assume the last layer is fully connected and has width $K$ and corresponds to scores for each class. For any $\delta>0$,  with probability $>1-\delta$ over the draw of the training set, for any network $\mathcal{A}=(A^1,A^2,\ldots,A^L)$ we have
		\begin{align}
		&\mathbb{P}\left( \argmax_{j\in\{1,2,\ldots,K\}} (F_L(x))_{j}\neq y  \right)\nonumber \\&\leq \frac{n-\#(I_{\beta,\gamma})}{n}+ \frac{8}{n}+\frac{1536}{\sqrt{n}}R_{\mathcal{A}}\left[\log_2(32\Gamma n^2+7\bar{W}n)\right]^{\frac{1}{2}}\log(n)+3\sqrt{\frac{\log(\frac{2}{\delta })}{2n}}\\
		&+   3\sqrt{\frac{1}{n}\sum_{l=1}^L\log\left(2+\frac{B_{l-1}(X)}{L}\right) +\log\left(2+\frac{\|(A-M)^\top\|_{2,1}}{L}\right)+ \log\left(2+\frac{\rho^{\mathcal{A}}_{l+}}{L}\right)    }     ,
		\end{align}

		where $$I=\left\{ i\leq n: (f(x_i) )_{y_i}-\max_{j\neq y_i}( (f(x_i) )_{j})>\gamma \land \forall l\leq L : \left|F^{0\rightarrow l} (x_i)\right|_{l}\leq B_{l}(X)            \right\},$$

		\begin{align}
		R_{\mathcal{A}}^{2/3}&:= \sum_{l=1}^L \left(  (\|(A-M)^\top\|_{2,1}+\frac{1}{L})(B_{l-1}(X)+\frac{1}{L}))(\rho^{\mathcal{A}}_{l+}  +\frac{1}{L})     \right)^{2/3}, \quad \text{and} \nonumber\\
		\Gamma&:=\max_{l=1}^L\left((B_{l-1}(X)+\frac{1}{L})(\|(A-M)^\top\|_{2,1}+\frac{1}{L})O_{l-1}m_l(\rho^{\mathcal{A}}_{l+}+\frac{1}{L})\right),
		\end{align}
		$\rho^{\mathcal{A}}_{l+}=\max_{i=l}^L\frac{\rho^{\mathcal{A}}_{l\rightarrow i}}{B_i(X)}.$
		
		Here, $X$ denotes the design matrix containing the sample points, $B_L=\gamma>0$ is arbitrary and $B_l(X)$ for $l\leq L-1$ can be chosen arbitrarily so that $B_l(X)\geq 1$in a way that depends on $X$, with the choice $B_l(X)=\max\left(\max_{i\leq n}\left|F_{0\rightarrow l}^{\mathcal{A}}(x_i)\right|_l,1\right)$ yielding  $$I=\left\{ i\leq n: (f(x_i) )_{y_i}-\max_{j\neq y_i}( (f(x_i) )_{j})>\gamma   \right\}.$$
	\end{theorem}
	
	\begin{proof}
		We split the space using Lemma~\ref{unionbound} for the parameters  $B_0,B_1,\ldots,B_{L-1}$, $\rho^{\mathcal{A}}_{l+}$. Note that the bound is increasing in all of those parameters, se we can treat them all as a "$N_i$"s from Lemma~\ref{unionbound}. Setting all the $k_i$'s to $L$ yields the result. Note that the dependence on $\gamma$ is hidden in the definition of $\rho^{\mathcal{A}}_{l+}$. The bound doesn't go to zero as $\gamma$ as a result of the need to estimate the risk of intermediary activations being too large.
		
		As explained in the main text, if one is willing to forgo the gains obtained from the sparsity of the connections inside the definition of $B_l$, then one can obtain bounds that scale like $1/\gamma$.
	\end{proof}
	
	\begin{proof}[Proof of Theorem~\ref{posthocasymptlip}]
		This is simply a question of reducing to the $\tilde{O}$ notation. Note that $\log(n^2)$, $\frac{1}{\sqrt{n}}$, $\Gamma$  are all $\tilde{O}(\frac{R_\mathcal{A}}{\sqrt{n}})$, and asymptotically, $\log_2(32\Gamma n^2+7\bar{W}n)\leq \log_2(32\Gamma n^2)+\log_2(7\bar{W}n)$, thus we only have to take care of the last line. For this, note that each of the concerned log terms inside the square root are also $\tilde{O}(\frac{R_\mathcal{A}}{\sqrt{n}})$, yielding the result (with a factor of $\sqrt{R}$ from the sum).

	\end{proof}

	Similarly, in the case where the spectral norms are well controlled, we have the following result:
	
	\begin{theorem}
		\label{posthocresultspecs}
		Suppose we have a $K$ class classification problem and are given $n$ i.i.d. observations $(x_1,y_1),(x_2,y_2),\ldots,(x_n,y_n)\in \mathbb{R}^{U_0\times w_0}\otimes \{1,2,\ldots,K\}$ drawn from our ground truth distribution $(X,Y)$, as well as a fixed architecture as described in Section~\ref{precise}, where we assume the last layer is fully connected and has width $K$ and corresponds to scores for each class. For any $\delta>0$,  with probability $>1-\delta$ over the draw of the training set, for any network $\mathcal{A}=(A^1,A^2,\ldots,A^L)$ we have
		\begin{align}
		&\mathbb{P}\left( \argmax_{j\in\{1,2,\ldots,K\}} (F_L(x))_{j}\neq y  \right)\nonumber \\&\leq \frac{n-\#(I_{\beta,\gamma})}{n}+ \frac{8}{n}+\frac{3072}{\sqrt{n}}R_{\mathcal{A}}\left[\log_2(64\Gamma n^2+7\bar{W}n)\right]^{\frac{1}{2}}\log(n)+3\sqrt{\frac{\log(\frac{2}{\delta })}{2n}}+\sqrt{\log\left(\frac{2}{\gamma n}\right)}\\
		&+   3\sqrt{\frac{1}{n}\sum_{l=1}^L\log\left(2+LB_{l-1}(X)\right) +\log\left(2+\|(A-M)^\top\|_{2,1}\right)+ \log\left(2+L\|\tilde{A}\|_{\sigma'}\right)    }     ,
		\end{align}

		where $$I=\left\{ i\leq n: (f(x_i) )_{y_i}-\max_{j\neq y_i}( (f(x_i) )_{j})>\gamma \land \forall l\leq L : \left|F^{0\rightarrow l} (x_i)\right|_{l}\leq B_{l}(X)            \right\},$$

		\begin{align}
		R_{\mathcal{A}}^{2/3}&:= \sum_{l=1}^L \left(  (\|(A-M)^\top\|_{2,1}+\frac{1}{L})(B_{l-1}(X)+\frac{1}{L}))\max_{l'}\frac{\prod_{i=l+1}^{l'}(\|\tilde{A}\|_{\sigma'}+\frac{1}{L}) }{B_{l'}(X)}\right)^{2/3}, \quad \text{and} \nonumber\\
		\Gamma&:=\max_{l=1}^L\left((B_{l-1}(X)+\frac{1}{L})(\|(A-M)^\top\|_{2,1}+\frac{1}{L})O_{l-1}m_l\max_{l'}\frac{\prod_{i=l+1}^{l'}(\|\tilde{A}\|_{\sigma'}+\frac{1}{L}) }{B_{l'}(X)}\right).
		\end{align}

		Here, $X$ denotes the design matrix containing the sample points, $B_L=\gamma>0$ is arbitrary and $B_l(X)$ for $l\leq L-1$ can be chosen arbitrarily so that $B_l(X)\geq 1$ in a way that depends on $X$, with the choice $B_l(X)=\max\left(\max_{i\leq n}\left|F_{0\rightarrow l}^{\mathcal{A}}(x_i)\right|_l,1\right)$ yielding  $$I=\left\{ i\leq n: (f(x_i) )_{y_i}-\max_{j\neq y_i}( (f(x_i) )_{j})>\gamma   \right\}.$$
	\end{theorem}
	\begin{proof}
		We apply theorems~\ref{firstmilestone} and Lemma~\ref{unionbound} for the parameters $\gamma,B_0,\ldots,B_L,a_1,\ldots,a_L,s_1,s_2,\ldots,s_L$ where $s_i$ is a bound on $\|\tilde{A}_i\|_{\sigma'}$: note first that if $\|\tilde{A}_i\|_{\sigma'}\leq s_i$ for all $i$, then $\rho^{\mathcal{A}}_{l_1\rightarrow l_2}\leq \prod_{l=l_1+1}^{l_2}\rho_l s_i$ for all $l_1\leq l_2$, and if furthermore $\gamma^{\mathcal{A}}\geq \bar{\gamma}$ and $B_l(X)\geq \bar{B}_l$, for all $l\leq L-1$, we have $\rho^{\mathcal{A}}_{l_1+}\leq \sqrt{l_1}\max_{l_2}\frac{\prod_{l=l_1+1}^{l_2}\rho_l s_i}{\bar{B}_{l_2}}$, with $\bar{B}_{L}=\bar{\gamma}$.
		Thus we can apply Lemma~\ref{unionbound}\footnote{Technically, we are applying a slight variation where $f$ can have factors that are either increasing or decreasing in the same variable $N_i$, and the term $f(\min(\gamma_1(\mathcal{N})/2,\frac{1}{\kappa_1}),\ldots,\min(\gamma_l(\mathcal{N})/2,\frac{1}{\kappa_l}),N_1(\mathcal{N})+\beta_1,\ldots,N_L(\mathcal{N})+\beta_L)$ is replaced by an evaluation of $f$ where each factor involving $N_i$ chooses whichever of $(N_i,N_i+\beta_i)$ maximises it} with the $\beta_i=\frac{1}{L}$ and $\gamma$ being treated as a decreasing variable with $k=n$, to obtain the required result ($\bar{\gamma}\geq \gamma^{\mathcal{A}}/2$ and $\bar{B}_l\geq B_l/2$, furthermore, the case $\gamma\leq \frac{1}{2n}$ is trivial since the RHS is $\geq 1$ ).
		
	\end{proof}

	We can now proceed to the proof of theorem~\ref{posthocasymptsimple}:
	
	\begin{proof}
		
		The only difference between this proof and that of Theorem~\ref{posthocasymptlip} is in the treatment of the sum of log terms at the last line. For this, note that $\|(A-M)^{\top}\|_{2,1}\leq \|(A)^{\top}\|_{2,1}+\|(M)^{\top}\|_{2,1} \leq \|(M)\|^{\top}_{2,1}+\sqrt{\bar{W}}\|(A)^{\top}
		\|_{2}\leq \|(M)^{\top}\|_{2,1}+\bar{W}\|A\|_{\sigma'}\leq \|(M)^{\top}\|_{2,1}+\bar{W}^{3/2}\|A\|_{\sigma'}) $, thus
		
		$	   \sqrt{\frac{1}{n}\sum_{l=1}^L \log\left(2+L\|(A-M)^\top\|_{2,1}\right)    } $ is $\tilde{O}(\prod_{l=1}^L\|\tilde{A}\|_{\sigma'}\rho_l)$, which is $\tilde{O}(\frac{R_{\mathcal{A}}}{\sqrt{n}})$ as expected.
		
	\end{proof}
	
	\section{Simpler results with explicit norms}
	\label{withnorms}
	
	In this Section, we show slight variations of our bounds sticking closer to~\cite{Spectre} by only using norms involved at each individual layer or pair of layer. Theorem~\ref{posthocasymptsimple} follows from the theorems below. Whilst the results don't seem to follow directly from the above, the treatment is extremely similar.  Suppose we are given some norms $|\nbull|_{\mathcal{L}_l}$ on the activation spaces and some norms $|\nbull|_{\mathcal{L}_l^*}$ on the weight spaces such that $|\nbull|_{\mathcal{L}_l}\leq |\nbull|_{l}$ and the Lipschitz constant of $\Lambda_l(A)$ with respect to $|\nbull|_{\mathcal{L}_{l-1}}$ and $|\nbull|_{\mathcal{L}_l}$ is bounded by $|A|_{\mathcal{L}^*}$ and $|\nbull|_{\infty,l}\leq \sqrt{k_l}|\nbull|_{\mathcal{L}_l}$
	
	Note that for any $s$, a simple argument on internal vs. external covering numbers shows that Theorem~\ref{onestep} can be adapted to yield a cover such that $\|\tilde{A}\|_{\sigma'}\leq s$ for any $A$ in the cover, at the cost of a factor of $2$ in $\epsilon$.

	We have the following simplified variation of Theorem~\ref{Chainingprop}:

	\begin{proposition}
		\label{Chainingprop-}
		Let $L$ be a natural number and $a_1,\ldots,a_L>0$ be real numbers. Let $\mathcal{V}_0,\mathcal{V}_1,\ldots,\mathcal{V}_{L}$ be $L+1$ vector spaces, with arbitrary norms $|\nbull|_0,|\nbull|_1,\ldots,|\nbull|_L$, let $B_1,B_2,\ldots,B_L$ be $L$ vector spaces with norms $\|\nbull\|_{1},\|\nbull\|_{2},\ldots,\|\nbull\|_{L}$ and $\mathcal{B}_1,\mathcal{B}_2,\ldots,\mathcal{B}_L$ be the balls of radii $a_1,a_2,\ldots,a_L$ in the spaces $B_1,B_2,\ldots,B_L$ with the norms $\|\nbull\|_{1},\|\nbull\|_{2},\ldots,\|\nbull\|_{L}$ respectively. Suppose also that for each $l\in \{1,2,\ldots,L\}$ we are given an operator $F^l: \mathcal{V}_{l-1}\times B_{l}\rightarrow \mathcal{V}_{l}: (x,A)\rightarrow F^l_{A}(x)$. Suppose also that there exist real numbers $\rho_1,\rho_2,\ldots,\rho_L>0$ such that the following properties are satisfied.
		\begin{enumerate}
			\item For all $l\in \{1,2,\ldots,L\}$ and for all $A\in \mathcal{B}_l$, the Lipschitz constant of the operator $ F^l_{A}$ with respect to the norms $|\nbull|_{l-1}$ and $|\nbull|_{l}$ is less than $\rho_l$.
			\item For all $l\in \{1,2,\ldots,L\}$, all $b>0$, and all $\epsilon>0$, there exists a subset $\mathcal{C}_l(b,\epsilon)\subset \mathcal{B}_l$ such that
			\begin{align}
			\log(\#\left(\mathcal{C}_l(b,\epsilon)\right))\leq \frac{C_{l,\epsilon}a_l^2b^2}{\epsilon^2},
			\end{align}
			where $C_{l,\epsilon}$ is some function of $l,\epsilon$ and,
			and, for all $A\in \mathcal{B}_l$ and all $X\in\mathcal{V}_{l-1}$ such that $|X|_{i-1}\leq b$, there exists an $\bar{A}\in \mathcal{C}_l(b,\epsilon)$ such that
			\begin{align}
			\left|F^l_{A}(X)-F^l_{\bar{A}}(X)\right|_{l}\leq  \epsilon.
			\end{align}
		\end{enumerate}
		
		For each $l$ and each  $\mathcal{A}^l=(A^1,A^2,\ldots,A^l)\in \mathcal{B}^l:=\mathcal{B}_1\times  \mathcal{B}_2 \times \ldots,  \mathcal{B}_l$, let us define $$F^l_{\mathcal{A}^l}:\mathcal{V}_0\rightarrow \mathcal{V}_L: x\rightarrow F^l_{\mathcal{A}^l}(x)= F^l_{A^l}\circ \ldots\circ  F^{2}_{A^2}\circ F^{1}_{A^1},$$
		and $F_{\mathcal{A}}=F^L_{\mathcal{A}^L}.$
		For each $\epsilon>0$, there exists a subset $\mathcal{C}_\epsilon$ of $\mathcal{B}^L$ such that for all $\mathcal{A}=(A^1,A^2,\ldots,A^L)\in \mathcal{B}:=\mathcal{B}^L$, there exists an $\bar{\mathcal{A}}\in \mathcal{C}_\epsilon$ such that the following two conditions are satisfied.
		
		\begin{align}
		\left|  F^l_{\mathcal{A}^l}(X)-   F^l_{\bar{\mathcal{A}}^l}(X) \right|_{l}&\leq\frac{\epsilon}{\prod_{j=l+1}^L\rho_j} \quad \quad (\forall l\leq L),\quad \text{and }\\
		\log \#(\mathcal{C})&\leq \frac{|X|_{1}^2}{\epsilon^2}\prod _{i=1}^L \rho_i^2 \left[  \sum_{l=1}^L \left(\frac{C_{l,\epsilon}^{\frac{1}{2}}a_l }{\rho_l}\right)^{\frac{2}{3}} \right]^3\nonumber \leq L^2\frac{|X|_{1}^2}{\epsilon^2}\prod _{i=1}^L \rho_i^2 \sum_{l=1}^L \left(\frac{C_{l,\epsilon}^{\frac{1}{2}}a_l }{\rho_l}\right)^{2}.
		\end{align}
		In particular, for any $X\in \mathcal{V}_0$ and any $\epsilon>0$, the following bound on the $(\epsilon,|\nbull|_{L})$-covering number of $\{F_{\mathcal{A}}(X):\mathcal{A}\in \mathcal{B}^{L} \}$ holds.
		
		\begin{align}
		\log \mathcal{N} \left( \{F_{\mathcal{A}}(X):\mathcal{A}\in \mathcal{B} \}    ,\epsilon,|\nbull|_{L}  \right) & \leq L^2\frac{|X|_{0}^2}{\epsilon^2}\prod _{i=1}^L \rho_i^2 \sum_{i=1}^L \left(\frac{C_{l,\epsilon}^{\frac{1}{2}}a_l }{\rho_i}\right)^{2} .
		\end{align}
		
	\end{proposition}
	
	\begin{proof}

		For $l=1,\ldots,L$, let $\epsilon_l=\frac{\epsilon \alpha_l}{ \prod_{i=l+1}^L \rho_i}$, where the $\alpha_l>0$ will be determined later satisfying $\sum_{l=1}^L\alpha_l=1$.

		Using the second assumption, let us pick for each $l$ the subset $\mathcal{C}_l=\mathcal{C}_l\left(|X|_0\prod_{i=1}^{l-1}\rho_i,\epsilon_l\right)$ satisfying the assumption. Let us define also the set $ \mathcal{C}:=\mathcal{C}_1\times \mathcal{C}_{2}\times \ldots\times \mathcal{C}_L\subset \mathcal{B} $.
		
		\textit{Claim 1}
		
		For all $A\in \mathcal{B}$, there exists a $\bar{\mathcal{A}}\in \mathcal{C}$ such that for all $l\leq L$,
		\begin{align}
		\left|  F^l_{\mathcal{A}}(X)-   F^l_{\bar{\mathcal{A}}}(X) \right|_{l}\leq\frac{\epsilon}{\prod_{j=l+1}^L\rho_j}.
		\end{align}
		\textit{Proof of Claim 1}
		
		To show this, observe first that for any $1\leq l\leq L$ and for any $A^1,A^2,\ldots,A^l$,
		\begin{align}
		\label{trivospectre}
		\left|F^{l-1}\circ \ldots \circ F^2\circ F^1 (X)\right|_l\leq |X|_0 \prod_{i=1}^{l-1}\rho_i,
		\end{align}
		and therefore, by definition of $\mathcal{C}_l$, we have that for any $A^1,A^2,\ldots,A^{l-1}$, $\{F_{A^1,A^2,\ldots,A^{l-1},A^l}(X):A^l\in \mathcal{C}_l\}$ is an $(\epsilon_l,|\nbull|_{l})$ cover of $\{F_{A^1,A^2,\ldots,A^{l-1},A^l}(X):A^l\in \mathcal{B}_l\}$.

		Let us now fix $A^1,A^2,\ldots,A^L$ and define $\bar{A}_l\in \mathcal{C}_l$ inductively so that $F^l_{\bar{A}_l}(F_{\bar{A}_1,\bar{A}_2,\ldots,\bar{A}_{l-1}}(X))$ is an element of $\{  F^l_{A}(F_{\bar{A}_1,\bar{A}_2,\ldots,\bar{A}_{l-1}}(X)):A\in \mathcal{C}_l\}$  minimising the distance to $F_{\bar{A}_1,\bar{A}_2,\ldots,\bar{A}_{l-1},A_l}(X)$ in terms of the $|\nbull|_{l}$ norm.

		We now have for all $l\leq L$:
		
		\begin{align}
		\left|  F_{\mathcal{A}}(X)-   F_{\bar{\mathcal{A}}}(X) \right|_{l}&\leq \sum_{i=1}^l 	\left|  F_{\left(\bar{A}_1,\bar{A}_2,\ldots,\bar{A}_{i-1},A^i,\ldots,A^l\right)}(X)-   F_{\left(\bar{A}_1,\bar{A}_2,\ldots,\bar{A}_{i},A^{i+1},\ldots,A^l\right)}(X) \right|_{l}\nonumber \\
		&\leq  \sum_{i=1}^l \prod_{j=i+1}^l\rho_j	\left|  F_{(\bar{A}_1,\bar{A}_2,\ldots,\bar{A}_{i-1},A^i)}(X)-   F_{\left(\bar{A}_1,\bar{A}_2,\ldots,\bar{A}_{i}\right)}(X) \right|_{l}\nonumber \\
		&\leq \sum_{i=1}^l \prod_{j=i+1}^l\rho_j \epsilon_i =\frac{1}{\prod_{j=l+1}^L\rho_j}\sum_{i=1}^l \epsilon \alpha_i\leq \frac{\epsilon}{\prod_{j=l+1}^L\rho_j},
		\end{align}
		as expected.
		
		This concludes the proof of the claim.
		
		To prove the proposition, we now simply need to calculate the cardinality of $\mathcal{C}$:
		
		\begin{align}
		\label{countingchain-}
		&\log \mathcal{N} \left( \{F_{\mathcal{A}}(X):\mathcal{A}\in \mathcal{B} \}    ,\epsilon,|\nbull|_{L}  \right)\leq \log(\#(\mathcal{C}))
		\leq \sum_{l=1}^L \log(\#(\mathcal{C}_l))\nonumber \\&=\sum_{l=1}^L  \frac{C_{l,\epsilon} a_l^2\left (|X|_0\prod_{i=1}^{l-1}\rho_i\right)^2}{\epsilon_l^2} \leq \frac{1}{\epsilon^2}\sum_{l=1}^L  \frac{C_{l,\epsilon} a_l^2\left (|X|_0\prod_{i=1}^{l-1}\rho_i\right)^2\left(\prod_{i=l+1}^L \rho_i\right)^2}{\alpha_l^2}
		\nonumber \\
		&=\frac{|X|_0^2\prod_{i=1}^L\rho_i^2}{\epsilon^2}\sum_{l=1}^L\frac{C_{l,\epsilon} a_l^2}{\rho_l^2\alpha_l^2}.
		\end{align}
		
		Optimizing over the $\alpha_l$'s subject to $\sum_{l=1}^L\alpha_l=1$, we find the Lagrangian condition $$\left(-\frac{2C_{l,\epsilon} a_l^2/\rho_l^2}{\alpha_l^3}\right)_{l=1}^L\propto (1)_{l=1}^L,$$
		yielding $$\alpha_l= \frac{(\sqrt{C_{l,\epsilon}}a_l/\rho_l)^{\frac{2}{3}}}{\sum_{i=1}^L(\sqrt{C_i}a_i/\rho_i)^{\frac{2}{3}}}. $$
		Substituting back into equation~\eqref{countingchain}, we obtain
		\begin{align*}
		\log \mathcal{N} \left( \{F_{\mathcal{A}}(X):\mathcal{A}\in \mathcal{B} \}    ,\epsilon,|\nbull|_{L}  \right) & \leq\frac{|X|_0^2\prod_{i=1}^L\rho_i^2}{\epsilon^2}\left[\sum_{i=1}^L\left(\frac{\sqrt{C_i}a_i}{\rho_i}\right)^{\frac{2}{3}}\right]^{2}\sum_{l=1}^L\left(\frac{\sqrt{C_{l,\epsilon}} a_l}{\rho_l}\right)^{2-4/3} \\
		&\leq \frac{|X|_0^2\prod_{i=1}^L\rho_i^2}{\epsilon^2}\left[\sum_{l=1}^L\left(\frac{\sqrt{C_{l,\epsilon}} a_l}{\rho_l}\right)^{2/3}\right]^3,
		\end{align*}
		as expected.
		The second inequality follows by Jensen's inequality.
	\end{proof}

	Using this, we obtain similarly:
	\begin{theorem}
		\label{posthocresultNORMS}
		Suppose we have a $K$ class classification problem and are given $n$ i.i.d. observations $(x_1,y_1),(x_2,y_2),\ldots,(x_n,y_n)\in \mathbb{R}^{U_0\times w_0}\otimes \{1,2,\ldots,K\}$ drawn from our ground truth distribution $(X,Y)$, as well as a fixed architecture as described in Section~\ref{precise}, where we assume the last layer is fully connected and has width $K$ and corresponds to scores for each class. For any $\delta>0$,  with probability $>1-\delta$ over the draw of the training set, for any network $\mathcal{A}=(A^1,A^2,\ldots,A^L)$
		\begin{align}
		&\mathbb{P}\left( \argmax_{j\in\{1,2,\ldots,K\}} (F_L(x))_{j}\neq y  \right)\nonumber \\&\leq \frac{n-\#(I_{\beta,\gamma})}{n}+ \frac{8}{n}+\frac{1536}{\sqrt{n}}R_{\mathcal{A}}\left[\log_2(32\Gamma n^2+7\bar{W}n)\right]^{\frac{1}{2}}\log(n)+3\sqrt{\frac{\log(\frac{2}{\delta })}{2n}}+\sqrt{\log\left(\frac{2}{\gamma n}\right)}\\
		&+   3\sqrt{\frac{1}{n}\sum_{l=1}^L\log\left(2+\sup_{i}\|x_i\|_{\mathcal{L}_0}\right) +\log\left(2+\frac{\|(A-M)^\top\|_{2,1}}{L}\right)+ \log\left(2+\frac{\|A\|_{\mathcal{L}_l^*}}{L}\right)    }     ,
		\end{align}

		where $$I=\left\{ i\leq n: (f(x_i) )_{y_i}-\max_{j\neq y_i}( (f(x_i) )_{j})>\gamma       \right\},$$

		\begin{align}
		R_{\mathcal{A}}&:=\frac{1}{\gamma}L\left(\sup_{i}\|x_i\|_{\mathcal{L}_0}+1\right)\left[\sum_{l=1}^L \left((\|(A-M)^\top\|_{\Fr}+\frac{1}{L})\sqrt{k_l}\prod_{i\neq l}\left(\|A\|_{\mathcal{L}_i^*}+\frac{1}{L}\right) \right)^{2}\right]^{1/2}, \quad \text{and} \nonumber\\
		\Gamma&:=\left(\sup_{i}\|x_i\|_{\mathcal{L}_0}+1\right)\max_{l=1}^L\left(O_{l-1}m_l\left((\|(A-M)^\top\|_{2,1}+\frac{1}{L})\prod_{i\neq l}\left(\|A\|_{\mathcal{L}_i^*}+\frac{1}{L}\right) \right)\right),
		\end{align}

		where $$I=\left\{ i\leq n: (f(x_i) )_{y_i}-\max_{j\neq y_i}( (f(x_i) )_{j})>\gamma   \right\}.$$
		Note that there can be pooling over channels in this case, with the constant $k_l$ being determined after pooling.
		
		Furthermore, if $\kappa_l$ denotes instead the constant such that $|\nbull|_{l,\infty}\leq \sqrt{\kappa_l}|\nbull|_{l}$, the quantity $R_{\mathcal{A}}$ in the above equation can be replaced by

		\begin{align}
		R_{\mathcal{A}}&:= \left(\sup_{i}\|x_i\|_{\mathcal{L}_0}+1\right)\left[\sum_{l=1}^L \left((\|(A-M)^\top\|_{2,1}+\frac{1}{L})\sqrt{\kappa_l}\prod_{i\neq l}\left(\|A\|_{\mathcal{L}_i^*}+\frac{1}{L}\right) \right)^{2/3}\right]^{3/2},\nonumber\\
		\end{align}
		
	\end{theorem}
	
	After passing to the asymptotic regime (taking the choice $|\nbull|_{\mathcal{L}_L}=|\nbull|_{\infty}$ so that $\|\tilde{A}^L\|_{\mathcal{L}_L^*}=\max_{i}\|A^L_{i,\nbull}\|_{2}$):
	
	\begin{theorem}	
		For training and testing points  $(x,y),(x_1,y_1),(x_2,y_2),\ldots,(x_n,y_n)$ as usual drawn iid from any probability distribution over $\mathbb{R}^d\times\{1,2,\ldots,K\}$,	with probability at least $1-\delta$, every network $F_{\mathcal{A}}$ with weight matrices $\mathcal{A}=\{A^1,A^2,\ldots,A^L\}$ and every margin $\gamma>0$ satisfy:
		\begin{align}
		\mathbb{P}\left(\argmax_j(F_{\mathcal{A}}(x)_j)\neq y\right)&\leq \widehat{R}_\gamma(F_{\mathcal{A}})+\widetilde{\mathcal{O}}\left(   \frac{R_{\mathcal{A}}}{\gamma \sqrt{n}}\log(\bar{W})+\sqrt{\frac{ \log(1/\delta) }{ n }  }\right),
		\end{align}
		where  $W$ is the maximum number of neurons in a single layer (after pooling) and  \begin{align}
		\label{fuyan}
		R_{\mathcal{A}}:=&\frac{1}{\gamma}L\left(\rho_L\max_{i}\|A^L_{i,\nbull}\|_{2}\prod_{l=1}^{L-1} \rho_l\|\tilde{A}^l\|_{\mathcal{L}_{l^*}}\right)  \left( \sum_{l=1}^{L-1}\frac{k_l\|(A^l-M^l)^{\top}\|_{\Fr}^{2}}{\|\tilde{A}^l\|_{\mathcal{L}_{l^*}}^{2}}  +\frac{\|A^L\|_{2}^{2}}{\max_{i}\|A^L_{i,\nbull}\|^{2}_{2}}\right)^{\frac{1}{2}},
		\end{align}
		$A^L_{i,\nbull}$ denotes the $i$'th row of $A^L$.
		Furthermore, if $\kappa_l$ denotes instead the constant such that $|\nbull|_{l,\infty}\leq \sqrt{\kappa_l}|\nbull|_{l}$, the quantity $R_{\mathcal{A}}$ in the above equation can be replaced by
		
		\begin{align}
		R_{\mathcal{A}}:=&\frac{1}{\gamma}\left(\rho_L\max_{i}\|A^L_{i,\nbull}\|_{2}\prod_{l=1}^{L-1} \rho_l\|\tilde{A}^l\|_{\mathcal{L}_{l^*}}\right)  \left( \sum_{l=1}^{L-1}\frac{\kappa_l^{1/3}\|(A^l-M^l)^{\top}\|_{2,1}^{2/3}}{\|\tilde{A}^l\|_{\mathcal{L}_{l^*}}^{2/3}}  +\frac{\|A^L\|_{2}^{2/3}}{\max_{i}\|A^L_{i,\nbull}\|^{2/3}_{2}}\right)^{\frac{3}{2}},
		\end{align}
	\end{theorem}

	In particular, if 
	
	\section{Localised analysis with loss function augmentation}
	\label{augmentdetail}

	\begin{proposition}
		\label{Chainingprop1}
		Let $L$ be a natural number and $a_1,\ldots,a_L>0$ be real numbers. Let $\mathcal{V}_0,\mathcal{V}_1,\ldots,\mathcal{V}_{L}$ be $L+1$ finte dimensional vector spaces each endowed with two norms,   $\|\nbull\|_{\infty}$ (the natural $\infty$ norm)  and  $|\nbull|_l$ for $0\leq l\leq L$. 
		Let $B_1,B_2,\ldots,B_L$ be $L$ finite dimensional vector spaces with norms $\|\nbull\|_{1},\|\nbull\|_{2},\ldots,\|\nbull\|_{L}$ and $\mathcal{B}_1,\mathcal{B}_2,\ldots,\mathcal{B}_L$ be the balls of radii $a_1,a_2,\ldots,a_L$ in the spaces $B_1,B_2,\ldots,B_L$ with the norms $\|\nbull\|_{1},\|\nbull\|_{2},\ldots,\|\nbull\|_{L}$ respectively\footnote{The proof works with $\mathcal{B}_1,\mathcal{B}_2,\ldots,\mathcal{B}_L$ being arbitrary sets, but we formulate the problem as above to aid the intuitive comparison with the areas of application of the Proposition.}.

		Suppose also that for each $l\in \{1,2,\ldots,L\}$ we are given an operator $F^l: \mathcal{V}_{l-1}\times B_{l}\rightarrow \mathcal{V}_{l}: (x,A)\rightarrow F^l_{A}(x)$, which is just composed of a linear map $F_{-}^{l}:\mathcal{V}_{l-1}\times B_{l}\rightarrow \mathcal{V}^{-}_{l}$ followed by max and Relu operations incorporated in the activation function $G_l:\mathcal{V}^{-}_{l}\rightarrow \mathcal{V}_{l}$.
		For each $l_1,l_2$ with $l_2>l_1$ and each  $\mathcal{A}^{l_1,l_2}=(A^{l_1+1},\ldots,A^{l_2})\in \mathcal{B}^{l_1,l_2}:=\mathcal{B}_{l_1+1}\times  \mathcal{B}_{l_1+2} \times \ldots  \mathcal{B}_{l_2}$, let us define $$F^{l_1\rightarrow l_2}_{\mathcal{A}^{l_1,l_2}}:\mathcal{V}_{l_1}\rightarrow \mathcal{V}_{l_2}: x\rightarrow F^{l_1\rightarrow l_2}_{\mathcal{A}^{l_1,l_2}}(x)= F^{l_2}_{A^{l_2}}\circ \ldots\circ  F^{l_1}_{A^{l_1}}(x),$$
		and $F_{\mathcal{A}}=F^L_{\mathcal{A}^L}=F^{0\rightarrow L}_{\mathcal{A}^L}.$
		Write similarly $F_{-}^{0\rightarrow l}(X)$ for the preactivations at layer $l$. We assume that an extra component, with index $0$, of $F^{1\rightarrow l}$, computes the minimum distance to a threshold (in the case where there is no pooling, this is the maximum absolute value of any prectivation), so that $$(F^{0\rightarrow l}(X))_0=\min_{(i,j)\in R_l}\left|F_{-}^{0\rightarrow l}(X)_{i}-F^{0\rightarrow l}(X)_j\right|,$$
		where $R_l$ represents the set of pairs of components such that $G_l(x)_j$ potentially depends on the $i^{th}$ component of $x$. We also write $E_l(x_i)$ for $(F^{0\rightarrow l}(x_i))_0$.

		For each  $\mathcal{A}^L=(A^1,A^2,\ldots,A^L)\in \mathcal{B}^L:=\mathcal{B}_1\times  \mathcal{B}_2 \times \ldots,  \mathcal{B}_L$, and for each $l_1,l_2\in \{1,2,\ldots,L\}$, and for each $x\in \mathcal{V}_{0}$ such that $F^{l_1\rightarrow l}_{\mathcal{A}}(x)_0>0$,  the Lipschitz constant of the gradient of $F^{l_1\rightarrow l_2}$ evaluated at $F^{0\rightarrow l_1}(x)$, with respect to the norms $\|\nbull\|_{\infty}$ and $|\nbull|_{l_2}$ is  denoted by $\rho^{\mathcal{A},x}_{l_1\rightarrow l_2}$. The corresponding Lipschitz constant with respect to the norms $\|\nbull\|_{\infty}$ and $\|\nbull\|_{\infty}$ is denoted by $\theta^{\mathcal{A},x}_{l_1\rightarrow l_2}$.

		We suppose the following conditions are satisfied:
		For all $l\in \{1,2,\ldots,L\}$, all $b>0$,  all $z_1,z_2,\ldots,z_n\in\mathcal{V}_{l-1}$ such that $|z_i|_{l-1}\leq b \quad \forall i$ and all $\epsilon>0$, there exists a subset $\mathcal{C}_l(Z,\epsilon)\subset \mathcal{B}_l$ such that
		\begin{align}
		\log(\#\left(\mathcal{C}_l(Z,\epsilon)\right))\leq \frac{C_{l,\epsilon,n}a_l^2b^2}{\epsilon^2},
		\end{align}
		where $C_{l,\epsilon,n}$ is some function of $l,\epsilon,n$
		and, for all $A\in \mathcal{B}_l$, there exists an $\bar{A}\in \mathcal{C}_l(b,\epsilon)$ such that for all $i\leq n$,
		\begin{align}
		\left|F^l_{A}(z_i)-F^l_{\bar{A}}(z_i)\right|_{\infty,l}\leq  \epsilon.
		\end{align}

		For any $0<\epsilon<1$, $b=(b_0,b_{1},b_2,\ldots,b_L)$ such that $b_l\geq 1\quad \forall l$ and $b_L=1$, any set of positive numbers $\rho_1,\rho_2,\ldots,\rho_{L-1}$, any $E_1,E_2,\ldots,E_L$,  and for any $X\in \mathcal{V}_0$ such that $|X|_{0}\leq b_0$, there exists a subset $\mathcal{C}_{\epsilon,b,X}$ of $\mathcal{B}^L$ such that for all $\mathcal{A}=(A^1,A^2,\ldots,A^L)\in \mathcal{B}:=\mathcal{B}^L$ 	there exists a $\bar{\mathcal{A}}\in \mathcal{C}_{\epsilon,b,X}$ such that for all $i\leq n$ such that the following conditions are satisfied:
		\begin{enumerate}
			\item $|\mathcal{F}^l(x_i)|_{l}\leq b_l$  for all $l$
			\item For all $l_1\leq l_2\leq L$, $\rho^{\mathcal{A},x_i}_{l_1\rightarrow l_2}\leq \rho_{l_1}b_{l_2}$.
			\item $E_{l}(x_i)\geq 2 E_{l} \quad \forall l$
			\item For all $l_1\leq l_2\leq L$, $\theta^{\mathcal{A},x_i}_{l_1\rightarrow l_2}\leq E_{l_2}\rho_{l_1}$,	where as usual $E_{l_2}(x_i)=F^{0\rightarrow l_2}(x_i)_0$ denotes the maximum preactivation at layer $l_2$ for input $x_i$,
		\end{enumerate}

		one has
		\begin{align}
		\label{toolll}
		\left|  \mathcal{F}^l_{\mathcal{A}^l}(x_i)-   \mathcal{F}^l_{\bar{\mathcal{A}}^l}(x_i) \right|_{l}&\leq \epsilon \quad \quad (\forall l\leq L \quad \forall i \leq n )\nonumber \\
		\left| \mathcal{F}^l_{\bar{\mathcal{A}}}(x_i)\right|_{l}&\leq 2b_l\quad  \forall l<L\quad \nonumber  \quad \\
		\left|E_{l}^{\bar{\mathcal{A}}}(x_i)\right|&\geq E_l \nonumber \\
		\left|E_{l}^{\mathcal{A}}(x_i)-E_{l}^{\bar{\mathcal{A}}}(x_i)\right|&\leq \epsilon  E_l    \quad \text{and}\nonumber \\
		\Theta^{\bar{\mathcal{A}},x_i}_{l_1\rightarrow l_2}&\leq E_{l_2}\rho_{l_1}\quad \text{and}\quad\rho^{\bar{\mathcal{A}},x_i}_{l_1\rightarrow l_2}\leq \rho_{l_1}b_{l_2}\quad \forall l_1,l_2.
		\end{align}

		Furthermore, we have
		\begin{align}
		\log \#(\mathcal{C}_{\epsilon,b,X})&\leq 4 \left[  \sum_{l=1}^L \left(\frac{C_{l,\epsilon}^{\frac{1}{2}}a_lb_{l-1}\rho_l}{\epsilon}\right)^{\frac{2}{3}} \right]^3
		\leq 4\frac{L^2}{\epsilon^2}  \sum_{l=1}^L \left(\frac{C_{l,\epsilon}^{\frac{1}{2}}a_lb_{l-1}\rho_l}{\epsilon}\right)^{2}.
		\end{align}

	\end{proposition}
	
	\begin{proof}

		As in the proof of Proposition~\ref{Chainingprop}, for $l=1,\ldots,L$, let $\epsilon_l=\frac{\epsilon \alpha_l}{ \rho_{l}}$, where the $\alpha_l>0$ will be determined later satisfying $\sum_{l=1}^L\alpha_l=1$. And again, for any $X=(x_1,\ldots,x_n)$, we define the covers $\mathcal{D}_l\subset \mathcal{B}_l$ for $l\leq L$  by induction by
		$\mathcal{D}_l=\cup_{\mathcal{A}\in \mathcal{D}_1\times \ldots\times \mathcal{D}_{l-1}}\mathcal{C}(2b_l,\{F^{0\rightarrow l-1}_{\mathcal{A}}(x_i):i\leq n\},\epsilon_l)$. Let us write also $\mathcal{D}:=\mathcal{D}_{1}\times \mathcal{D}_{2}\ldots \times \mathcal{D}_L$, and write $d_l$ for the cardinality of $\mathcal{D}_l$.
		The key is to show that none of the thresholds such as Relu or max change value between $\bar{\mathcal{A}}$ and $\mathcal{A}$, which can be seen from the equations above and by induction: let us suppose that the first four of the five inequalities above hold for layers before $l-1$, and that no threshold phenomenon has occured.
		
		Since no threshold has occured, we have that for all $l_{1}\leq l-1$, (and for all $j\leq n$),
		$$\left|  F^{0\rightarrow l}_{\left(\bar{A}_1,\ldots,\bar{A}_{i-1},A_i,\ldots,A_l\right)}(x_j)-   F^{0\rightarrow l}_{\left(\bar{A}_1,\ldots,\bar{A}_{i},A_{i+1},\ldots,A_l\right)}(x_j) \right|_{l}
		\leq \epsilon_i \rho^{\left(\bar{A}_1,\ldots,\bar{A}_{i-1},A_i,\ldots,A_l\right)}_{i\rightarrow l}    = \epsilon_i \rho^{\mathcal{A}}_{i\rightarrow l}  \leq\rho_{i}\epsilon_i b_l,
		$$
		and
		
		$$\left|  F^{0\rightarrow l}_{\left(\bar{A}_1,\bar{A}_2,\ldots,\bar{A}_{i-1},A_i,\ldots,A_l\right)}(x_j)-   F^{0\rightarrow l}_{\left(\bar{A}_1,\bar{A}_2,\ldots,\bar{A}_{i},A_{i+1},\ldots,A_l\right)}(x_j) \right|_{\infty}
		\leq \rho_{i}\epsilon_i E_{l}.
		$$
		Using this, we obtain as before:
		\begin{align}
		\left|  F^{0\rightarrow l}_{\mathcal{A}}(x_j)-   F^{0\rightarrow l}_{\bar{\mathcal{A}}}(x_j) \right|_{l}&\leq \sum_{i=1}^L	\left|  F^{0\rightarrow l}_{\left(\bar{A}_1,\bar{A}_2,\ldots,\bar{A}_{i-1},A_i,\ldots,A_l\right)}(x_j)-   F^{0\rightarrow l}_{\left(\bar{A}_1,\bar{A}_2,\ldots,\bar{A}_{i},A_{i+1},\ldots,A_i\right)}(x_j) \right|_{l}\nonumber \\
		&\leq \sum_{i=1}^l \rho_i b_l \epsilon_i\leq \epsilon b_l,
		\end{align}
		and similarly
		\begin{align}
		\left|  E_{l}^{\mathcal{A}}(x_j)-   E^{\bar{\mathcal{A}}}_l(x_j) \right|_{l}&\leq \sum_{i=1}^L	\left|  F^{0\rightarrow l}_{\left(\bar{A}_1,\bar{A}_2,\ldots,\bar{A}_{i-1},A_i,\ldots,A_l\right)}(x_j)-   F^{0\rightarrow l}_{\left(\bar{A}_1,\bar{A}_2,\ldots,\bar{A}_{i},A_{i+1},\ldots,A_i\right)}(x_j) \right|_{l}\nonumber \\
		&\leq \sum_{i=1}^l \rho_i E_l\epsilon_i\leq \epsilon E_l.
		\end{align}
		From this, since $E_{l}^{\mathcal{A}}(x_j)\geq 2E_l$ by assumption, we conclude that no threshold is crossed at layer $l$, and by the triangle inequality $E_{l}^{\bar{\mathcal{A}}}(x_j)\geq 2E_l$.
		The second equation $	\left| \mathcal{F}^l_{\bar{\mathcal{A}}}(x_i)\right|_{l}\leq 2b_l$ also follows by the triangle inequality.

		By induction, we have proved that no Relu or max pooling threshold was crossed at any layer and the first four inequalities hold. The last inequalities follow from the assumption and the fact that no treshold occurs.
		
	\end{proof}

	\begin{theorem}
		let $b=(b_0,b_{1},b_2,\ldots,b_L)$ such that $b_l\geq 1\quad \forall l$ and $b_L=1$, s $\rho_1,\rho_2,\ldots,\rho_{L-1}$, any $E_1,E_2,\ldots,E_L>0$, $a_1,\ldots,a_L>0$ be given. For any $\delta>0$, with probability $>1-\delta$, every network satisfying $\|(A^l-M^l)^{\top}\|_{2,1}\leq a_l$ for all $l$ satisfies

		\begin{align}
		&\mathbb{P}\left( \argmax_{j\in\{1,2,\ldots,K\}} (F_L(x))_{j}\neq y  \right)\nonumber \\&\leq \frac{n-\#(I)}{n}+ \frac{8}{n}+\frac{1536}{\sqrt{n}}R\left[\log_2(32\Gamma n^2+7\bar{W}n)\right]^{\frac{1}{2}}\log(n)+3\sqrt{\frac{\log(\frac{2}{\delta })}{2n}},
		\end{align}
		where 	\begin{align}
		R^{2/3}&:= \sum_{l=1}^L \left(  a_lb_{l-1} \rho_{l+}         \right)^{2/3}, \quad \text{and} \nonumber\\
		\Gamma&:=\max_{l=1}^L\left(b_{l-1}a_lO_{l-1}m_l\rho_{l+}\right),
		\end{align}
		$I$ is the set of $i\in \{1,2,\ldots,n\}$ such that:
		\begin{enumerate}
			\item $|\mathcal{F}^l(x_i)|_{l}\leq b_l$  for all $l$
			\item For all $l_1\leq l_2\leq L$, $\rho^{\mathcal{A},x_i}_{l_1\rightarrow l_2}\leq \rho_{l_1}b_{l_2}$.
			\item $E_{l}(x_i)\geq 3 E_{l} \quad \forall l$
			\item For all $l_1\leq l_2\leq L$, $\theta^{\mathcal{A},x_i}_{l_1\rightarrow l_2}\leq E_{l_2}\rho_{l_1}$,	where as usual $E_{l_2}(x_i)=F^{0\rightarrow l_2}(x_i)_0$ denotes the maximum preactivation at layer $l_2$ for input $x_i$,
			\item $F(x_i)_{y_i}-\max_jF(x)_j\geq \gamma$,
		\end{enumerate}
		
		and 	\begin{align}
		R^{2/3}&:= \sum_{l=1}^L \left(  a_lb_{l-1} \rho_{l}         \right)^{2/3}, \quad \text{and} \nonumber\\
		\Gamma&:=\max_{l=1}^L\left(b_{l-1}a_lO_{l-1}m_l\rho_{l}\right).
		\end{align}
		
	\end{theorem}
	
	\begin{proof}
		We apply the Rademacher theorem to the loss function:

		\begin{align}
		&l(x,y)=\max\bigg[\sup_{l=1}^{L-1} \lambda_{B_l}\left(\left|F^{0\rightarrow l}(x)\right|_{l}-2B_{l}\right),\lambda_{\gamma}\left(\max_{j\neq y}(F_{\mathcal{A}}(x))_{j}-F_{\mathcal{A}}(x)_{y}\right) ,\nonumber \\&1\left(\exists l_1,l_2: \theta^{\mathcal{A},x_i}_{l_1\rightarrow l_2}> E_{l_2}\rho_{l_1} \lor   \rho^{\mathcal{A},x_i}_{l_1\rightarrow l_2}> \rho_{l_1}b_{l_2} \right), \max_l\left(\lambda_{E_{l}}(2E_l-E_{l}(x)) \right)\bigg]
		\end{align}

		Writing $\mathcal{H}$ for the function class defined by $l(x,F^{\mathcal{A}}(x))$ for $F^{\mathcal{A}}$ satisfying the conditions of the Theorem, since $Y\mapsto\max_{j\neq y}(Y)_{j}-Y_{y}$ is $2$ Lipschitz with respect to the $l^\infty$ norm,  and $l(x,y)=1$ for any $x$ such that there exists $l$ such that $|F^{0\rightarrow l}(x)|_{l}\geq 2b_l$ or for any $x$ that doesnt satisfy the conditions in Theorem~\ref{Chainingprop1},   Propositions~\ref{Chainingprop1} and~\ref{onestep} guanrantee that the covering number of $\mathcal{H}$ satisfies
		\begin{align}
		\label{covmassive}
		\log(\mathcal{N}(\mathcal{H},\epsilon))&\leq 4\times 64\times 2^2 \left[ \sum_{l=1}^L \left(   \frac{a_l(2b_{l-1})}{\epsilon } \rho_{l+}   \log_2\left(8\frac{a_l(2b_{l-1})nm_lO_{l-1}}{(\epsilon/\rho_{l+})}+7O_{l-1}m_ln\right)      \right)^{2/3}\right]^3\nonumber \\
		&\leq 2^{12} R^2 \log_2(32\Gamma n/\epsilon+7\bar{W}n)
		\end{align}
		
		Applying the Rademacher Theorem~\ref{rademachh}, we now obtain

		\begin{align}
		\label{initialmassive}
		&\mathbb{P}\left( \argmax_{j\in\{1,2,\ldots,K\}} (F_L(x))_{j}\neq y  \right)\leq \mathbb{E}\left( l(x,y) \right)\nonumber\\
		&\leq \frac{\sum_{i=1}^nl(x_i,y_i)}{n} +3\sqrt{\frac{\log(\frac{2}{\delta})}{2n}}+2\hat{\rad}_n(l(x,y) )\nonumber \\
		&\leq \frac{n-\#(I)}{n}++3\sqrt{\frac{\log(\frac{2}{\delta})}{2n}}+2\hat{\rad}_n(l(x,y) ).
		\end{align}
		Similarly to the previous proofs, plugging inequality~\eqref{covmassive} into~\eqref{initialmassive} and using Dudley's entropy formula yields the desired result.
	\end{proof}

	Again, by using Lemma~\ref{unionbound}, we can turn this result into:
	
	\begin{theorem}
		\label{posthocresultmassive}
		Suppose we have a $K$ class classification problem and are given $n$ i.i.d. observations $(x_1,y_1),(x_2,y_2),\ldots,(x_n,y_n)\in \mathbb{R}^{U_0\times w_0}\otimes \{1,2,\ldots,K\}$ drawn from our ground truth distribution $(X,Y)$, as well as a fixed architecture as described in Section~\ref{precise}, where we assume the last layer is fully connected and has width $K$ and corresponds to scores for each class. For any $\delta>0$,  with probability $>1-\delta$ over the draw of the training set, for any network $\mathcal{A}=(A^1,A^2,\ldots,A^L)$we have
		\begin{align}
		\label{mainmassive}
		&\mathbb{P}\left( \argmax_{j\in\{1,2,\ldots,K\}} (F_L(x))_{j}\neq y  \right)\nonumber \\&\leq \frac{n-\#(I_{\beta,\gamma})}{n}+ \frac{8}{n}+\frac{1536}{\sqrt{n}}R_{\mathcal{A}}\left[\log_2(32\Gamma n^2+7\bar{W}n)\right]^{\frac{1}{2}}\log(n)+3\sqrt{\frac{\log(\frac{2}{\delta })}{2n}}\nonumber\\
		&+   3\sqrt{\frac{1}{n}\sum_{l=1}^L\log\left(2+\frac{B_{l-1}(X)}{L}\right) +\log\left(2+\frac{\|(A-M)^\top\|_{2,1}}{L}\right)+ \log\left(2+\frac{\rho^{\mathcal{A}}_{l}}{L}\right)    } ,
		\end{align}
		
		where $\rho^{\mathcal{A}}_{l}$, $E_l$ and $B_{l}(X)\geq 1$ can be chosen in any way that depends on both $\mathcal{A}$ and $X$, and $I$ is then defined as the set of indices $i\leq n$ such that
		\begin{enumerate}
			\item $|\mathcal{F}^l(x_i)|_{l}\leq b_l$  for all $l$
			\item For all $l_1\leq l_2\leq L$, $\rho^{\mathcal{A},x_i}_{l_1\rightarrow l_2}\leq \rho_{l_1}b_{l_2}$.
			\item $E_{l}(x_i)\geq 3 E_{l} \quad \forall l$
			\item For all $l_1\leq l_2\leq L$, $\theta^{\mathcal{A},x_i}_{l_1\rightarrow l_2}\leq E_{l_2}\rho_{l_1}$,	where as usual $E_{l_2}(x_i)=F^{0\rightarrow l_2}(x_i)_0$ denotes the maximum preactivation at layer $l_2$ for input $x_i$,
			\item $F(x_i)_{y_i}-\max_jF(x)_j\geq \gamma$.
		\end{enumerate}
		
		The particular choice $E_l=\frac{1}{3}\max_i E^{0\rightarrow l}(x)$, $B_l(X)=\max\left(\max_{i\leq n}\left|F_{0\rightarrow l}^{\mathcal{A}}(x_i)\right|_l,1\right)$ and \\ $\rho^{\mathcal{A}}_l=\max\left(\max_{i}\max_{\tilde{l}\geq l}\frac{\rho^{\mathcal{A},x_i}_{l_1\rightarrow l_2}}{b_{l_2}}, \max_{i}\max_{\tilde{l}\geq l}\frac{\theta^{\mathcal{A},x_i}_{l_1\rightarrow l_2}}{E_{l_2}}\right)$ yields   $$I=\left\{ i\leq n: (f(x_i) )_{y_i}-\max_{j\neq y_i}( (f(x_i) )_{j})>\gamma   \right\}.$$
		In the above formula,
		\begin{align}
		R_{\mathcal{A}}^{2/3}&:= \sum_{l=1}^L \left(  (\|(A-M)^\top\|_{2,1}+\frac{1}{L})(B_{l-1}(X)+\frac{1}{L}))(\rho^{\mathcal{A}}_{l+}  +\frac{1}{L})     \right)^{2/3}, \quad \text{and} \nonumber\\
		\Gamma&:=\max_{l=1}^L\left((B_{l-1}(X)+\frac{1}{L})(\|(A-M)^\top\|_{2,1}+\frac{1}{L})O_{l-1}m_l(\rho^{\mathcal{A}}_{l}+\frac{1}{L})\right).
		\end{align}

	\end{theorem}
	
	After reducing to the $\tilde{O}$ notation, we obtain:
	
	\begin{theorem}
		\label{posthocasymptmassive}
		For training and testing points  $(x,y),(x_1,y_1),(x_2,y_2),\ldots,(x_n,y_n)$ as usual drawn iid from any probability distribution over $\mathbb{R}^d\times\{1,2,\ldots,K\}$,	with probability at least $1-\delta$, every network $F_{\mathcal{A}}$ with weight matrices $\mathcal{A}=\{A^1,A^2,\ldots,A^L\}$ and every margin $\gamma>0$ satisfy:
		\begin{align}
		\mathbb{P}\left(\argmax_j(F_{\mathcal{A}}(x)_j)\neq y\right)&\leq \frac{n-\#(I)}{n}+\widetilde{\mathcal{O}}\left(   \frac{(R_{\mathcal{A}}+L)}{ \sqrt{n}}\log(\bar{W})+\sqrt{\frac{ \log(1/\delta) }{ n }  }\right),
		\end{align}
		where  $W$ is the maximum number of neurons in a single layer (after pooling) and  \begin{align}
		&R_{\mathcal{A}}:= \left[\sum_{l=1}^{L-1}\left(B_{l-1}(X)\|A^l-M^l\|_{2,1}\rho^{\mathcal{A}}_l\right)^{2/3}+\left(\frac{B_{L-1}(X)}{\gamma}\|A^L-M^L\|_{\Fr}\right)^{2/3}\right]^{3/2},
		\end{align}
		\normalsize
		where for any layer $l_1$, 	$B_{l_1}(X):=\max_{i}\left|F^{0\rightarrow l_l}(x_i)\right|_{l_1}$ denotes the maximum $l^2$ norm of any convolutional patch of the layer $l_1$ activations, over all inputs. $B_L(X)=\gamma$,
		$E_l=\frac{1}{3}\max_i E^{0\rightarrow l}(x)$, $B_l(X)=\max\left(\max_{i\leq n}\left|F_{0\rightarrow l}^{\mathcal{A}}(x_i)\right|_l,1\right)$ and \\ $\rho^{\mathcal{A}}_l=\max\left(\max_{i}\max_{\tilde{l}\geq l}\frac{\rho^{\mathcal{A},x_i}_{l_1\rightarrow l_2}}{b_{l_2}}, \max_{i} \max_{\tilde{l}\geq l}\frac{\theta^{\mathcal{A},x_i}_{l_1\rightarrow l_2}}{E_{l_2}}\right)$, and
		
		$$I=\left\{ i\leq n: (f(x_i) )_{y_i}-\max_{j\neq y_i}( (f(x_i) )_{j})>\gamma   \right\}.$$
	\end{theorem}

	\section{Dudley's entropy formula}

	For completeness, we include a proof of (a variant of) the classic Dudley's entropy formula. To enable a comparison with the results used in~\cite{Spectre}, we write the result with arbitrary $L^p$ norms. We will, however, only use the $L^\infty$ version.
	
	\begin{proposition}
		Let $\mathcal{F}$ be a real-valued function class taking values in $[0,1]$, and assume that $0\in \mathcal{F}$. Let $S$ be a finite sample of size $n$. For any $2\leq p\leq \infty$, we have the following relationship between the Rademacher complexity  $\rad(\mathcal{F}|_{S})$ and the covering number $\mathcal{N}(\mathcal{F}|S,\epsilon,\|\nbull\|_{p})$.
		\begin{align*}
		\rad(\mathcal{F}|_{S})\leq \inf_{\alpha>0} \left(     4\alpha+\frac{12}{\sqrt{n}}  \int_{\alpha}^1      \sqrt{\log \mathcal{N}(\mathcal{F}|S,\epsilon,\|\nbull\|_{p}) }     \right),
		\end{align*}
		where the norm $\|\nbull\|_{p}$ on $\mathbb{R}^m$ is defined by $\|x\|_{p}^p=\frac{1}{n}(\sum_{i=1}^m|x_i|^p)$.
	\end{proposition}
	
	\begin{proof}
		Let $N\in \mathbb{N}$ be arbitrary and let $\epsilon_i=2^{-(i-1)}$ for $i=1,2,\ldots,N$. For each $i$, let $V_i$ denote the cover achieving $\mathcal{N}\left( \mathcal{F}|S,\epsilon_i,\|\nbull\|_{p}  \right)$, so that
		\begin{align}
		\label{Dudley1}
		\forall f\in \mathcal{F} \quad \exists v\in V_i \quad \left(\frac{1}{n} \sum_{t=1}^n \left(f(x_t)-v_t\right)^p   \right)^{\frac{1}{p}}\leq \epsilon_i,
		\end{align}
		and $\#(V_i)=\mathcal{N}\left( \mathcal{F}|S,\epsilon_i,\|\nbull\|_{p}  \right)$. For each $f\in\mathcal{F}$,let $v^i[f]$ denote the nearest element to $k$ in $V_i$. Then we have, where $\sigma_1,\sigma_2,\ldots,\sigma_n$ are $n$ i.i.d. Rademacher random variables,

		\begin{align*}
		&\mathbb{E}_{\sigma } \sup_{f\in\mathcal{F}}\frac{1}{n}\sum_{t=1}^n\sigma_t f(x_t)\\
		&=\mathbb{E}_{\sigma } \sup_{f\in\mathcal{F}}\left[ \frac{1}{n}\sum_{t=1}^n \sigma_t \left(f_t(x_t)-v^N_t[f]\right)-\sum_{i=1}^{N-1}\frac{1}{n}\sum_{t=1}^n \sigma_t \left(v_t^i[f]-v_t^{i+1}[f]\right) +\frac{1}{n}\sum_{t=1}^n \sigma_t v^1_t [f]     \right]\\
		&\leq \mathbb{E}_{\sigma } \sup_{f\in\mathcal{F}}\left[ \frac{1}{n}\sum_{t=1}^n \sigma_t \left(f_t(x_t)-v^N_t[f]\right)\right]+ \sum_{i=1}^{N-1} \mathbb{E}_{\sigma } \sup_{f\in\mathcal{F}}\left[\frac{1}{n}\sum_{t=1}^n \sigma_t \left(v_t^i[f]-v_t^{i+1}[f]\right) \right]\\ & \quad \quad \quad \quad \quad \quad\quad \quad \quad \quad \quad \quad\quad \quad \quad \quad \quad \quad\quad \quad \quad \quad \quad \quad\quad \quad \>+\mathbb{E}_{\sigma } \sup_{f\in\mathcal{F}}\left[  \frac{1}{n}\sum_{t=1}^n \sigma_t v^1_t [f]  \right].
		\end{align*}
		\normalsize
		For the third term, pick $V_1=\{0\}$, so that
		\begin{align*}
		\mathbb{E}_{\sigma } \sup_{f\in\mathcal{F}}\left[ \frac{1}{n} \sum_{t=1}^n \sigma_t v^1_t [f]  \right]=0.
		\end{align*}
		
		For the first term, we use H\"older's inequality to obtain, where $q$ is the conjugate of $p$,
		\begin{align*}
		\sum_{i=1}^{N-1} \mathbb{E}_{\sigma } \sup_{f\in\mathcal{F}}\left[\frac{1}{n}\sum_{t=1}^n \sigma_t \left(f_t(x_t)-v^N_t[f]\right) \right]&\leq \mathbb{E}_{\sigma} \left(\frac{1}{n}\sum_{t=1}^n |\sigma_t|^q   \right)^{\frac{1}{q}}\left(  \frac{1}{n}\sum_{t=1}^n \left|f_t(x_t)-v^N_t[f]\right|^p  \right)^{\frac{1}{p}}\\
		&\leq  \epsilon_N.
		\end{align*}
		
		Next, for the remaining terms, we define $W_i=\{v^i[f]-v^{i+1}[f]| f\in\mathcal{F}\}$. Then note that we have $|W_i|\leq |V_i||V_{i+1}|\leq |V_{i+1}|^2$, and then
		\begin{align*}
		\mathbb{E}_{\sigma } \sup_{f\in\mathcal{F}}\left[ \frac{1}{n} \sum_{t=1}^n\sigma_t  \left(v_t^i[f]-v_t^{i+1}[f]\right)  \right]\leq \mathbb{E}_{\sigma } \sup_{w\in W_i}\left[  \frac{1}{n} \sum_{t=1}^n \sigma_t w_t \right].
		\end{align*}
		
		Next,
		\begin{align*}
		&\sup_{w\in W_i} \sqrt{\frac{1}{n}\sum_{t=1}^n w_t^2}=\sup_{f\in\mathcal{F}} \left\|v^i[f]-v^{i+1}[f]\right\|_{2} \\&
		\leq \sup_{f\in\mathcal{F} }\left\| v^i[f]-(f(x_1),\ldots,f(x_n))\right\|_{2}+ \sup_{f\in\mathcal{F} }\left\| (f(x_1),\ldots,f(x_n))-v^{i+1}[f]\right\|_{2}\\
		&\leq \sup_{f\in\mathcal{F} }\left\| v^i[f]-(f(x_1),\ldots,f(x_n))\right\|_{p}+ \sup_{f\in\mathcal{F} }\left\| (f(x_1),\ldots,f(x_n))-v^{i+1}[f]\right\|_{p}
		\\ &\leq \epsilon_i+\epsilon_{i+1}=3\epsilon_{i+1},
		\end{align*}
		where at the third line, we have used the fact that $p\geq 2$.
		Using this, as well as Massart's lemma, we obtain
		\small
		\begin{align*}
		\mathbb{E}_{\sigma}\sup_{w\in W_i} \left[ \frac{1}{n}\sum_{t=1}^n\sigma_t w_t   \right]&\leq \frac{1}{\sqrt{n}}\sqrt{ 2\sup_{w\in W_i }\frac{1}{n} \sum_{t=1}^n w_t^2 \log|W_i|}  \leq \frac{3\epsilon_{i+1}}{\sqrt{n}}\sqrt{2\log|W_i|}\leq \frac{6}{\sqrt{n}}\epsilon_{i+1}\sqrt{\log|V_{i+1}|}.
		\end{align*}
		\normalsize
		Collecting all the terms, we have
		\begin{align*}
		\mathbb{E}_{\sigma} \sup_{f\in\mathcal{F}}\frac{1}{n}\sum_{t=1}^n \sigma_t f(x_t)&\leq \epsilon_N +\frac{6}{\sqrt{n}}\sum_{i=1}^{N-1}\epsilon_{i+1}\sqrt{\log\mathcal{N}\left( \mathcal{F}_{S},\epsilon_{i+1},\|\nbull\|_{p}  \right)}\\
		&\leq  \epsilon_N +\frac{12}{\sqrt{n}}\sum_{i=1}^{N}(\epsilon_i-\epsilon_{i+1})\sqrt{\log\mathcal{N}\left( \mathcal{F}_{S},\epsilon_{i},\|\nbull\|_{p}  \right)}\\
		&\leq \epsilon_N +\frac{12}{\sqrt{n}}\int_{\epsilon_{N+1}}^1\sqrt{\log\mathcal{N}\left( \mathcal{F}_{S},\epsilon,\|\nbull\|_{p}  \right)}d\epsilon.\\
		\end{align*}
		Finally, select any $\alpha>0$ and take $N$ to be the largest integer such that $\epsilon_{N+1}>\alpha$. Then $\epsilon_{N}=4\epsilon_{N+2}\leq 4\alpha$, and therefore
		\begin{align*}
		\epsilon_N +\frac{12}{\sqrt{n}}\int_{\epsilon_{N+1}}^1\sqrt{\log\mathcal{N}\left( \mathcal{F}_{S},\epsilon,\|\nbull\|_{p}  \right)}d\epsilon&\leq 4\alpha +\frac{12}{\sqrt{n}}\int_{\alpha}^1\sqrt{\log \mathcal{N}\left( \mathcal{F}|_{S},\epsilon,\|\nbull\|_{p}   \right)}d\epsilon,
		\end{align*}
		as expected.
	\end{proof}

	\section{Detailed comparison to other works}
	\label{comparisondeep}
	
	At the same time as the first version of this paper appeared on ArXiv, a different solution to the weight sharing problem (but not to the multiclass problem) was posted on arXiv~\cite{GRRR}. The bound, which relies on computing the Lipschitz constant of the map from parameter space to function space and applying known results about classifiers of a given number of parameters, states that for some constant $C$ and for large enough $n$, the generalisation gap $\mathbb{E}(l(\hat{g}))-\widehat{\mathbb{E}}(l(\hat{g}))$ satisfies with probability $\geq 1-\delta$, assuming each input has unit $l^2$ norm,
	
	\begin{align}
	\label{paramconv}
	\mathbb{E}(l(\hat{g}))-\widehat{\mathbb{E}}(l(\hat{g}))\leq \mathcal{C}B\sqrt{\frac{\mathcal{W}(\sum_{l=1}^Ls_l-\log(\gamma))+\log(1/\delta)}{n}},
	\end{align}
	where $s_l$ is an upper bound on the spectral norm of the matrix corresponding to the $l^{th}$ layer, $\gamma$ is the margin, and $\mathcal{W}$ is the number of parameters, taking weight sharing into account by counting each parameter of convolutional filters only once. We note that the method to obtain the bound is radically different from ours, and closer to~\cite{convattempt}. Indeed, it relies on the following general lemma mostly composed of known results, which bounds the complexity of function classes with a given number of parameters:
	
	\begin{proposition}{\cite{GRRR,Mohri,GINE,Pollard,Tal94,Tal96}}
		Let $G$ be a set of functions from a domain $Z$ to $[0,M]$ such that for some $B>5$ and for some $d\in \mathbb{N}$ and for some norm $\|\nbull\|_1$ on $\mathbb{R}^d$, there exists a  map from $\mathbb{R}^d$ to $G$ which is $B$-Lipschitz with respect to the norms $\|\nbull\|_1$ and $\|\nbull\|_{\infty}$. For large enough $n$ and for any distribution $P$ over $Z$, if $S$ is sampled $n$ times independently form $P$, for any $\delta>0$, we have with probability $\geq 1-\delta$ that for all $g\in G$,
		$$\mathbb{E}_{z\sim P}(g(z))\leq \hat{\mathbb{E}}_{S}(g)+CM\sqrt{     \frac     { d\log(B)+\log(1/\delta)  }  {n}     },$$
		where $C$ is some constant.
	\end{proposition}
	
	The proof of inequality~\eqref{paramconv} then boils down to explicitly bounding the Lipschitz constant of the map from parameter space to function space assuming some  fixed norm constraints on the weights.
	Note that the term $\sum_{l=1}^Ls_l$ comes from a logarithm of $\prod_{i=1}^Ls_i$.

	Norm-based bounds such as ours and those in~\cite{Spectre} require more details to work in activation space directly, thereby replacing the explicit parameter dependence by a dependence on the norms of the weight matrices.
	
	Furthermore, one notable advantage of Norm-based bounds is their suitability to be incorporated in further analyses that take distance from initialisation into account, as do the approaches of the SDE branch of the litterature (\cite{gradientover,Aroratwolayers,GradientDescent,NTK,Overparametrised,Repeat}). Indeed, note that the bound~\eqref{paramconv} from~\cite{GRRR} is still large, and still scales as the number of parameters, even if the weight matrices $A^l$ are \textit{arbitrarily close} to the initialised matrices $M^l$. In contrast, the capacity estimate in our bounds converges toa constant times either $\sqrt{\frac{1}{n}}$ (\ref{posthocasymptsimple}) or $\sqrt{\frac{L}{n}}$ (\ref{posthocasymptlip}) when the weights approach initialisation.
	
	In what follows, we illustrate this fact by comparing our bounds with those of~\cite{Spectre,GRRR} both in the general case and in a simple illustrative particular case.

	\textbf{Comparison}
	
	Here we use the same notation as in the rest of the paper, assume the lipschitz constants $\rho_l$ are $1$, and set fixed norm constraints. Below, $B$ denotes an upper bound on the $L^2$ norms of input data points. Recall $O_l$ is the number of convolutional patches in layer $l$, $m_l$ is the number of filters in layer $l$, $a_l$ is an upper bound on the $L^2$ norm of the filter matrix, and $s_l$ is an upper bound on the spectral norm of the corresponding full convolution operation.
	
	For a completely general feed forward convolutional neural network, we have the following comparison, where $\mathcal{C}$ is an unspecified constant,  $\Gamma=\max_{l=1}^LW_lBa_l\prod_{i\neq l}s_i $, and $d_l$ is the size of convolutional filters at layer $l$.  For ease of comparison, we compare only with the forms of our theorems involving explicit spectral norms.
	\begin{table}[h!]
		\centering
		\begin{tabular}{lll}
			\toprule
			& General bound      \\
			
			\midrule
			Prev. work& $\frac{72B\log(2w)\log(n)}{\sqrt{n}\gamma}L\prod_{i=1}^Ls_i\left[\sum_{l=1}^L\frac{O_{l-1}^2m_la_l^2}{s_l^2}\right]^{\frac{1}{2}}+\Omega$  \\
			Simult. work    & $\mathcal{C}\sqrt{\frac{\left[\sum_{l=1}^Ld_{l-1}O_{l-1}m_l\right]\left[\sum_{l=1}^Ls_l-\log(\gamma)\right]+\log(1/\delta)}{n}}$  \\
			Our bounds     &  $\frac{2^{11}BL\log(n)\log_2(32\Gamma n^2/\gamma+7\bar{W}n)^{\frac{1}{2}}}{\sqrt{n}\gamma}\prod_{i=1}^Ls_i\left[\sum_{l=1}^{L-1}\frac{w_lU_la_l^2}{s_l^2}+\frac{a_L^2}{s_L^2}\right]^{\frac{1}{2}}+\Omega$   \\
			\bottomrule
		\end{tabular}
	\end{table}
	Here, we write $\Omega=\frac{8}{n}+3\sqrt{\frac{\log(1/\delta)}{2n}}$.
	Note that at the term of the sum that corresponds to layer $l$, our bound is better than the bound in~\cite{Spectre} roughly by a factor of $\sqrt{\frac{O_{l-1}^2m_l}{w_l}}$. A factor of $\sqrt{\frac{\bar{W}_l}{W_l}}=\sqrt{\frac{O_{l-1}m_l}{U_lw_l}}$ is removed by exploiting the $L^\infty$-continuity of the activation functions and the pooling operation, while a further factor of $\sqrt{O_{l-1}}$ is removed by exploiting weight sharing.

	{The main advantage of the bound in~\cite{GRRR} compared to ours is the lack of a product of spectral norms $\prod_{i=1}^Ls_i$ as a factor inside the square root. This is a significant difference as it arguably removes implicit exponential depth dependence, and illustrate the difference between the methods. However, as explained in Section~\ref{augmentdetail} andSubsection~\ref{augment}, this problem can be tackled independently.
		
		The main disadantage of the bound in~\cite{GRRR} compared to ours is that it exhibits an explicit factor of $\mathcal{W}=\left[\sum_{l=1}^Ld_{l-1}O_{l-1}m_l\right]$, the total number of parameters in the network. Note also that the factor appears in a term where it is not multiplied by any norm quantities.This has important implications. If the trained network has a large number of very small weights (or weights very close to initialisation), the corresponding contribution is small.  This suggests the superior potential of refined norm-based bounds to explain the generalisation capabilities of DNN's at the overparametrised regime~\cite{sizematters,overparam,Overparametrised,gradientover}.
		More crucially, the spectral norm version of our bound~\eqref{posthocasymptsimple} converges to a very small number $\tilde{O}(\sqrt{\frac{1}{n}})$ when the weights approach the initialised values $M$, whilst the bound in~\cite{GRRR} still scales like $\tilde{O}(\sqrt{\mathcal{W}/n})$ in that case, making our bound better suited to incorporation in bounds that take the optimisation procedure into account.
		
		Furthermore, whilst the bound does not directly involve the input space dimension, it does explicitly depend on the size of the convolutional filters, including at the input ($0^{th}$) layer. On the other hand, our bound depends instead on the post pooling width at layers $l$  for $l=1,2,\ldots L$, which is the maximum possible number of active neurons (after pooling) at layer $l\geq 1$ (which excludes the input layer). Furthermore, norm-based bounds such as ours exhibit some degree of architecture robustness. If weights are pruned, our bound is the same as it would be if we had started with the smaller architecture, whilst the bound in~\cite{GRRR} still involves the original number of parameters.

		Furthermore, we believe the meaningful estimate of complexity lies in the term \small $\|(A^\top-M^\top)\|_{2,1}$\normalsize, whilst the product of spectral norms is more of a technicality: networks constrained to have all spectral norms equal to $1$ form a  rich function class of high relevance to the original problem, and it has been shown in~\cite{Otherlong} that this class can be approached through regularisation, and indeed that doing so even improves the accuracy.
		
		\textbf{Remark on the proof techniques:} In fact, it is interesting to note that the main advantages and disadvantages of our bound compared to that in~\cite{GRRR} are intimately related via a tradeoff that appears in the proof choice: it is possible to bound the covering number of a function class depending on a parameter $\theta$ in different ways depending on which of (1) the dimension of parameter space (2) the norm and architectural constraints on the parameters, is the most restrictive. When the dimension $d$ of parameter space is moderate and assuming the Lipschitz constant is known, it is straightforward to bound the covering number by a quantity of the form $(\frac{C}{\epsilon})^{d}$, directly for an arbitrary function class of which we know nothing except the number of parameters and the relevant Lipzitsch constant. When the norm constraints are stronger and the dimension is large or possibly infinite, it is necessary to use different tools such as the Maurey sparsification lemma and $L^\infty$ versions of it which apply directly only to linear classifiers and thus require more use of the architectural assumption and chaining arguments to generalise to DNN.
		To better understand the trade-off between the product of spectral norm and the architecture robustness, it is best to think of the simple example of a linear classifier in the setting of $L^2$ covering numbers: consider the Maurey sparsification lemma A.6 from~\cite{Spectre}. The quantity $k$ scales like the reciprocal of granularity $\epsilon$ of the cover. The covering number behaves like the number of choices of integers $(k_1,\ldots,k_d)$ such that $\sum_{i=1}^dk_d=k$. This quantity is equal to $\left(\begin{array}{c} k+d-1\\ d-1\end{array}\right)$. Depending on whether $d$ is large or small compared to $k$, this can be approximated by $k^d$ or $d^k$. The first choice yields explicit dependence on the number of parameters, and the second choice yields dependence on the norms of the input and weight matrices, which in the case of a neural network eventually translates into a product of spectral norms.

		\textbf{On input size independence/robustness to downsampling.}
		Note that contrary to ours, the bound in~\cite{GRRR} depends explicitly on $d$ ($d_0$ in the general case, the size of the first input layer's convolutional patches).

		We argue that this implies our bound exhibits an even stronger form of input size-independence.
		We consider an idealised scenario where a downsampled version of each image contains the same information as the original image: suppose that each input $x_i$, of size $2a \times 2b$, satisfies $(x_i)_{2j,2j}=(x_i)_{2j,2j+1}=(x_i)_{2j+1,2j}=(x_i)_{2j+1,2j+1}$ for any $j$, where $(x_i)_{r,r'}$ denotes the $(r,r')$ pixel of image $x_i$. If we create a downsampled version $\tilde{x}_i$, of size $a\times b$, of the input such that $(\tilde{x}_i)_{j_1,j_2}=\sqrt{4}(x_{i})_{2j_1,2j_2}$, and similarly replace the first layer convolutional weights $w_{u_1,u_2}$ of size $2c_1\times 2c_2$ by $\sqrt{\tilde{w}_{2u_1,2u_2}+\tilde{w}_{2u_1,2u_2+1}^2+\tilde{w}_{2u_1+1,2u_2}^2+\tilde{w}_{2u_1+1,2u_2+1}^2}$, of size $c_1\times c_2$, then assuming the stride is also divided by two in the downsampled case, all activation at the next layers (from layer 1 onwards) are the same in both the original and the downsampled version. Thus there is a natural bijection $\mathcal{F}$ between solutions to the first problem and the second, and it is not hard to convinve onesef that the image by $\mathcal{F}$ of the solution of SGD on one problem is the solution to SGD on the other, with the generalisation gap also staying exactly the same.However, in the case of the bound in~\cite{GRRR}, the bound is smaller in the case of the downsampled version due to the decrease in the number of parameters. Our bound, on the other hand, stays the same. Indeed, the maximum $L^2$ norm $\tilde{B}$ of convolutional patches stays the same, as does the $L^2$ norm of every convolutional filter $w$, despite the change in the number of parameters.

		\section{On class dependency and working with $L^2$ norms}
		\label{fillup}
		
		As mentioned after Theorem~\ref{fully}, the main advantages of our bounds in terms of class dependency compared to the work of Bartlett~\cite{Spectre} is to replace the capacity contribution of the last layer $\|(A_L-M_L)^\top\|_{2,1}$ by $\|A_L-M_L\|_{\Fr}$. As explained before, in the case where the norms of the classifying vectors $(A_L)_{c,\nbull}$ for $c\leq C$ are within a constant factor of each other, the bound in Bartlett has an implicit dependence of $C$, whilst our bound has an implicit dependence of $\sqrt{C}$ (ignoring logarithmic terms).  In this section, we show that for large $C$, the region of weight space where this condition does not hold has vanishingly small Lebesgue (or Gaussian) measure, further confirming the theoretical importance of our improvements.

		\begin{proposition}
			Let $X\in \mathbb{R}^n$  be a random variable which is either spherically symmetric or has i.i.d. component, and satisfies $\mathbb{E}(|X_1|^2)<\infty$.
			For all $0<\epsilon< \frac{1}{3}$ with $\epsilon(\mathbb{E}(X_1^2))^{-1}\leq \frac{1}{2}$ and $\epsilon(\mathbb{E}(|X_1|))^{-1}\leq \frac{1}{2}$ and for all $n$, we have, with probability $\geq 1- 5e^{-2\epsilon^2n}$, \begin{align}
			\label{bloups}
			C(1-U)\|X\|_{1}\leq \sqrt{n}\|X\|_{2}\leq C (1+U) \|X\|_{1},
			\end{align}
			
			with $C=\frac{\sqrt{\mathbb{E}((X_1)^2)}}{\mathbb{E}(|X_1|)}$ and  $U=4\frac{\epsilon}{\mathbb{E}(X_1^2)}+4\frac{\epsilon}{\mathbb{E}(|X_1|)}+\epsilon $
		\end{proposition}
		\begin{proof}
			Since the multivariate Gaussian is spherically symmetric and the inequality~\ref{bloups} is radially symmetric, we only need to prove the case with i.i.d. components. 
			We begin by picking $R_n$ large enough to ensure \begin{align}
			\label{firstie}
			&\mathbb{P}\left( \exists i \leq n:  |X_i|\geq R_n \right)\leq \sum_{i=1}^n 	\mathbb{P}\left(   |X_i|\geq R_n \right)\nonumber\\&= n\mathbb{P}\left(   |X_i|\geq R_n \right)\leq \exp(-\epsilon^2n),
			\end{align}
			$$\mathbb{E}\left((X_1)^2| X_1\leq R_n\right)>2\mathbb{E}(X_1^2)/3,$$ 	$$\mathbb{E}\left(|X_1|\big | X_1\leq R_n\right)>2\mathbb{E}\left(\left|X_1\right|  \right)/3,$$
			and \begin{align}
			\label{secc}
			(1-\epsilon) \frac{\mathbb{E}(|X_1|)}{\sqrt{\mathbb{E}((X_1)^2)}}&\leq\frac{ \mathbb{E}\left(|X_1|\big| X_1\leq R_n\right)}{\sqrt{\mathbb{E}\left((X_1)^2| X_1\leq R_n\right)}}\nonumber \\&\leq 	(1+\epsilon) \frac{\mathbb{E}(|X_1|)}{\sqrt{\mathbb{E}((X_1)^2)}}
			\end{align}
			
			which can be done by the assumption that $\mathbb{E}(|X_1|^2))<\infty$.
			Next, let $Y$ be the random variable $X$ conditioned on $|X_i|\leq R_n$ for all $i$ and let $\tilde{X}=YR_n^{-1}$.

			By applying Hoeffding's lemma, we note that we have \begin{align}
			\label{three}
			\mathbb{P}\left(\left|\|\tilde{X}\|_{2}^2-n\mathbb{E}(\tilde{X}_1^2)\right|\geq \epsilon n\right)\leq 2\exp\left(- 2n\epsilon^2   \right),
			\end{align}
			and 
			\begin{align}
			\label{four}
			\mathbb{P}\left(\left| \|\tilde{X}\|_1-n\mathbb{E}(|\tilde{X}_1|)\right|\geq \epsilon n    \right)\leq 2\exp\left(- 2n\epsilon^2   \right).
			\end{align}
			Note that $\left|\|\tilde{X}\|_{2}^2-n\mathbb{E}(\tilde{X}_1^2)\right|< \epsilon n$ implies \begin{align}
			\left|\|\tilde{X}\|_{2}-\sqrt{n}\sqrt{\mathbb{E}(\tilde{X}_1^2)}\right|&\leq \frac{\epsilon n}{\|\tilde{X}\|_{2}+\sqrt{n}\sqrt{\mathbb{E}(\tilde{X}_1^2)}}\nonumber \\&\leq \frac{\epsilon \sqrt{n}}{\sqrt{\mathbb{E}(\tilde{X}_1^2)}}.\end{align}
			
			Hence, by inequalities~\ref{three},~\ref{secc} and~\ref{four}, we have with probability $>1-4e^{-2\epsilon^2n}$, 
			\begin{align}
			&\frac{\sqrt{n}\|\tilde{X}\|_2}{\|\tilde{X}\|_1}\leq \frac{n\sqrt{\mathbb{E}(\tilde{X}_1^2)}+ \frac{\epsilon n}{\sqrt{\mathbb{E}(\tilde{X}_1^2)}}         }{  n\mathbb{E}(|\tilde{X}_1|) -\epsilon n             }\nonumber \\ &\leq \frac{\sqrt{\mathbb{E}(\tilde{X}_1^2)}}{\mathbb{E}(|\tilde{X}_1|)}\frac     {1+\frac{\epsilon}{\mathbb{E}(\tilde{X}_1^2)}}     {1- \frac{\epsilon}{\mathbb{E}(|\tilde{X}_1|)}} \nonumber \\
			&\leq \frac{\sqrt{\mathbb{E}(\tilde{X}_1^2)}}{\mathbb{E}(|\tilde{X}_1|)}\left(1+\frac{\epsilon}{\mathbb{E}(\tilde{X}_1^2)}\right)  \left(1+2\frac{\epsilon}{\mathbb{E}(|\tilde{X}_1|)}\right)\nonumber \\
			&\leq \frac{\sqrt{\mathbb{E}(\tilde{X}_1^2)}}{\mathbb{E}(|\tilde{X}_1|)} \left(1+\frac{\epsilon}{\mathbb{E}(\tilde{X}_1^2)}+2\frac{\epsilon}{\mathbb{E}(|\tilde{X}_1|}+\frac{\epsilon^2}{\mathbb{E}(|X_1|)\mathbb{E}(\tilde{X}_1^2)}\right)\nonumber \\
			&\leq \frac{\sqrt{\mathbb{E}(\tilde{X}_1^2)}}{\mathbb{E}(|\tilde{X}_1|)} \left(1+2\frac{\epsilon}{\mathbb{E}(\tilde{X}_1^2)}+2\frac{\epsilon}{\mathbb{E}(|\tilde{X}_1|)}\right)\nonumber \\
			&\leq  \frac{\sqrt{\mathbb{E}(X_1^2)}}{\mathbb{E}(|X_1|)} \left(1+3\frac{\epsilon}{\mathbb{E}(X_1^2)}+3\frac{\epsilon}{\mathbb{E}(|X_1|)}\right)(1+\epsilon)\nonumber \\
			&\leq \frac{\sqrt{\mathbb{E}(X_1^2)}}{\mathbb{E}(|X_1|)} \left(1+4\frac{\epsilon}{\mathbb{E}(X_1^2)}+4\frac{\epsilon}{\mathbb{E}(|X_1|)}+\epsilon \right),
			\end{align}
			as expected. The proof of the other inequality is similar. 
		\end{proof}

		\section{Rademacher Theorem}
		Recall the definition of the Rademacher complexity of a function class $\mathcal{F}$:
		\begin{definition}
			Let $\mathcal{F}$ be a class of real-valued functions with range $X$. Let also $S=(x_1,x_2,\ldots,x_n)\in X$ be $n$ samples from the domain of the functions in $\mathcal{F}$. The empirical Rademacher complexity $\rad_S(\mathcal{F})$ of $\mathcal{F}$ with respect to $x_1,x_2,\ldots,x_n$ is defined by
			\begin{align}
			\rad_S(\mathcal{F}):=\mathbb{E}_{\delta}\sup_{f\in\mathcal{F}}\frac{1}{n}\sum_{i=1}^n  \delta_if(x_i),
			\end{align}
			where $\delta=(\delta_1,\delta_2,\ldots,\delta_n)\in\{\pm 1\}^n$ is a set of $n$ iid Rademacher random variables (which take values $1$ or $-1$ with probability $0.5$ each).
		\end{definition}
		Recall the following classic theorem~\citep{rademach}:
		\begin{theorem}
			\label{rademachh}
			Let $Z,Z_1,\ldots,Z_n$ be iid random variables taking values in a set $\mathcal{Z}$. Consider a set of functions $\mathcal{F}\in[0,1]^{\mathcal{Z}}$. $\forall \delta>0$, we have with probability $\geq 1-\delta$ over the draw of the sample $S$ that $$\forall f \in \mathcal{F}, \quad \mathbb{E}(f(Z))\leq \frac{1}{n}\sum_{i=1}^nf(z_i)+2\rad_{S}(\mathcal{F})+3\sqrt{\frac{\log(2/\delta)}{2n}}. $$
		\end{theorem}

		\section{Experiments}
		\label{example}
		Although the main aim of this paper is purely theoretical, we provide two simple experiment strands here to illustrate the behaviour of our bound on data. 
		
		\subsection{Augmented MNIST}
		\begin{wrapfigure}{r}{0.44\textwidth}\vspace{-0.6cm}
			\hspace{-32pt} 
			\centering \includegraphics[width=0.38\textwidth]{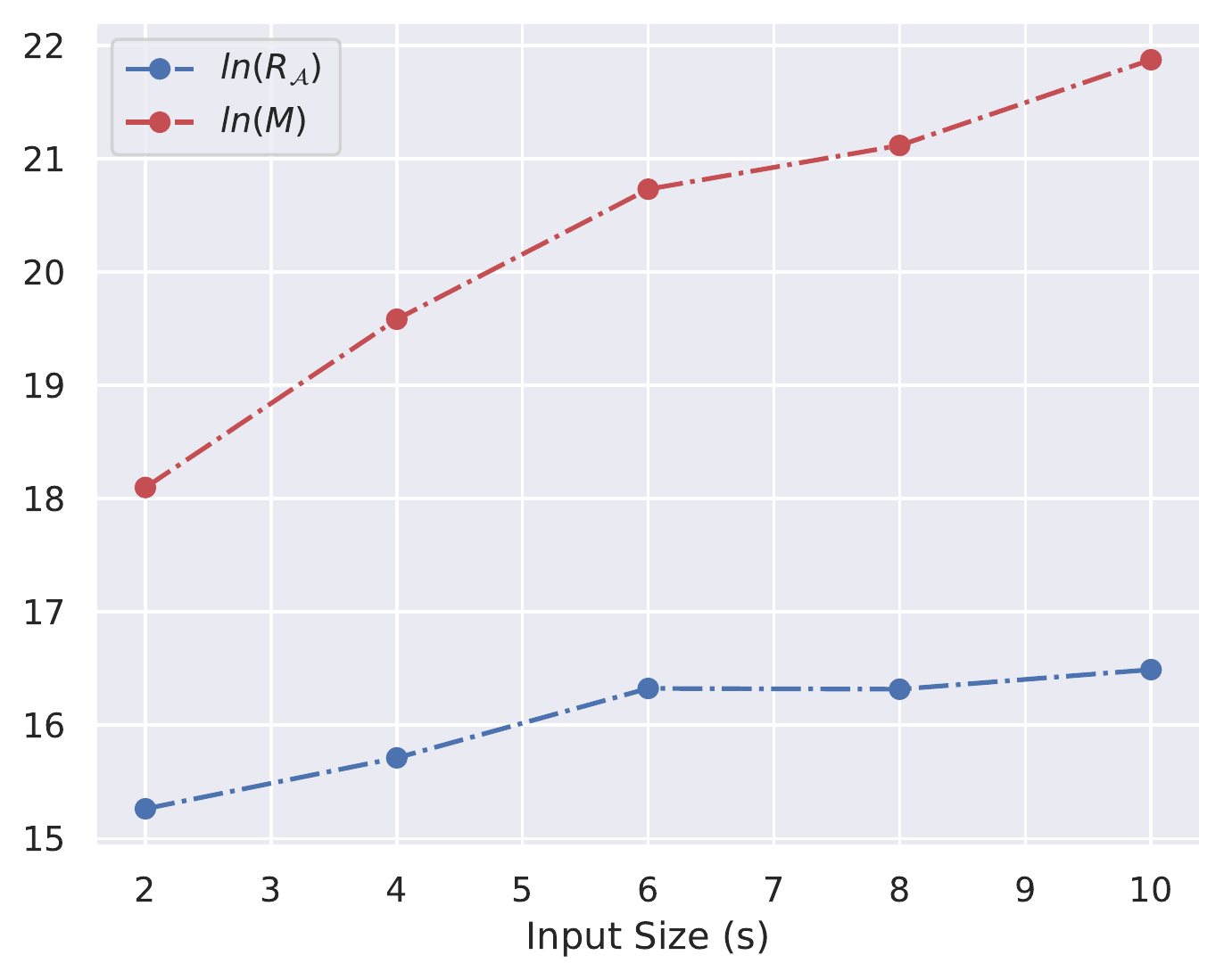}
			\vspace{-16pt}
		\end{wrapfigure}
		
		Our first experiment studies how the proposed bound changes with input dimension. The goal is to generate a sequence of datasets with increasing input sizes that have a small effect on the problem complexity and the convolutional architectural design.  To that end, for each $s\in \{2,4,\ldots,10\}$, we generate data points of size $28s \times 28s$ as follows. First, an image $I$ of the MNIST dataset is randomly sampled along with its label $\ell$. Then we embed $s/2$ non-intersecting copies of $I$ into a large black image of the size $28s \times 28s$ at random locations. For each input size $s \in \{2,4,6,8,10\}$, we generate training set of $50000$ data points.  
		
		The model used is a convolutional network with $4$ convolutional layers, followed by a fully-connected layer with $10$ outputs for the classes. The filters at each convolutional layer are of size $3 \times 3 $, applied with strides of $2$ and the numbers of channels from input to output are respectively as follows: $64$, $128$, $128$, and $64$.  
		
		We train all the models using the cross-entropy loss and weight decay with an ADAM optimizer until they achieve $99\%$ training accuracy. For each dataset, we select the margin to be the largest margin to achieve $96\%$ training accuracy.

		For each value of $s$, we compute the main term $R_{\mathcal{A}}$ in our bound (see~\eqref{OurR}) and the term $M/\gamma$ from equation~\eqref{source}. In the graph shown we plot the two bounds vs the dataset size $s$ in log-scale.
		

		\subsection{Synthetic data}
		\label{exampledetail}
		
		In this experiment, each data point is a sequence of length $L$ digits from the set $\{0,1,2,3\}$. We fix $20$ "signature" sequences $s_1,s_2,\ldots,s_{20}$ of length $15$, the first $10$ of which (i.e. $\{s_1,s_2,\ldots,s_{10}\}$) are associated with label $0$, and the last $10$  (i.e. $\{s_{11},s_{12},\ldots,s_{20}\}$)  of which are associated with label $1$. Each data point is created by inserting $5$ of the signatures into an originally uniformly random sequence of length $L$ at a uniformly random position. Optionally, we repeat each inserted subsequence a total of $iter$ times, where $iter$ is a parameter (duplicate signatures need not appear consecutively). The label is determined by a majority vote of the signatures present. For instance, if signatures $s_1,s_2$ and $s_{11}$ are present, the label is $0$.

		We use one-hot encoding and employ a two-layer neural network composed of one convolutional layer without any padding, and one fully connected layer. We do not use any offset terms. We use $50$ filters, and pooling is over the whole spacial region, so that the total number of hidden neurons is also $50$. Using a variation of our theorems from Section~\ref{withnorms}, we compute for each input the normalised margins $\gamma(x_i)/R$ where \footnotesize $R=\left((\widetilde{B}/\sqrt{n}) \sqrt{k \sum_{i=1}^Kf_i^2 \sup_{i\leq C}F_i^2+ \sum_{i=1}^Kf_i^2 \sum_{i=1}^CF_i^2      }\right)$\normalsize, and the $F_i$'s (resp. $f_i$'s) are upper bounds on the $L^2$ norms of 1st (resp. 2nd) layer filters.
		
		We run the model for both $N=350$ and $N=20000$, and for $L=1000,4000$. The parameter $\iter$, which we vary proportionately to the total length appears required for optimisation purposes. Of course, it also has some influence on generalisation, but bridging the data dependency gap is beyond the scope of this work, where we focus on generalisation bounds valid on the whole of weight space.
		
		We illustrate experimental results in Figure~\ref{margindist1}. Besides the margins normalised with $R$ being several orders of magnitude larger than the ones normalised with $M$, a point of interest is that in both data regimes the value of $L$ has a strong influence on the classically normalised margins, but a mild to moderate influence on both our normalised margins and two subjective measures of data insufficiency: the test accuracy and the distribution of the margins. For $N=20000$ and all values of $L$, the margins are clearly divided into three sets depending on how many inserted signatures in the datapoint are associated with the same label~\footnote{$\{3,2\}$ is frequent and difficult to classify, $\{4,1\}$ is easier and rarer, $\{5,0\}$ is even easier and very rare}. For $N=350$ (all values of $L$), the three groups are still identifiable, but are less well separated, which shows the problem is in a similarly borderline insufficient data regime. In conclusion, classification problems of similar difficulty but different data size lead to similar normalised margins using our formula but very different normalised margins when using $M$ from equation~\eqref{source}.

		\begin{figure}
			\begin{minipage}{.24\textwidth}
				\centering
				\includegraphics[width=\textwidth]{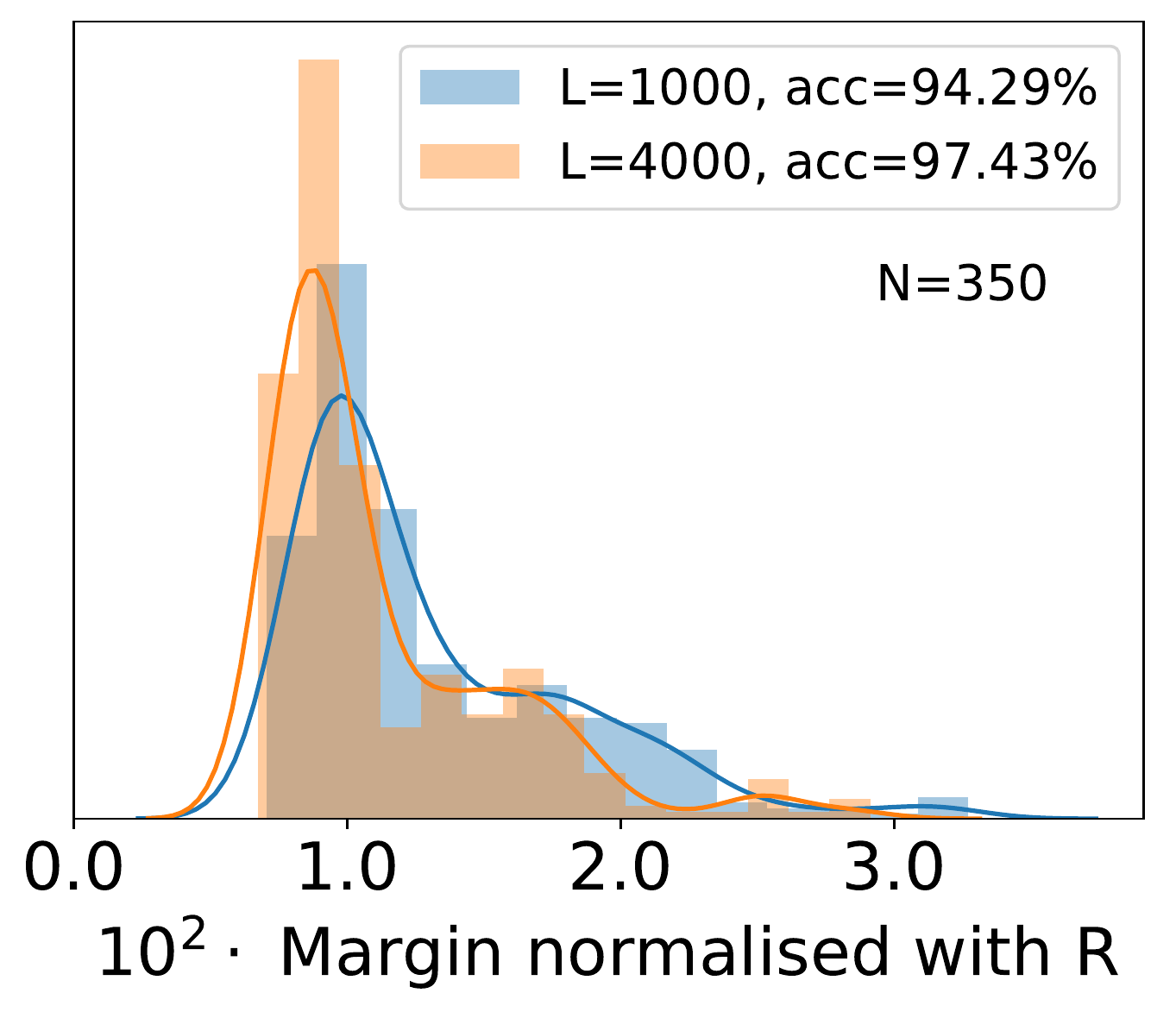}
			\end{minipage}
			\begin{minipage}{.24\textwidth}
				\centering
				\includegraphics[width=\textwidth]{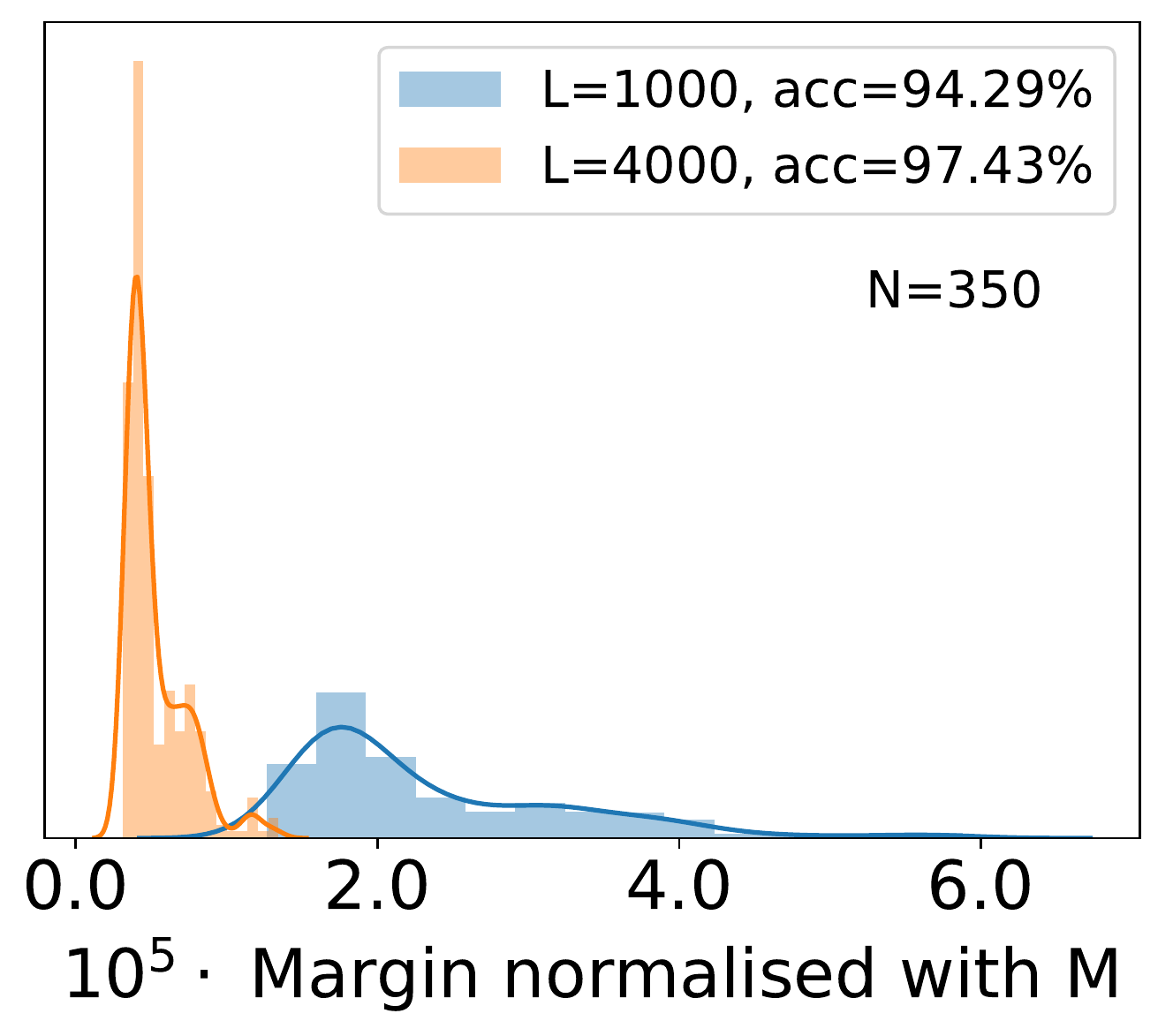}
			\end{minipage}
			\begin{minipage}{.24\textwidth}
				\centering
				\includegraphics[width=\textwidth]{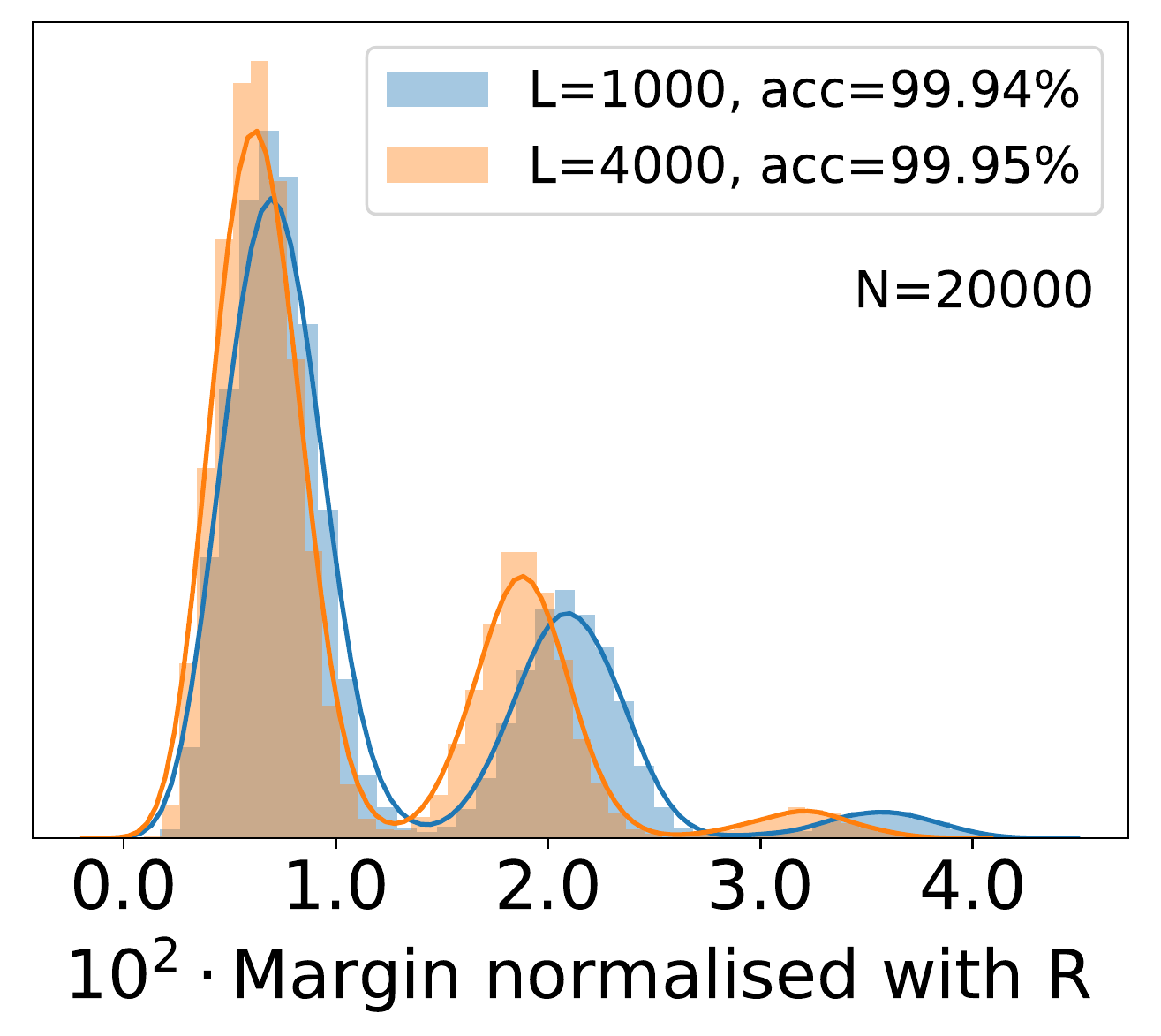}
			\end{minipage}
			\begin{minipage}{.24\textwidth}
				\centering
				\includegraphics[width=\textwidth]{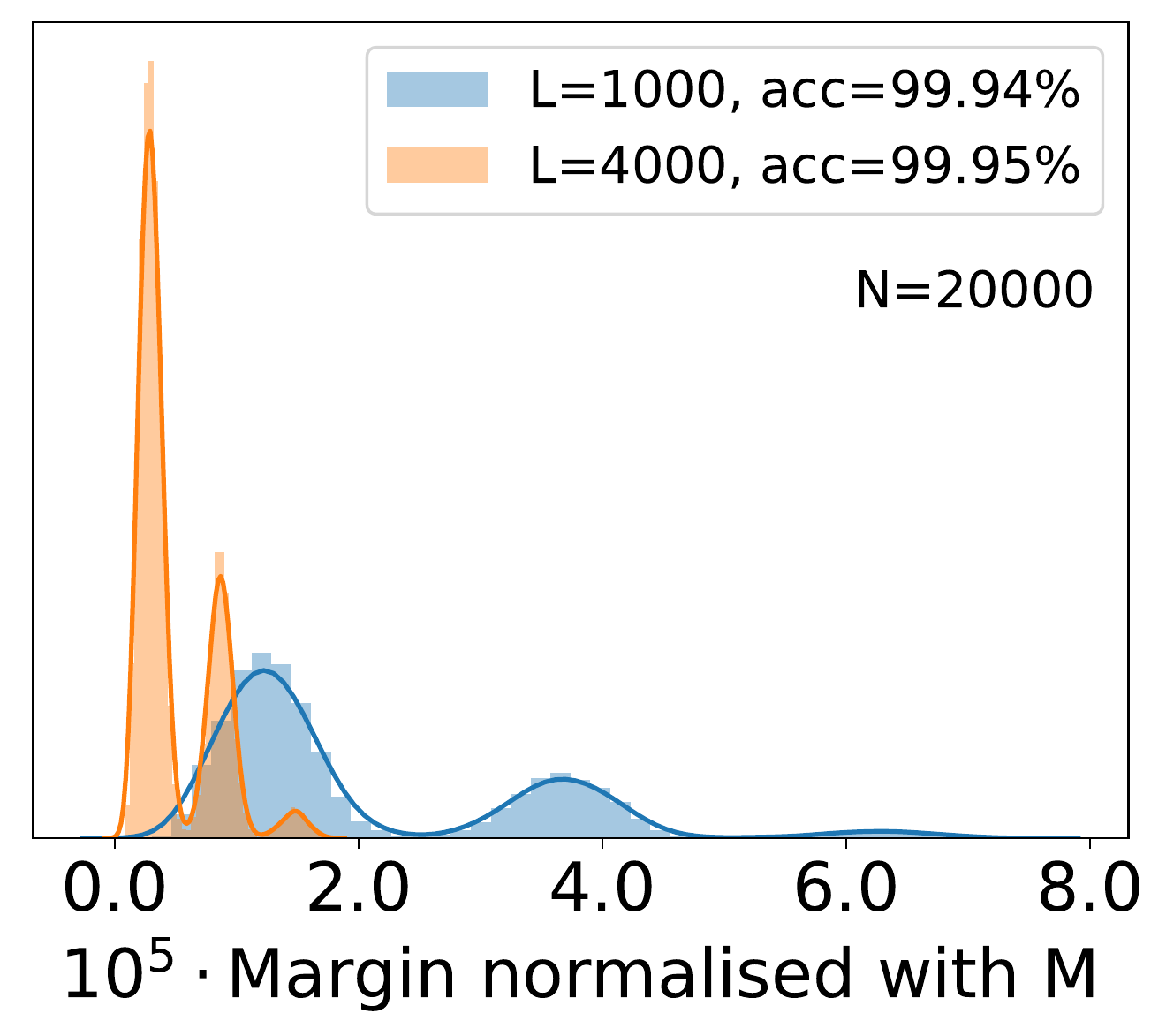}
			\end{minipage} 	\caption{Distribution of normalised margins for different values of $N$ and $L$}
			\label{margindist1}
		\end{figure}

		\bibliography{bibliography}

\end{document}